%% file: main3.tex
\documentclass{article}

\usepackage[main,final]{neurips_2026}

\usepackage[utf8]{inputenc}
\usepackage[T1]{fontenc}
\usepackage{hyperref}
\hypersetup{hypertexnames=false}
\usepackage{url}
\usepackage{booktabs}
\usepackage{amsfonts,amsmath,amssymb,amsthm}
\usepackage{mathtools}
\usepackage{microtype}
\usepackage{xcolor}
\usepackage{graphicx}
\usepackage{multirow}
\usepackage{enumitem}
\usepackage{multicol}
\usepackage{float}
\usepackage{algorithm}
\usepackage{algpseudocode}

\graphicspath{{figures/}}

\newcommand{\R}{\mathbb{R}}
\newcommand{\E}{\mathbb{E}}

\newcommand{\Law}{\mathcal{L}}
\newcommand{\dd}{\,\mathrm{d}}

\newcommand{\core}{\widetilde{\rho}_1}

\newcommand{\pisur}{\widetilde{\pi}}

\newtheorem{theorem}{Theorem}
\newtheorem{proposition}[theorem]{Proposition}
\newtheorem{lemma}[theorem]{Lemma}
\newtheorem{corollary}[theorem]{Corollary}
\theoremstyle{definition}
\newtheorem{assumption}{Assumption}

\theoremstyle{remark}
\newtheorem{remark}[theorem]{Remark}
\setcitestyle{numbers,square,comma}

\title{Coreset-Induced Conditional Velocity Flow Matching}

\author{%
  Zihua She \\
  Department of Statistics\\
  Purdue University\\
  West Lafayette, IN 47907 \\
  \texttt{she8@purdue.edu} \\
  \And
  Jianxi Su \\
  Terry College of Business\\
  University of Georgia\\
  Athens, GA 30602 \\
  \texttt{jianxi.su@uga.edu} \\
  \AND
  Xiao Wang \\
  Department of Statistics\\
  Purdue University\\
  West Lafayette, IN 47907 \\
  \texttt{wangxiao@purdue.edu} \\
}

\begin{document}


\maketitle

\begin{abstract}
We propose Coreset-Induced Conditional Velocity Flow Matching (CCVFM), a
generative model that augments hierarchical rectified flow with a
data-informed source distribution. Standard flow matching regresses only
the conditional mean of the velocity; hierarchical flow matching instead
models the full conditional velocity law with a second, inner flow in
velocity space, but this inner
flow is asked to transport isotropic Gaussian noise to a multimodal
target velocity distribution from scratch. Our key observation is that this inner source can
be replaced by a closed-form surrogate built from a coreset of the
target. CCVFM first compresses the target into weighted atoms using an entropic
Sinkhorn coreset and lifts them to a Gaussian mixture. The induced conditional
velocity law is then a closed-form Gaussian mixture that can be sampled without
a learned neural sampler. 
{A lightweight correction flow, trained from this surrogate source,
then refines the remaining surrogate-to-target residual rather than
learning an entire noise-to-data map.}
{We prove that the surrogate transport cost equals the
target--surrogate Wasserstein gap under an explicit compression
assumption, whereas the noise-source analogue has a dimension-scale
lower bound. We further characterize the conditional second moment of
the training target when the source velocity is drawn from the surrogate
law independently of the target velocity, and show that the part of it
that depends on the source is small when the surrogate conditional law is
close to the true conditional velocity law in mean and covariance.}
Empirically, on
MNIST, CIFAR-10, ImageNet-32, and CelebA-HQ, the proposed method reaches
competitive few-step generation under matched architectures. Code is available at \url{https://github.com/ziwaa-se/CCVFM}.

\end{abstract}

\section{Introduction}

Flow matching~\citep{lipman2023flow,liu2023rectified,albergo2023stochastic}
has emerged as a simple and competitive paradigm for generative
modeling. Given a source $X_0\sim\rho_0$ and a target $X_1\sim\rho_1$,
one interpolates $X_t=(1-t)X_0+tX_1$, defines the velocity
$V=X_1-X_0$, and regresses a field $u_\theta(x,t)$ on $V$ under a
squared loss. The minimizer is the \emph{conditional mean}
$u^\star(x,t)=\E(V\mid X_t=x)$, which suffices to match the marginal
evolution but \emph{collapses the entire conditional law
$\pi(v\mid x,t)$ to a single direction}. Whenever the same
intermediate state $x$ is consistent with several endpoint
completions, mean-field learning blurs the local transport
structure and forces long integration paths (high NFE) at inference
to compensate. The NFE cost of few-step generation therefore has a
structural component that no amount of solver engineering can
remove: the field itself has been averaged.

The closest prior attempt to break this bottleneck is the
\emph{hierarchical rectified flow} (HRF2) of \citet{zhang2025hrf},
which models the full conditional law $\pi(v\mid x,t)$ via a
second-level flow in velocity space. HRF2 proves a key structural
identity: when $\rho_1$ is a GMM, the induced conditional velocity
law is itself a closed-form GMM \citep[Thm.~1 and Cor.~1 of][]{zhang2025hrf},
and we leverage this identity extensively. Its limitation is the choice of source: the
correction flow starts from an isotropic
$v_0\sim\mathcal{N}(0,I)$ and must learn the \emph{entire} transport
from noise to a potentially multimodal velocity law, so the
correction step carries most of the load and low-NFE generation
remains hard.

We introduce Coreset-Induced Conditional Velocity Flow Matching
(CCVFM), which keeps HRF2's hierarchical structure but replaces its
isotropic source with a data-informed surrogate in closed form. Stage~I fits $K$ weighted atoms using entropic
Sinkhorn and lifts them to a GMM $\core=\sum_k w_k\mathcal N({\mu_k},\Sigma_k)$.
Stage~II plugs $\core$ into the closed-form velocity identity, giving a
zero-parameter mixture sampler for $\pisur(v\mid x,t)$. Stage~III trains a correction flow whose source velocity is drawn through the mixture representation of $\pisur$ rather than from $\mathcal N(0,I)$, so that it learns a surrogate-to-target residual instead of the full noise-to-data map.

The theoretical comparison is sharp at the generation boundary $t=0$, where
$\pi(v\mid x_0,0)=\rho_1(x_0+v)$. The CCVFM correction cost equals
$W_2(\rho_1,\core)$ and is $O(K^{-1/d})$ under our explicit Stage-I compression
assumption; the corresponding noise-source cost has a positive dimension-scale
lower bound. 
In addition, write $C=(X_t,t)$ and let $\mu_\pi(C),\Sigma_\pi(C)$ and
$\mu_{\mathrm{sur}}(C),\Sigma_{\mathrm{sur}}(C)$ denote the conditional
means and covariances of $\pi(\cdot\mid C)$ and $\pisur(\cdot\mid C)$.
When the source velocity is drawn from $\pisur(\cdot\mid C)$ independently
of the target velocity, the conditional second moment of the
Stage~III training residual equals
$\operatorname{tr}\Sigma_\pi(C)+\operatorname{tr}\Sigma_{\mathrm{sur}}(C)+
\|\mu_\pi(C)-\mu_{\mathrm{sur}}(C)\|^2$. Its source-dependent excess, i.e.\
the part beyond the intrinsic spread $\operatorname{tr}\Sigma_\pi(C)$
shared by every source, is small whenever the
surrogate conditional law is close to the true conditional velocity law
in mean and covariance, in contrast to the independent-Gaussian-source
baseline whose source-dependent excess is $\|\mu_\pi(C)\|^2+d$.

\paragraph{Relation to standard flow matching.}
With a single outer step ($J{=}1$, our default and the setting of our
theory), CCVFM is equivalent at the population optimum to data-space flow
matching from $\core$ to $\rho_1$ under the same endpoint coupling
(\S\ref{sec:generation}). It shares with flow matching / conditional flow
matching (FM/CFM) the straight interpolation, squared loss, and Euler
sampling; its novelty lies in the coreset-GMM surrogate source and the
resulting surrogate-to-data residual-correction task. Samplers with $J>1$
outer steps are a secondary extension (Appendix~\S B).

\paragraph{Contributions.}
(1) We introduce a coreset-to-velocity pipeline that replaces the inner
Gaussian source in hierarchical flow matching by a compact, analytically
sampleable coreset-GMM surrogate, which provides a more informative source
distribution.
(2) We prove a transport-task reduction (Theorem~\ref{thm:reduction}):
the CCVFM correction-source cost vanishes with $K$ under the Stage-I
compression assumption, while the HRF2 (noise-source) cost is bounded
below by a positive constant in $K$ and $n$. We further establish a
second-moment identity for the surrogate source (Theorem~\ref{thm:posterior-coupling}):
a source velocity drawn from $\pisur(\cdot\mid C)$ through its mixture
weights, independently of the target velocity, has marginal exactly
$\pisur(\cdot\mid C)$, and its conditional training-residual second moment
admits a closed-form expression in the conditional means and covariances of
$\pi(\cdot\mid C)$ and $\pisur(\cdot\mid C)$, whose source-dependent
excess is small whenever the surrogate is close to the target in mean
and covariance. We validate both identities of Theorem~\ref{thm:posterior-coupling}
by Monte Carlo on seven datasets (Appendix~\S D).
(3) We obtain strong few-step image results: MNIST
FID$_{50k}=0.75$, CIFAR-10 FID$_{50k}=6.35$,
ImageNet-32 FID$_{50k}=8.76$, and CelebA-HQ 256
FID$_{28k}=4.17$, all at $51$ NFE. (4) {Building on existing nearest-neighbour-based
generative-model diagnostics
\citep{balle2022reconstructing,meehan2020nonparametric,borji2018evaluation,carlini2023diffusion},
we adapt the $1$-NN primitive into} a $1$-NN Inception-feature
goodness-of-fit diagnostic 
showing no preferential closeness to training data.


\section{Background}
\label{sec:background}

Rectified flow~\citep{liu2023rectified} and conditional flow
matching~\citep{lipman2023flow} sample $X_0\sim\mathcal N(0,I_d)$ and
$X_1\sim\rho_1$, set $X_t=(1-t)X_0+tX_1$, and regress a field on
$V=X_1-X_0$. The optimum $\E[V\mid X_t=x]$ matches marginals but discards the
full conditional law, which motivates velocity-distribution methods such as
HRF2~\citep{zhang2025hrf}. HRF2 learns $\pi(v\mid x_t,t)$ through a second
flow with Gaussian inner source and uses the identity
\begin{equation}
  \pi(v\mid x_t,t) \propto \rho_0\big(x_t-tv\big)\,\rho_1\big(x_t+(1-t)v\big),
\label{eq:true-conditional}
\end{equation}
whose GMM specialization gives explicit component means, covariances, and
weights whenever $\rho_1$ is a GMM. With $V_1=X_1-X_0$, HRF2 draws
$V_0\sim\mathcal N(0,I_d)$ independently of $V_1$ and regresses an inner
field $f_\theta(V_\tau,\tau,X_t,t)$ on $V_1-V_0$, where
$V_\tau=(1-\tau)V_0+\tau V_1$. The remaining difficulty is that HRF2's
inner flow starts from isotropic noise and must cross mixture modes.

Weighted coresets compress data into atoms $\{(w_k,\mu_k)\}_{k=1}^K$ that
approximate $\rho_1$ in transport distance
\citep{cuturi2013sinkhorn,claici2018wasserstein}. Classical quantization
\citep{graf2000foundations} gives the $O(K^{-1/d})$ benchmark for globally
optimal $K$-atomic approximants under standard assumptions. For the deployed
entropic-Sinkhorn/GMM output we use the corresponding rate only through the
explicit Stage-I compression assumption.


\section{CCVFM}
\label{sec:method}

CCVFM has three stages (Figure~\ref{fig:method}): build a coreset-GMM
surrogate $\core$, read off its closed-form conditional velocity law
$\pisur(v\mid x_t,t)$, and train a correction flow whose source is
built from $\pisur(\cdot\mid C)$, so the learned residual is a
surrogate-to-target correction rather than a noise-to-data map.

\begin{figure}[!htb]
  \centering
  \includegraphics[width=0.95\linewidth]{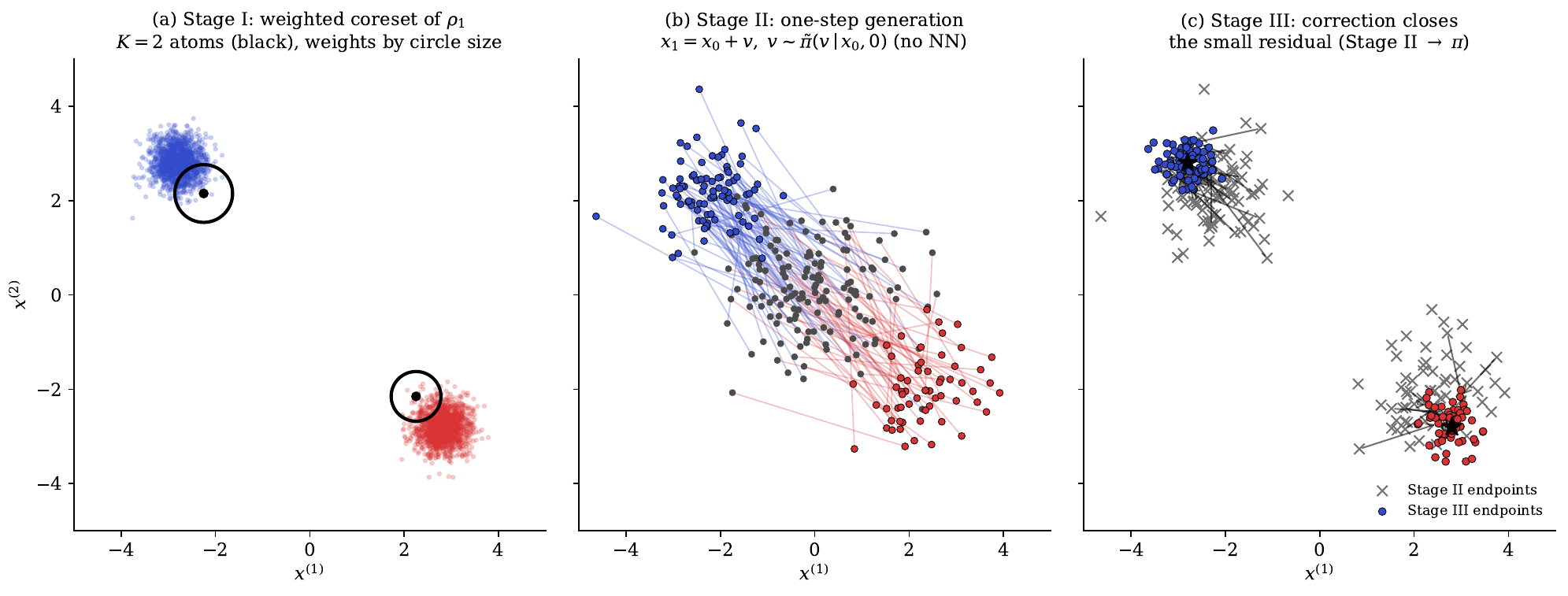}
  \caption{CCVFM pipeline: an entropic-Sinkhorn coreset is lifted to a
GMM surrogate $\core$; the induced law $\pisur(v\mid x_t,t)$ is a closed-form
GMM; a correction flow with a source built from $\pisur$ refines the
surrogate-to-target residual.}
  \label{fig:method}
\end{figure}

\subsection{Stage I: Sinkhorn coreset and GMM lift}
\label{sec:stageI}

Given samples $x_1,\ldots,x_n\sim\rho_1$ and budget $K\ll n$, Stage~I
approximates $\hat\rho_{1,n}=n^{-1}\sum_i\delta_{x_i}$ by
$\sum_k w_k\delta_{\mu_k}$. 
The coupling $T\in\R_+^{n\times K}$ records soft assignments from data
to anchors, with marginals $\sum_kT_{ik}=1/n$ and $\sum_iT_{ik}=w_k$;
$T$ is used only inside Stage~I to update the anchors and weights.

Concretely, we compute $(w,\mu,T)$ jointly by minimizing the
{Kullback-Leibler-regularized entropic objective}
\begin{equation}
  \min_{\{\mu_k,w_k\},\,T_{ik}\ge 0}
  \sum_{i,k} T_{ik}\|x_i-\mu_k\|^2
  \;{+\;\lambda\,\mathrm{KL}\!\big(T\,\big\|\,\tfrac{1}{n}\mathbf{1}_n\otimes w\big)},\qquad
  \sum_k T_{ik}=\tfrac{1}{n},\;\;\sum_i T_{ik}=w_k,
\label{eq:coreset}
\end{equation}
{with $\mathrm{KL}(T\|P)=\sum_{i,k}T_{ik}\log(T_{ik}/P_{ik})$ for the
prior measure $P_{ik}=\frac{1}{n}w_k$. Holding $(\mu,w)$ fixed and
optimizing in $T$ subject to the row constraints yields the closed-form
update $T_{ik}\propto w_k\exp(-\|x_i-\mu_k\|^2/\lambda)$
(the GMM E-step used in Algorithm~\ref{alg:stage1}), since the column marginal $\sum_iT_{ik}=w_k$ is generated rather than constrained.}
We alternate this $T$-update with barycentric anchor updates
$\mu_k\leftarrow w_k^{-1}\sum_iT_{ik}x_i$, $w_k\leftarrow\sum_iT_{ik}$,
at $O(nK)$ cost per outer iteration. The atom measure is then lifted to
$\core=\sum_k w_k\mathcal N({\mu_k},\Sigma_k)$, where
$\Sigma_k=L_kL_k^\top+{\sigma_k^2}I_d$ is fit from weighted
residuals by SVD {(component-specific noise variance, matching
Algorithm~\ref{alg:stage1})}.
Smoothing makes $\tilde\pi$ well-defined; the deployed compression rate is
handled explicitly in Assumption~\ref{ass:stageI-compression}.

\subsection{Stage II: closed-form conditional velocity law}
\label{sec:stageII}

Once $\core$ is available we can read off the induced conditional
velocity law without any further learning. Plugging $\core$ into the identity \eqref{eq:true-conditional} and completing the square
in $v$ yields:
\begin{align}
  \pisur(v\mid x_t,t) &= \sum_{b=1}^K \gamma_b(x_t,t)\,\mathcal{N}\!\big(v;{m_b(x_t,t)},\Lambda_b(t)\big),\label{eq:pisur}\\
  A_b(t) &= t^2 I + (1-t)^2\Sigma_b^{-1},\;
  \Lambda_b(t)=A_b(t)^{-1},\;\notag\\
  b_b(x_t,t) &= t x_t - (1-t)\Sigma_b^{-1}(x_t-\mu_b),\;
  {m_b}(x_t,t)=\Lambda_b(t)\,b_b(x_t,t),\notag
\end{align}
with closed-form weights
$\gamma_b(x_t,t)\propto w_b\,\mathcal N\big(x_t;t\mu_b,(1-t)^2I+t^2\Sigma_b\big)$. At the generation boundary $t=0$, \eqref{eq:pisur} simplifies
drastically:
\begin{equation}
  \pisur(v\mid x_0,0)=\core(x_0+v)=\textstyle\sum_b w_b\,\mathcal{N}(v;\mu_b-x_0,\Sigma_b).
\label{eq:pisur-t0}
\end{equation}
Sampling from~\eqref{eq:pisur-t0} at $t=0$ requires only a categorical
draw $b^*\sim\mathrm{Cat}(w_1,\ldots,w_K)$ followed by a Gaussian
sample, with no learned function evaluation. The low-rank factor lets each
draw run in $O(dr)$ via
$v\gets\mu_{b^*}-x_0+L_{b^*}z+\sigma_{b^*}\xi$ with
$z\in\R^r,\xi\in\R^d$ standard normal.

\subsection{Stage III: correction flow with a surrogate source}
\label{sec:stageIII}

Stage~II samples from $\pisur$ track $\pi$ only up to the surrogate
gap $W_2(\rho_1,\core)$, which is still nonzero on realistic image
data. Stage~III closes the gap with a second flow-matching stage
\emph{in velocity space}: we train a correction network whose source
is built from the closed-form surrogate law $\pisur(\cdot\mid C)$, which is
the source law used at inference, and we run the two flows as a nested
sampler at inference.

\label{sec:correction-flow}
Given $(x_0,x_1)$ with $X_t=(1-t)x_0+tx_1$ and $V_1=x_1-x_0$, each
training step draws a source velocity $V_0$ from the mixture
representation of $\pisur(\cdot\mid X_t,t)$ (the two-step draw is given
after~\eqref{eq:fm-loss}), picks $\tau\sim\mathcal{U}[0,1]$
and forms $V_\tau=(1-\tau)V_0+\tau V_1$, and regresses a correction
net $f_\theta(v,\tau,x_t,t)$ toward $V_1-V_0$:
\begin{equation}
  \mathcal{L}(\theta)=\E\big\|f_\theta(V_\tau,\tau,X_t,t)-(V_1-V_0)\big\|^2.
\label{eq:fm-loss}
\end{equation}
Equation~\eqref{eq:fm-loss} has the same form as HRF2's inner
objective, but the source $V_0$ is built from the informative
$\pisur(\cdot\mid C)$ rather than $\mathcal{N}(0,I)$, so the
learnable residual is a surrogate-to-target correction rather than a
noise-to-data transport.

\label{sec:coupling}
The classical flow-matching choice for the source $V_0$ is the
independent isotropic Gaussian, $V_0\sim\mathcal{N}(0,I_d)$ drawn
independently of $V_1\sim\pi$, which yields a conditional residual
$\E[\|V_1-V_0\|^2\mid X_t,t]=\operatorname{tr}\Sigma_\pi(X_t,t)+
\|\mu_\pi(X_t,t)\|^2+d$, lower-bounded by $c_\gamma D_\pi^2+d$ whenever
$\pi(\cdot\mid X_t,t)$ is multimodal with cross-mode diameter $D_\pi$
(see Section \ref{sec:a7}). This full cross-mode variance is what
the correction network would have to absorb, and is exactly the
transport task CCVFM was designed to avoid.

{We instead draw $V_0$ through the mixture representation of
Eq.~\eqref{eq:pisur}. With $C=(X_t,t)$, a component label $B$ is drawn
first, followed by the Gaussian step
$V_0\mid B=b,C\sim\mathcal N\bigl(m_b(C),\Lambda_b(t)\bigr)$.
When the label is drawn from the surrogate weights, independently of $V_1$,
$
  B\mid C\sim
  \mathrm{Cat}\bigl(\gamma_1(C),\ldots,\gamma_K(C)\bigr)$,
marginalizing over the component label gives
\[
  p(V_0=v\mid C)
  =
  \sum_{b=1}^K
  \gamma_b(C)
  \mathcal N\bigl(v;m_b(C),\Lambda_b(t)\bigr)
  =
  \pisur(v\mid C).
\]
We call this the \emph{direct surrogate-source sampler}: its source
marginal is exactly the surrogate law $\pisur(\cdot\mid C)$, the same
source law used at inference, and it is the draw analyzed in
Theorem~\ref{thm:posterior-coupling}; other label draws keep the same
Gaussian step (Remark~\ref{rem:general-label}).
Our image experiments anchor the label to the training image; the
theoretical properties of related anchored couplings are analyzed in
\citet{wang2026optimal}.
At the generation boundary $t=0$, this reduces to
\[
  B\sim\mathrm{Cat}(w_1,\ldots,w_K),
  \qquad
  V_0\mid B=b,X_0=x_0
  \sim
  \mathcal N(\mu_b-x_0,\Sigma_b),
\]
so the source law at $t=0$ is $\pisur(v\mid x_0,0)=\core(x_0+v)$ exactly.
The resulting correction flow is trained from a data-adapted surrogate
source rather than from an isotropic Gaussian source.

}

\subsection{Generation procedure}
\label{sec:generation}

Our theory (\S \ref{sec:theory}) analyzes the $J=1$ case at
$t=0$, where $\pi(v\mid x_0,0)=\rho_1(x_0+v)$ collapses the
conditional velocity law to a translated copy of the target. We
therefore recommend $J=1$ as the default, with the inner correction
budget $L$ the only knob; this estimator enjoys the closed-form
guarantees of Theorem~\ref{thm:reduction}.

\paragraph{One-step recommended sampler ($J=1$).}
Given the trained correction field $f_\theta$ and a budget $L\ge 1$:
(1) draw $X_0\sim\mathcal{N}(0,I_d)$; (2) sample
$b^*\sim\mathrm{Cat}(w_1,\ldots,w_K)$ and
$V\sim\mathcal{N}(\mu_{b^*}-X_0,\Sigma_{b^*})$, which realizes
$V\sim\pisur(\cdot\mid X_0,0)$ via~\eqref{eq:pisur-t0};
(3) for $\ell=0,\ldots,L-1$:
$V\gets V+(1/L)\,f_\theta(V,\ell/L,X_0,0)$;
(4) return $\hat X_1\gets X_0 + V$.
{The Stage~II step (2) is a categorical draw plus a Gaussian
sample, with no learned-network evaluation; we count it as a single
zero-parameter NFE for benchmarking equivalence with diffusion baselines
that include the initial denoising step. Hence total
$\mathrm{NFE}=L+1$, used uniformly across all CCVFM FID rows in
Section~\ref{sec:experiments} and matching Tables~\ref{tab:mnist}
and~\ref{tab:hd-images}.}

\paragraph{$J=1$ as data-space flow matching.}
With $Y_0=X_0+V_0$ and $Y_1=X_1$, the $t=0$ slice
of~\eqref{eq:fm-loss} is flow matching in data space from the law of $Y_0$
to $\rho_1$ (with $X_0$ as an extra input); since $(Y_0,Y_1)$ is independent
of $X_0$, its population optimum coincides with that of data-space flow
matching under the same endpoint coupling. The two views differ only at
the estimator level, and Theorem~\ref{thm:reduction} depends only on the
source marginal (Appendix~\S\ref{sec:j1-reduction-supp}).

\paragraph{General $J>1$ (optional).}
Because $\pisur(\cdot\mid x,t)$ is closed-form at every $t$, the same
network also defines a nested sampler on an outer grid
$0=t_0<\cdots<t_J=1$ with total $\mathrm{NFE}=J(L+1)$
(Appendix~\S B). This variant is \emph{not} covered by the
theory; for the networks trained over $t\in[0,1]$ reported here,
$J=1$ matched or beat $J\ge 2$ on every dataset we tested
(Table~\ref{tab:nfe-sweep}). With a network trained only on the
slices $t\in\{0,\tfrac12\}$ that $J=2$ visits, $J=2$ improves MNIST FID
at equal learned-network evaluations (Table~\ref{tab:j2-matched}).
Characterizing when $J>1$ pays off is left to future work.
Pseudocode for both variants is in Appendix~\S B.

\section{Theory}
\label{sec:theory}


{Our theoretical analysis has two parts. \S\ref{sec:theory-surrogate}
quantifies the conditional transport task that the correction net
must solve at $t=0$, contrasting the deployed Stage~I source with
the HRF2 noise source. \S\ref{sec:theory-coupling} verifies the exact
source marginal of the direct surrogate-source sampler of
\S\ref{sec:stageIII} and compares its conditional training-target scale with that
of an independent Gaussian source. We analyze the $J=1$ sampler of
\S\ref{sec:generation}, where the velocity identity collapses and
computations become closed-form.}

Define the \emph{conditional
transport cost} of a correction stage with source $Q(\cdot\mid x,t)$
at $t=0$ by
\begin{equation}
  D(Q) \;:=\; \Big(\E_{X_0\sim\rho_0}\!\big[\,W_2^2\big(Q(\cdot\mid X_0,0),\pi(\cdot\mid X_0,0)\big)\big]\Big)^{1/2}.
\label{eq:defD}
\end{equation}
All proofs are deferred to Appendix~\S A.

\begin{assumption}[Stage-I compression for the deployed surrogate]
\label{ass:stageI-compression}
Let $\mu_{\mathrm{in}}$ denote the probability measure compressed by
Stage~I, and let $
  \core=\sum_{k=1}^K w_k\,\mathcal{N}({\mu_k},\Sigma_k) $
be the $K$-component surrogate returned by the deployed
entropic-Sinkhorn/GMM procedure. We assume that there exists a constant
$C_Q<\infty$, independent of $K$, such that
\begin{equation}
\label{eq:stageI-compression}
  W_2(\mu_{\mathrm{in}},\core)\le C_QK^{-1/d}.
\end{equation}
Here $W_2$ is the Euclidean $2$-Wasserstein distance on $\R^d$, and we
use the assumption only in the population surrogate-gap statement
(Theorem~\ref{thm:reduction}(i)) and in the $K$-rate of Corollary~\ref{cor:nfe-gap}, where $\mu_{\mathrm{in}}=\rho_1$.
This is an assumption on the
implemented Stage-I output. Classical quantization theory gives the
corresponding $K^{-1/d}$ benchmark for globally optimal $K$-atomic
quantizers on bounded $d$-dimensional domains, but it does not by
itself prove \eqref{eq:stageI-compression} for finite-iteration
entropic Sinkhorn followed by a GMM lift.
\end{assumption}

The nominal exponent $d$ makes~\eqref{eq:stageI-compression}
pessimistic for images; for data near a lower-dimensional set we read it
with an effective dimension $d_{\mathrm{eff}}\ll d$. A transport-cost
proxy gives $d_{\mathrm{eff}}\approx 12$ for MNIST and $\approx 15$ for
CIFAR-10 (an empirical diagnostic, Appendix~\S D).

\subsection{Bounding the surrogate gap}
\label{sec:theory-surrogate}

At $t=0$, \eqref{eq:true-conditional} simplifies to
$\pi(v\mid x_0,0)=\rho_1(x_0+v)$: the conditional velocity law is
the target shifted by $-x_0$. For any surrogate of the form
$Q(v\mid x_0,0)=q(x_0+v)$, the Wasserstein distance to $\pi(\cdot\mid x_0,0)$
equals $W_2(q,\rho_1)$ and is \emph{independent of $x_0$}, so
CCVFM and HRF2 can be compared by a single scalar per source.

\begin{theorem}[Transport-task reduction]
\label{thm:reduction}
Let $Q_{\mathrm{CCVFM}}(v\mid x_0,0)=\core(x_0+v)$ and
$Q_{\mathrm{HRF2}}(v\mid x_0,0)=\mathcal{N}(v;0,I_d)$. Then
\begin{enumerate}[label=\emph{(\roman*)},leftmargin=*,itemsep=0pt]
  \item $D(Q_{\mathrm{CCVFM}}) = W_2(\rho_1,\core)$. Under Assumption~\ref{ass:stageI-compression} with $\mu_{\mathrm{in}}=\rho_1$, $D(Q_{\mathrm{CCVFM}})\le C_QK^{-1/d}$.
  \item $D(Q_{\mathrm{HRF2}}) \;\ge\; \sqrt{d}\cdot{\big(\sqrt{\sigma_1^2+1}-1\big)}$,
        where $\sigma_1^2=\E\|X_1\|^2/d$ is the per-coordinate second moment
        of $\rho_1$. This lower bound is strictly positive whenever
        $\sigma_1\neq 0$ and \emph{constant in $K$ and $n$}.
\end{enumerate}
Consequently, whenever the coreset error tends to zero with $K$,
$D(Q_{\mathrm{CCVFM}})/D(Q_{\mathrm{HRF2}})\to 0$ as
$K\to\infty$ for every fixed nondegenerate target: CCVFM's correction net
is asked to solve a transport task that is arbitrarily smaller than
HRF2's.
\end{theorem}

{Theorem~\ref{thm:reduction} certifies that, asymptotically,
CCVFM's correction net is asked to solve a transport task that vanishes
with $K$, while HRF2's correction-source cost remains bounded below by
a positive constant in $K$ and $n$.}

\begin{corollary}[NFE sufficient scaling]
\label{cor:nfe-gap}
Let $D:=D(Q)$, and let $M$ bound the $L^2$ norm of the correction field
along the exact trajectory, and $L_v$ and $L_\tau$ be its Lipschitz
constants in $v$ and in $\tau$. Under the space-time regularity assumptions in
Appendix~\S A.4 and the residual flattening conditions
$M=O(D)$, $L_v=O(D)$, and $L_\tau=O(D^2)$ for the correction ODE, a
sufficient number of Euler function evaluations to reduce the
$L^2$ discretization error below $\varepsilon_{\mathrm{disc}}$ is
\[
  \mathrm{NFE}_{Q}(\varepsilon_{\mathrm{disc}})
  =O\!\left(1+\frac{D(Q)^2}{\varepsilon_{\mathrm{disc}}}\right).
\]
Consequently, for CCVFM, with
$\Delta_K:=D(Q_{\mathrm{CCVFM}})=W_2(\rho_1,\core)$,
\[
  \mathrm{NFE}_{\mathrm{CCVFM}}(\varepsilon_{\mathrm{disc}})
  =O\!\left(1+\frac{\Delta_K^2}{\varepsilon_{\mathrm{disc}}}\right).
\]
If, in addition, Assumption~\ref{ass:stageI-compression} gives
$\Delta_K=O(K^{-1/d})$, then
  $\mathrm{NFE}_{\mathrm{CCVFM}}(\varepsilon_{\mathrm{disc}})
  =O\!\left(1+\frac{K^{-2/d}}{\varepsilon_{\mathrm{disc}}}\right)$.
For HRF2, if $d$ and $\rho_1$ are fixed and the corresponding
regularity constants are $O(1)$, then
  $\mathrm{NFE}_{\mathrm{HRF2}}(\varepsilon_{\mathrm{disc}})
  =O\!\left(1+\frac{1}{\varepsilon_{\mathrm{disc}}}\right)$.
\end{corollary}
Corollary~\ref{cor:nfe-gap} is a conditional sufficient result: the
flattening conditions are assumptions on the learned field, not
consequences of a small $D$ (Remark~\ref{rem:Lf-D}). It is consistent
with CCVFM's empirical behavior: matched architectures but far fewer correction steps to
reach the same FID when the learned residual field is flatter for the
data-informed source. Without the flattening condition, the generic
sufficient bound is $O(1+(L_vD+L_\tau)/\varepsilon_{\mathrm{disc}})$ with the
regularity constants treated as problem-dependent quantities.

{
{\subsection{Exact surrogate-source sampling and residual scale}
\label{sec:theory-coupling}

Stage~III trains the correction net $f_\theta$ by regressing on the
residual $V_1-V_0$ at random interpolation time $\tau$
(Eq.~\eqref{eq:fm-loss}). The conditional second moment
$\E[\|V_1-V_0\|^2\mid C]$ with $C=(X_t,t)$ controls the scale of the
flow-matching regression target. Under the direct surrogate-source
sampler of \S\ref{sec:stageIII}, $V_0$ is drawn from the surrogate law
$\pisur(\cdot\mid C)$ itself, independently of $V_1$. 

\begin{theorem}[Direct surrogate-source second moment vs.\ independent Gaussian source]
\label{thm:posterior-coupling}
Fix $C=(X_t,t)$ and let $V_1\sim\pi(\cdot\mid C)$. Define $\mu_\pi(C):=\E[V_1\mid C],  \Sigma_\pi(C):=\operatorname{Cov}(V_1\mid C),
S_\pi(C):=\operatorname{tr}\Sigma_\pi(C)$. Let the surrogate law be $
  \pisur(v\mid C)
  =
  \sum_{b=1}^K
  \gamma_b(C)
  \mathcal N\bigl(v;m_b(C),\Lambda_b(t)\bigr)$.
Define its conditional mean by
$
  \mu_{\mathrm{sur}}(C)
  :=
  \sum_{b=1}^K\gamma_b(C)m_b(C),
$
and decompose its conditional covariance trace as
$
  \operatorname{tr}\Sigma_{\mathrm{sur}}(C)
  =
  A_{\mathrm{sur}}(C)+B_{\mathrm{sur}}(C),
$
where
$
  A_{\mathrm{sur}}(C)
  :=
  \sum_{b=1}^K
  \gamma_b(C)\operatorname{tr}\Lambda_b(t)
$
is the average within-component covariance scale, and
$
  B_{\mathrm{sur}}(C)
  :=
  \sum_{b=1}^K
  \gamma_b(C)
  \|m_b(C)-\mu_{\mathrm{sur}}(C)\|^2
$
is the weighted between-component spread of the surrogate source.
Equivalently,
$B_{\mathrm{sur}}(C)
  =
  \frac12
  \sum_{b=1}^K\sum_{b'=1}^K
  \gamma_b(C)\gamma_{b'}(C)
  \|m_b(C)-m_{b'}(C)\|^2$.
Finally, let $M_{\mathrm{sur}}(C)
  :=
  \|\mu_\pi(C)-\mu_{\mathrm{sur}}(C)\|^2$.

\textbf{(i) Direct surrogate-source sampler.}
If
\[
  B\mid C\sim
  \mathrm{Cat}\bigl(\gamma_1(C),\ldots,\gamma_K(C)\bigr),
  \qquad
  V_0\mid B=b,C
  \sim
  \mathcal N\bigl(m_b(C),\Lambda_b(t)\bigr),
\]
and this draw is conditionally independent of $V_1$ given $C$, then $V_0\mid C\sim\pisur(\cdot\mid C)$, and
\begin{equation}
\label{eq:surrogate-second-moment-main}
  \E\bigl[\|V_1-V_0\|^2\,\big|\,C\bigr]_{\mathrm{sur}}
  =
  S_\pi(C)
  +
  A_{\mathrm{sur}}(C)
  +
  B_{\mathrm{sur}}(C)
  +
  M_{\mathrm{sur}}(C).
\end{equation}

\textbf{(ii) Independent Gaussian source.}
If $V_0\sim\mathcal N(0,I_d)$ independently of $V_1$ given $C$, then
\begin{equation}
\label{eq:indep-gauss-second-moment-main}
  \E\bigl[\|V_1-V_0\|^2\,\big|\,C\bigr]_{\mathrm{Gauss}}
  =
  S_\pi(C)
  +
  \|\mu_\pi(C)\|^2
  +
  d .
\end{equation}
Moreover, suppose $\pi(\cdot\mid C)$ admits a latent mixture
representation with component label $Z$, weights
$\gamma_b^\pi(C):=\Pr(Z=b\mid C)$, conditional component means
$m_b^\pi(C):=\E[V_1\mid Z=b,C]$, and cross-mode diameter $ D_\pi(C):=\max_{b,b'}\|m_b^\pi(C)-m_{b'}^\pi(C)\|$.
Let $c_\gamma(C):=\gamma_{b_-}^\pi(C)\gamma_{b_+}^\pi(C)$ for a pair
$(b_-,b_+)$ attaining $D_\pi(C)$. Then
\begin{equation}
\label{eq:indep-lower-main}
  \E\bigl[\|V_1-V_0\|^2\,\big|\,C\bigr]_{\mathrm{Gauss}}
  \ge
  c_\gamma(C)D_\pi^2(C)+d .
\end{equation}
\end{theorem}

Theorem~\ref{thm:posterior-coupling} gives an exact identity for the
conditional training-target scale under the direct surrogate-source
sampler, whose source marginal is exactly $\pisur(\cdot\mid C)$. 
The same conditioning argument applies to other designs of the label
draw: whenever $V_0\mid(B,C,V_1)\sim\mathcal N(m_B(C),\Lambda_B(t))$,
one has $\E[\|V_1-V_0\|^2\mid C]=\E[\|V_1-m_B(C)\|^2+\operatorname{tr}\Lambda_B(t)\mid C]$,
which reduces to~\eqref{eq:surrogate-second-moment-main} when $B$ is drawn
from $\gamma(C)$ independently of $V_1$ (Remark~\ref{rem:general-label}).

The total second-moment ratio is
\[
  \mathcal R_{\mathrm{tot}}(C)
  :=
  \frac{
  \E[\|V_1-V_0\|^2\mid C]_{\mathrm{sur}}
  }{
  \E[\|V_1-V_0\|^2\mid C]_{\mathrm{Gauss}}
  }
  =
  \frac{
  S_\pi(C)
  +
  A_{\mathrm{sur}}(C)
  +
  B_{\mathrm{sur}}(C)
  +
  M_{\mathrm{sur}}(C)
  }{
  S_\pi(C)+\|\mu_\pi(C)\|^2+d
  } .
\]
Thus the direct surrogate source gives a smaller total training target
than the independent Gaussian source whenever
$
  A_{\mathrm{sur}}(C)
  +
  B_{\mathrm{sur}}(C)
  +
  M_{\mathrm{sur}}(C)
  <
  \|\mu_\pi(C)\|^2+d$.

Equivalently, the surrogate source is beneficial when its
within-component covariance, its weighted between-component spread, and
its conditional mean mismatch are jointly smaller than the source
penalty incurred by the isotropic Gaussian source.

It is useful to isolate the source-dependent part of the comparison.
Since the intrinsic conditional spread $S_\pi(C)$ appears in both
couplings, define the excess ratio
\[
  \mathcal R_{\mathrm{ex}}(C)
  :=
  \frac{
  \E[\|V_1-V_0\|^2\mid C]_{\mathrm{sur}}-S_\pi(C)
  }{
  \E[\|V_1-V_0\|^2\mid C]_{\mathrm{Gauss}}-S_\pi(C)
  }
  =
  \frac{
  A_{\mathrm{sur}}(C)
  +
  B_{\mathrm{sur}}(C)
  +
  M_{\mathrm{sur}}(C)
  }{
  \|\mu_\pi(C)\|^2+d
  } .
\]
This excess ratio becomes small when the surrogate conditional source is
concentrated near the true conditional velocity law:
\[
  A_{\mathrm{sur}}(C)\ll \|\mu_\pi(C)\|^2+d,
  \qquad
  B_{\mathrm{sur}}(C)\ll \|\mu_\pi(C)\|^2+d,
  \qquad
  M_{\mathrm{sur}}(C)\ll \|\mu_\pi(C)\|^2+d .
\]
In the extreme case $\core=\rho_1$ at $t=0$, the surrogate excess equals
$\operatorname{tr}\operatorname{Cov}(X_1)\le d$ under standardization or
$[-1,1]$ scaling, so it never exceeds the Gaussian excess
(Remark~\ref{rem:thm3-preprocessing}). These statements concern the
direct sampler; a Monte Carlo check on seven datasets matches both
closed forms (Appendix~\S D).
}

\section{Experiments}
\label{sec:experiments}

We evaluate CCVFM on four image datasets: three in raw pixel space, MNIST ($28\times 28$), CIFAR-10 ($32\times 32\times 3$), and
ImageNet-32 \citep{chrabaszcz2017downsampled} ($32\times 32\times 3$
with $1{,}000$ classes), all at matched U-Net architecture with HRF2, and one high-resolution latent-space benchmark, CelebA-HQ 256
($256\times 256\times 3$, encoded by DC-AE f32c32 to a $d{=}2048$
latent), where the correction net is a DiT-L. All
FID values use InceptionV3 pool features~\citep{heusel2017gans};
training details, hyperparameter sweeps, and toy 2D results are in
the Appendix~\S C, D and E.

\paragraph{FID protocol.} We report two FID variants:
$\mathrm{FID}_{50k}$ uses $50$k generated samples against $50$k
\emph{training} images (the protocol of \citealp[Table 7]{zhang2025hrf}),
and $\mathrm{FID}_{10k}$ uses $10$k generated samples against the
$10$k held-out \emph{test} set. Reported numbers are single-seed
point estimates 
. On MNIST, five sampling seeds give FID$_{50k}=0.772\pm0.011$ at
$51$ NFE and a retrained model gives $0.747$ (Appendix~\S D); the other
benchmarks are single training runs.

\subsection{MNIST: FID and goodness-of-fit diagnostics}
\label{sec:mnist}

\begin{table}[t]
\centering
\small
\setlength{\tabcolsep}{4pt}
\caption{MNIST pixel-space FID. CCVFM uses $K=2000$, rank $r=50$,
U-Net base width $128$, $200$k training iterations, EMA $0.9999$.
FID$_{50k}$ uses $50$k training images as the reference (HRF2 Table
7 protocol); FID$_{10k}$ uses the $10$k held-out test set. RF and
HRF2 baselines from \citet[Table 7]{zhang2025hrf}; DDPM and FM rows
($\star$) report $5$k-reference numbers from their original papers
for historical context.}
\label{tab:mnist}
\begin{tabular}{lrrr}
\toprule
Method & NFE & FID$_{10k}\downarrow$ & FID$_{50k}\downarrow$ \\
\midrule
Rectified Flow \citep{liu2023rectified}             & 10   & ---   & 7.974 \\
Rectified Flow                                      & 100  & ---   & 5.563 \\
HRF2 \citep{zhang2025hrf}                           & 10   & ---   & 6.644 \\
HRF2                                                & 100  & ---   & 2.588 \\
HRF2                                                & 500  & ---   & 2.574 \\
DDPM \citep{ho2020ddpm}$^{\star}$                   & 1000 & ---   & 12.80 \\
Flow Matching \citep{lipman2023flow}$^{\star}$      & 100  & ---   & 10.20 \\
\midrule
\textbf{CCVFM Stage III}, $L{=}5$        & 6   & 13.80 & 12.45 \\
\textbf{CCVFM Stage III}, $L{=}10$       & 11  & 4.05  & 2.88 \\
\textbf{CCVFM Stage III}, $L{=}20$       & 21  & \textbf{2.26}  & \textbf{1.09} \\
\textbf{CCVFM Stage III}, $L{=}50$       & 51  & \textbf{1.91}  & \textbf{0.75} \\
\bottomrule
\end{tabular}
\end{table}

\paragraph{Scaling law and HRF2 comparison.}
We trained three configurations with $K\in\{500,1000,2000\}$. At
$K{=}500$ ($40$k iterations, U-Net base $64$), CCVFM reaches
$\mathrm{FID}_{50k}=3.79$ at $11$ NFE; at $K{=}1000$ ($80$k
iterations, base $64$), $1.93$ at $11$ NFE; at the headline $K{=}2000$ ($200$k
iterations, base $128$, rank $50$), $\mathbf{1.09}$ at $21$ NFE and
$\mathbf{0.75}$ at $51$ NFE (Table~\ref{tab:mnist}). The improvement in the best
FID as $K$ and the training budget grow is qualitatively consistent with the Stage-I
compression assumption. Our $L{=}20$ result improves on HRF2's
best ($2.574$ at $500$ NFE) by $58\%$ with $\approx 24\times$ fewer
NFE; our $L{=}50$ result is $\approx 3\times$ lower at $\approx 10\times$
fewer NFE. Sample grids are in Figure~\ref{fig:samples} and
Appendix~\S F.

\paragraph{Goodness of fit: the generator learns $\rho_1$, not the training set.}
\label{sec:gof}

A low FID alone cannot distinguish a generator that has learned the
target law from one that has memorized its training set. To rule
memorization out, we compare three pools of $10{,}000$ samples each:
a \emph{Generated} pool (CCVFM with $K{=}2000$ and $L{=}20$,
$\mathrm{FID}_{50k}{=}1.09$),
the $10{,}000$-image held-out \emph{Test} set, and a size-matched
$10{,}000$-image random subset of the \emph{Training} set. For each
ordered pair $(A,B)$ we compute the nearest-neighbour distance
$d_{A\!\to\!B}(a)=\min_{b\in B}\|\phi(a)-\phi(b)\|$ for each $a\in A$,
where $\phi$ is the InceptionV3 pool embedding. Comparing the
distributions of these distances across pairs is a direct,
non-parametric goodness-of-fit probe; the absence of a memorization
signal would manifest as symmetry (KS and Wasserstein-1 near zero)
between the Generated-vs-Training and Generated-vs-Test comparisons.

\begin{table}[t]
\centering
\small
\setlength{\tabcolsep}{5pt}
\caption{Goodness-of-fit summary on MNIST generations from the
headline configuration
(Inception feature space, size-matched pools of $10$k samples,
$1$-nearest-neighbour). KS is the Kolmogorov--Smirnov statistic and
$W_1$ the $1$-Wasserstein distance between the two distance
distributions named in ``Pair A'' and ``Pair B''.
\emph{Memorization tests} (rows $3$, $4$) compare generated-to-train
against generated-to-test distances. The real$\leftrightarrow$real
baseline (row $5$) bounds how small these gaps can be even for two
i.i.d.\ draws from $\rho_1$.}
\label{tab:gof-knn}
\begin{tabular}{lccl}
\toprule
Pair A vs.\ Pair B & KS$\downarrow$ & $W_1\downarrow$ & Interpretation \\
\midrule
Ge$\to$Te vs.\ Te$\to$Ge              & 0.0306 & 0.0728 & fidelity symmetry (Ge vs.\ Test) \\
Ge$\to$Tr vs.\ Tr$\to$Ge              & 0.0248 & 0.0629 & fidelity symmetry (Ge vs.\ Train) \\
\textbf{Ge$\to$Tr vs.\ Ge$\to$Te}     & \textbf{0.0173} & \textbf{0.0420} & \textbf{no memorization signal} \\
\textbf{Tr$\to$Ge vs.\ Te$\to$Ge}     & \textbf{0.0238} & \textbf{0.0535} & \textbf{no dual memorization} \\
Te$\to$Tr vs.\ Tr$\to$Te              & 0.0083 & 0.0116 & real $\leftrightarrow$ real baseline \\
\bottomrule
\end{tabular}
\end{table}

\textbf{What the tables say.} The memorization tests
(Table~\ref{tab:gof-knn}, rows 3--4) give $\mathrm{KS}\le0.024$,
$W_1\le0.054$: Generated-to-Train and Generated-to-Test distance
distributions are essentially the same. The real$\leftrightarrow$real
baseline (row 5, $\mathrm{KS}=0.008$, $W_1=0.012$) is the floor for
two i.i.d.\ draws from $\rho_1$; the memorization test sits within a
small factor of this floor, so the generated pool is distributed as
a third i.i.d.\ draw, not preferentially closer to training.
Improved precision/recall \citep{kynkaanniemi2019improved} computed
in Inception feature space with $k{=}5$ are also indistinguishable
across references ($\Delta P=0.008$, $\Delta R=0.004$; full table in
Appendix~\S G). The three diagnostic panels (1-NN density, per-sample
Generated-to-Train vs.\ Generated-to-Test scatter, improved
precision/recall) are visualized in Appendix~\S G; pixel-space
duplicates and CDF figures are also in Appendix~\S G.

\subsection{High-dimensional images: CIFAR-10, ImageNet-32, and CelebA-HQ 256}
\label{sec:hd-images}

\begin{figure}[!htb]
  \centering
  \setlength{\abovecaptionskip}{2pt}
  \setlength{\belowcaptionskip}{0pt}
  \setlength{\tabcolsep}{2pt}
  \begin{tabular}{@{}ccc@{}}
    \includegraphics[width=.32\linewidth]{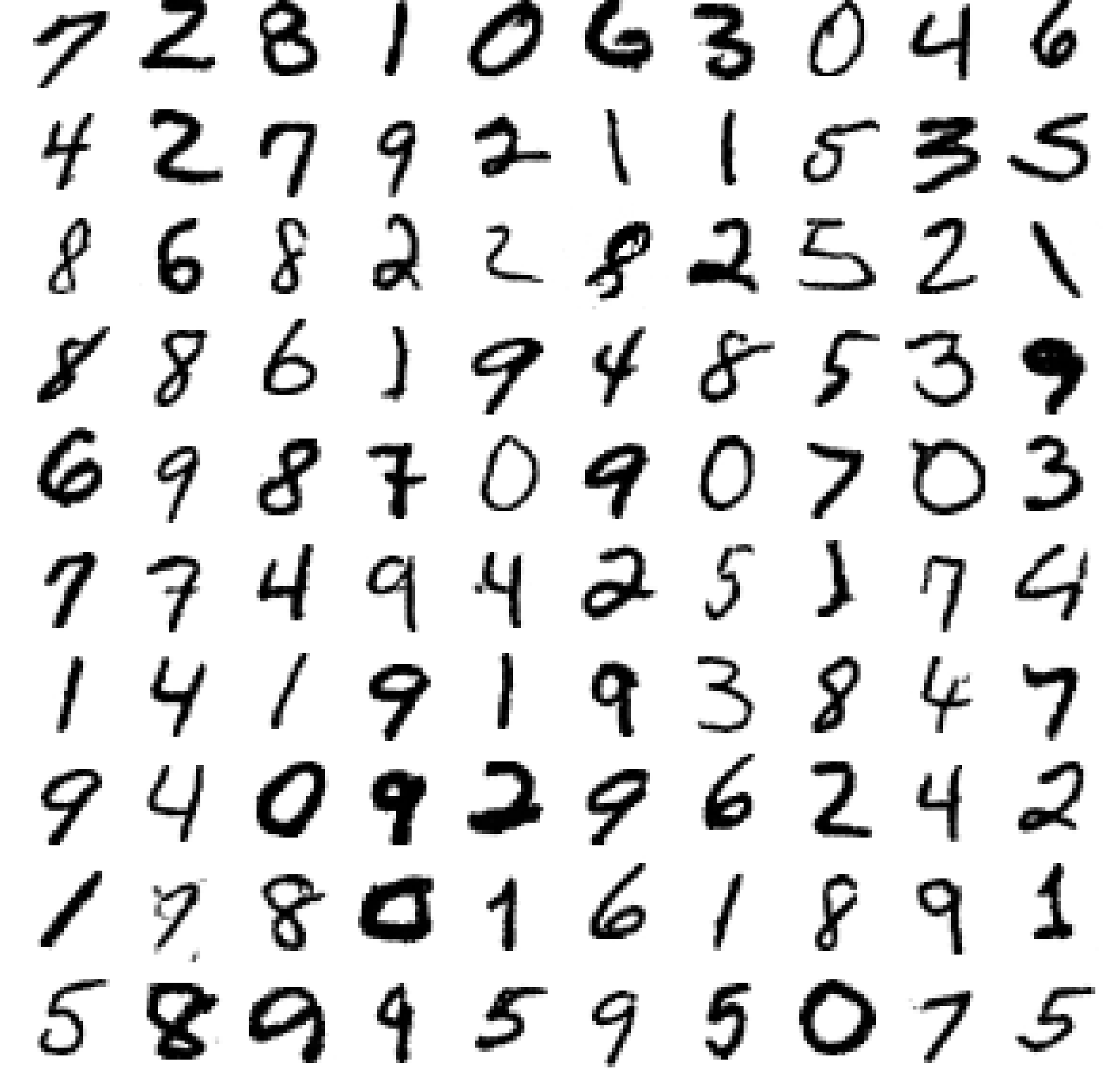} &
    \includegraphics[width=.32\linewidth]{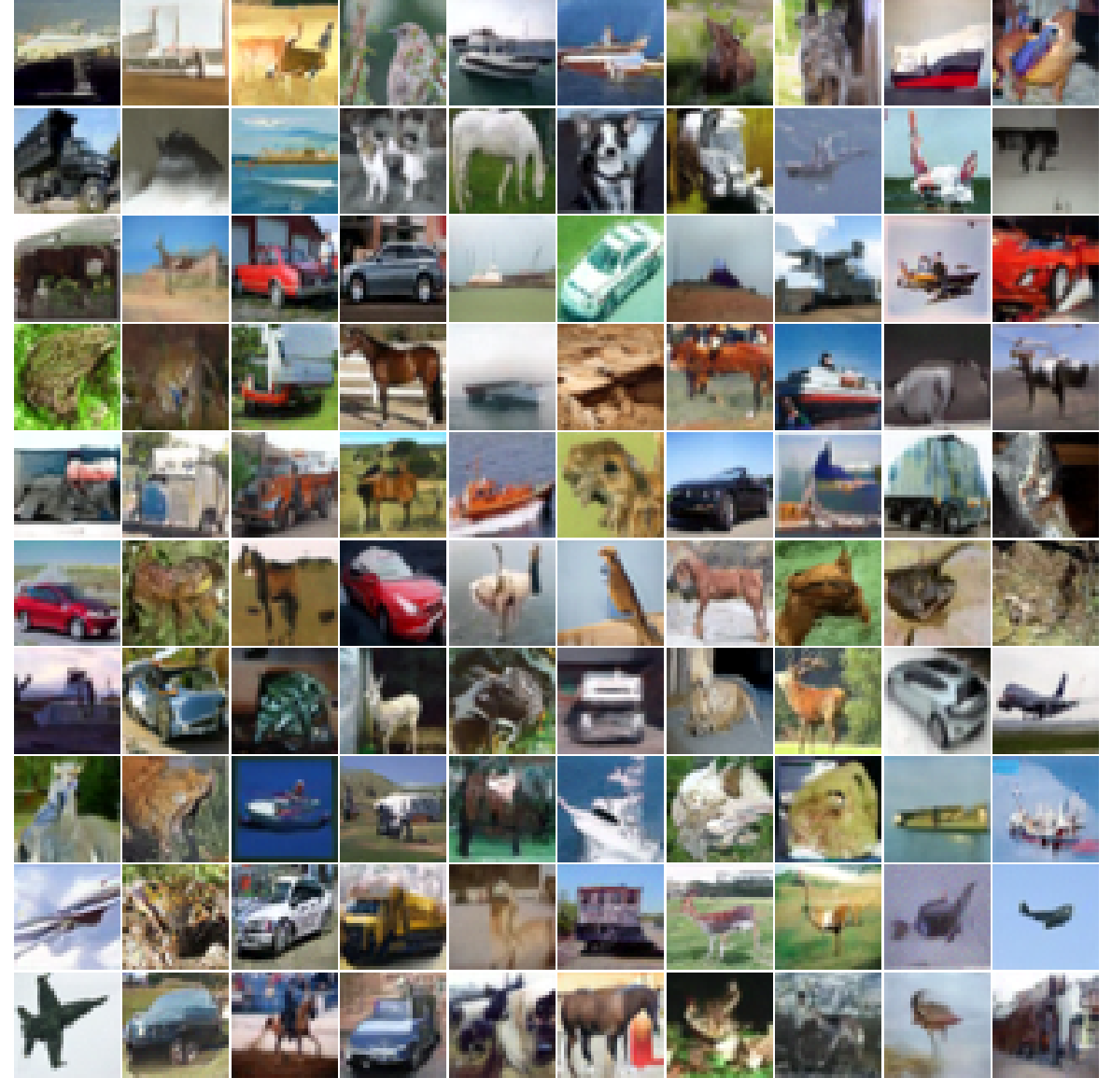} &
    \includegraphics[width=.32\linewidth]{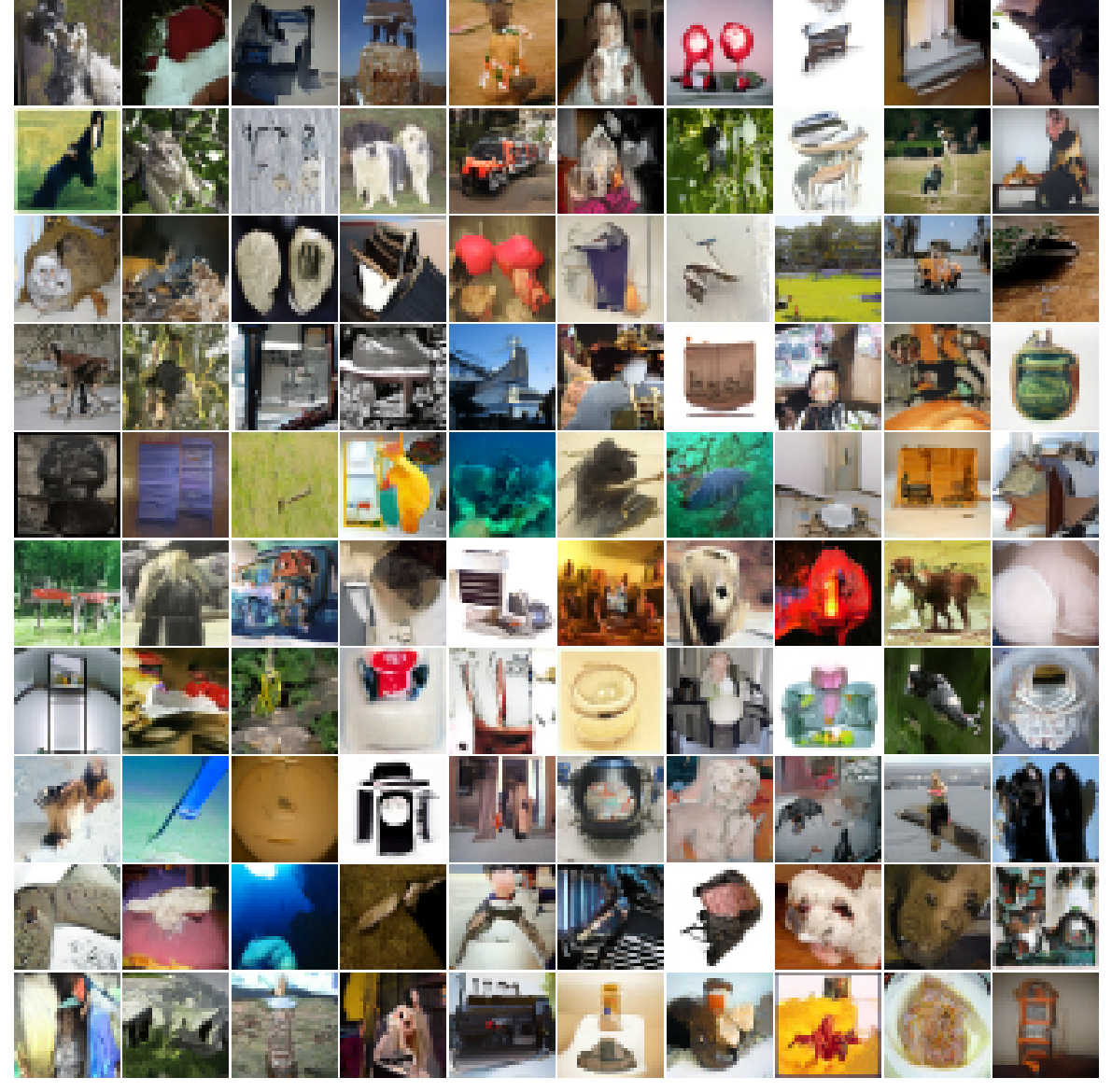} \\[-2pt]
    {\scriptsize (a) MNIST, $21$ NFE, FID$_{50k}{=}1.09$} &
    {\scriptsize (b) CIFAR-10, $51$ NFE, FID$_{50k}{=}6.35$} &
    {\scriptsize (c) ImageNet-32, $51$ NFE, FID$_{50k}{=}8.76$} \\[2pt]
    \multicolumn{3}{c}{\includegraphics[width=.97\linewidth]{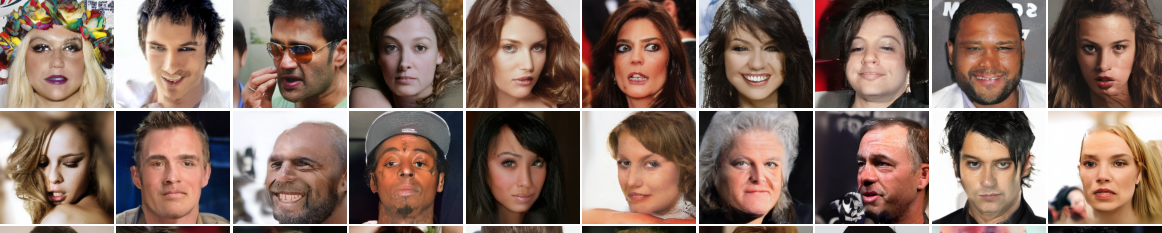}} \\[-2pt]
    \multicolumn{3}{c}{\scriptsize (d) CelebA-HQ 256, $51$ NFE, FID$_{28k}{=}4.17$ (uncurated $2{\times}5$ panel)}
  \end{tabular}
  \caption{Uncurated samples used for the reported FID pools: first $100$ samples for (a)--(c), and a $2{\times}5$ CelebA-HQ panel from the seed-99 30k pool for (d).}
  \label{fig:samples}
\end{figure}

We evaluate CCVFM on three pixel-/latent-space image benchmarks: CIFAR-10 ($32{\times}32{\times}3$,
$d{=}3072$, $n{=}50$k), ImageNet-32 ($32{\times}32{\times}3$,
$d{=}3072$, $n{\approx}1.28$M), and CelebA-HQ 256 ($256{\times}256{\times}3$,
encoded by DC-AE f32c32 to a $d{=}2048$ latent, $n{\approx}28$k).
Table~\ref{tab:hd-images} reports the three FID sweeps side by side.

\begin{table}[t]
\centering
\small
\caption{FID on 3 high-dimensional benchmarks.
\emph{(a)} CIFAR-10 FID$_{50k}$ at $K{=}10000$, $r{=}80$; the
matched $K{=}5000$ sweep is in Appendix~\S D; $\dagger$
denotes numbers not directly comparable (different protocol).
\emph{(b)} ImageNet-32 FID$_{50k}$ (matched U-Net, same
protocol; HRF2 from \citet{zhang2025hrf}).
\emph{(c)} CelebA-HQ 256 FID$_{28k}$ (28k generated vs. all $\sim$28k
real CelebA-HQ images) at $K{=}10000$, $r{=}128$, DiT-L correction
($\sim$458M params), 400k training steps, no classifier-free
guidance.}
\label{tab:hd-images}
\begin{minipage}[t]{0.32\linewidth}\centering
\footnotesize
\setlength{\tabcolsep}{1pt}
\begin{tabular}{@{}lrr@{}}
\toprule
Method & NFE & FID$_{50k}$ \\
\midrule
RF \citep{liu2023rectified}$^{\dagger}$    & 1     & 378.0 \\
2-RF \citep{liu2023rectified}$^{\dagger}$  & 1     & 12.21 \\
DDPM \citep{ho2020ddpm}                    & 1000  & 3.17  \\
FM \citep{lipman2023flow}                  & adp.  & 6.35  \\
EDM \citep{karras2022edm}                  & 35    & 1.97  \\
CM \citep{song2023consistency}             & 1     & 8.70  \\
\midrule
\textbf{CCVFM}, $L{=}10$ & {11} & 9.66 \\
\textbf{CCVFM}, $L{=}20$ & {21} & 7.24 \\
\textbf{CCVFM}, $L{=}50$ & {51} & \textbf{6.35} \\
\bottomrule
\end{tabular}

\vspace{2pt}
\emph{(a) CIFAR-10.}
\end{minipage}\hfill
\begin{minipage}[t]{0.32\linewidth}\centering
\footnotesize
\setlength{\tabcolsep}{1pt}
\begin{tabular}{@{}lrrr@{}}
\toprule
Method & NFE & FID$_{50k}$ & $\Delta$ \\
\midrule
HRF2                     & {11} & 20.29          & --- \\
\textbf{CCVFM}, $L{=}10$ & {11} & \textbf{12.55} & $-38\%$ \\
\midrule
HRF2                     & {21} & 12.49          & --- \\
\textbf{CCVFM}, $L{=}20$ & {21} & \textbf{9.51}  & $-24\%$ \\
\midrule
HRF2                     & {51} & 9.02           & --- \\
\textbf{CCVFM}, $L{=}50$ & {51} & \textbf{8.76}  & $-3\%$ \\
\bottomrule
\end{tabular}

\vspace{2pt}
\emph{(b) ImageNet-32. $\Delta$ vs.\ HRF2.}
\end{minipage}\hfill
\begin{minipage}[t]{0.32\linewidth}\centering
\footnotesize
\setlength{\tabcolsep}{2pt}
\begin{tabular}{@{}lrr@{}}
\toprule
Method & NFE & FID$_{28k}$ \\
\midrule
DDPM \citep{ho2020ddpm}    & 1000 & 7.03  \\
LSGM                       & 138  & 7.22  \\
ADM                        & 1000 & 5.11  \\
LDM-4                      & 200  & 5.11  \\
EDM \citep{karras2022edm}  & 79   & 3.50  \\
\midrule
CCVFM-B, $L{=}50$        & {51} & 5.59 \\
\textbf{CCVFM-L}, $L{=}50$ & {51} & \textbf{4.17} \\
\bottomrule
\end{tabular}

\vspace{2pt}
\emph{(c) CelebA-HQ 256.}
\end{minipage}
\end{table}

On CIFAR-10 (Table~\ref{tab:hd-images}a), CCVFM with $K{=}10000$
improves from $\mathrm{FID}_{50k}=9.66$ at {$11$ NFE} to $6.35$ at
{$51$ NFE}. ImageNet-32 (Table~\ref{tab:hd-images}b) gives the cleanest
HRF2 comparison (same U-Net, dataset, and protocol; only the source
changes): CCVFM wins by $38\%$ at {$11$ NFE} and $24\%$ at {$21$ NFE},
and ties within $3\%$ at {$51$ NFE}. On CelebA-HQ 256
(Table~\ref{tab:hd-images}c), the DC-AE f32c32 latent
($d{=}2048$) lets a $K{=}10000$ coreset capture face structure that
DiT-B ($131$M params) under-fits; scaling to DiT-L ($458$M) reaches
$\mathrm{FID}_{28k}{=}\mathbf{4.17}$ at {$51$ NFE}, competitive with
ADM and LDM-4 at $4$--$20\times$ fewer NFE without classifier-free
guidance. Sample grids are in Figure~\ref{fig:samples} and
Appendix~\S F. Ablations on $L$, $K$, the source/coupling choice, EMA, and the
matched CIFAR-10 $K{=}5000\to 10000$ coreset-size sweep are in
Appendix~\S D, together with Stage-I cost, rank, memorization,
KDE-source, compression, and manifold studies, and the Stage-III
training cost relative to HRF2.

\section{Related Work}
\label{sec:related}

\textbf{Generative modeling and few-step sampling.} GANs
\citep{goodfellow2014gan}, VAEs {\citep{kingma2013auto}},
and normalizing flows {\citep{rezende2015normalizing,dinh2017realnvp}}
established the modern toolkit, with tradeoffs in stability, likelihood
bounds, or architecture. Score-based diffusion models
\citep{ho2020ddpm,song2021score,karras2022edm} drive recent image
quality but often need long chains; EDM \citep{karras2022edm}
is the strongest CIFAR-10 efficiency benchmark. Consistency models
\citep{song2023consistency} and progressive distillation
\citep{salimans2022progressive} compress trained teachers, while
ReFlow \citep{liu2023rectified} and mini-batch OT-CFM
\citep{tong2024improving,pooladian2023multisample} straighten paths.
CCVFM is orthogonal: it reduces the inner correction transport before
any distillation, straightening, OT coupling, or solver choice.

\textbf{Flow matching and velocity distributions.} Conditional flow
matching \citep{lipman2023flow}, rectified flow
\citep{liu2023rectified}, and stochastic interpolants
\citep{albergo2023stochastic}
regress the conditional mean velocity, which is the squared-loss
optimum but collapses multimodal local velocity laws. The closest
prior work is HRF2 \citep{zhang2025hrf}, which models
$\pi(v\mid x_t,t)$ via a second-level velocity-space flow. 
We keep HRF2's identity and closed-form GMM corollary, but replace its
$\mathcal{N}(0,I)$ source with a coreset-GMM surrogate whose
conditional law is sampled in closed form through its mixture
representation.
Concurrently, \citet{guo2025velocity} proposed a variational low-dimensional latent
route to velocity multimodality.

\textbf{Coresets, smoothing, and diagnostics.} Stage~I builds on
entropic Optimal Transport (OT) \citep{cuturi2013sinkhorn} and weighted coreset / quantization theory \citep{claici2018wasserstein,graf2000foundations, yin2025wasserstein}.
Classical quantization provides the optimal $K$-atomic benchmark,
while the deployed Sinkhorn / GMM rate is used through the explicit
Stage-I compression assumption. 
In our diagnostic, we use improved precision / recall
\citep{sajjadi2018assessing,kynkaanniemi2019improved,naeem2020reliable},
and a $1$-NN KS/$W_1$ memorization test
{built on prior nearest-neighbour-based generative-model
diagnostics
\citep{balle2022reconstructing,meehan2020nonparametric,borji2018evaluation,carlini2023diffusion}
.}

\section{Conclusion and Limitations}
\label{sec:conclusion}

CCVFM replaces the velocity-space flow's isotropic noise source with
a closed-form coreset-GMM surrogate and trains the correction network
from that surrogate source. The surrogate gap equals
$W_2(\rho_1,\core)$ and vanishes with $K$ under the Stage-I
compression assumption, while the noise-source analogue has a
dimension-scale lower bound that is constant in $K$ and $n$
(Theorem~\ref{thm:reduction}). Under the direct surrogate-source
sampler, the source-dependent training excess is small whenever the
surrogate is close to $\pi$ in mean and covariance
(Theorem~\ref{thm:posterior-coupling}). 

Limitations include pessimistic high-$d$ constants, the assumed
compression rate for the deployed coreset/GMM, residual-field regularity
in the NFE scaling, EDM's remaining CIFAR-10 lead at
matched NFE, and single-run results on CIFAR-10, ImageNet-32, and
CelebA-HQ. Future work
should prove algorithmic rates for the entropic-Sinkhorn/GMM stage,
extend the analysis to other target-density classes and to $J>1$,
extend CCVFM to text-to-image generation via per-condition coresets, and
combine it with distillation/straightening.

\section*{Acknowledgements and Disclosures}

Xiao Wang acknowledges support from the National Science Foundation under Grant No.~SES-2316428 and from Anthropic through its AI for Science program. This work used DeltaAI at NCSA through allocation MTH260028 from the ACCESS program, which is supported by U.S. NSF grants \#2138259, \#2138286, \#2138307, \#2137603, and \#2138296. This research used the DeltaAI advanced computing and data resource, which is supported by the NSF (award OAC 2320345) and the State of Illinois. DeltaAI is a joint effort of UIUC and its NCSA.

{\small
\setlength{\bibsep}{0pt plus 0.2ex}
\bibliographystyle{plainnat}
\bibliography{references}
}


\appendix


\input{supp3}

\newpage
\input{checklist.tex}

\end{document}

%% file: supp3.tex

\section{Proofs}
\label{sec:proofs}

\subsection{Notation}

Throughout, $X_0\sim\rho_0=\mathcal{N}(0,I_d)$ and
$X_1\sim\rho_1$ are independent, $X_t=(1-t)X_0+tX_1$,
$V=X_1-X_0$. The true conditional velocity law is
$\pi(v\mid x,t)$; the surrogate is $\pisur(v\mid x,t)$ induced by
replacing $\rho_1$ with the coreset GMM $\core$. We identify a
density with its induced probability measure whenever no ambiguity is
possible; if the Stage-I surrogate is a finite coreset, the same
notation $\core$ denotes the corresponding probability measure. For a
surrogate conditional source $Q(\cdot\mid x,t)$, the \emph{conditional
transport cost} (Eq.~\eqref{eq:defD} of the main paper) is
\begin{equation}
\label{eq:defD-supp}
  D(Q) \;=\; \Big(\E_{X_0\sim\rho_0}\!\big[W_2^2\big(Q(\cdot\mid X_0,0),\,\pi(\cdot\mid X_0,0)\big)\big]\Big)^{1/2},
\end{equation}
defined in the second-moment form, consistent with the
$D(Q)$ used in Theorem~\ref{thm:reduction-supp}.

\begin{table}[H]
\centering
\small
\caption{Main notation.}
\label{tab:notation}
\begin{tabular}{@{}ll@{}}
\toprule
Symbol & Meaning \\
\midrule
$\rho_0=\mathcal N(0,I_d)$, $\rho_1$ & outer source and data law \\
$X_t=(1-t)X_0+tX_1$, $V=X_1-X_0$ & outer interpolation and velocity \\
$C=(X_t,t)$ & conditioning state of the inner flow \\
$\pi(\cdot\mid C)$, $\pisur(\cdot\mid C)$ & true and surrogate conditional velocity laws \\
$\mu_\pi,\Sigma_\pi$; $\mu_{\mathrm{sur}},\Sigma_{\mathrm{sur}}$ & conditional means / covariances of $\pi$, $\pisur$ \\
$\core=\sum_k w_k\mathcal N(\mu_k,\Sigma_k)$ & Stage-I coreset-GMM surrogate, $\Sigma_k=L_kL_k^\top+\sigma_k^2I_d$ \\
$K$, $r$ & number of components, covariance rank \\
$\gamma_b(C)$, $m_b(C)$, $\Lambda_b(t)$ & weights, means, covariances of $\pisur(\cdot\mid C)$ \\
$V_0$, $V_1$, $\tau$, $V_\tau$ & inner source, inner target, inner time, inner interpolant \\
$f_\theta(v,\tau,x_t,t)$ & Stage-III correction network \\
$D(Q)$ & conditional transport cost~\eqref{eq:defD-supp} \\
$J$, $L$ & outer steps, inner Euler steps ($\mathrm{NFE}=J(L+1)$) \\
\bottomrule
\end{tabular}
\end{table}

\subsection{Conditional velocity identity and GMM specialization}

We restate these for completeness; the proofs are standard (see
\cite{zhang2025hrf}).

\begin{proposition}[Conditional velocity identity, {\S~2 main paper, Eq.~(1)}]
\label{prop:cond-id}
For any $t\in(0,1)$ and $x$ with $\rho_t(x)>0$,
\[
  \pi(v\mid x,t) = \frac{\rho_0(x-tv)\,\rho_1(x+(1-t)v)}{\rho_t(x)}.
\]
At $t=0$, by independence of $X_0,X_1$,
$\pi(v\mid x_0,0)=\rho_1(x_0+v)$.
\end{proposition}

\begin{proof}
The joint density of $(X_t,V)$ at $(x,v)$ is
$\rho_0(x-tv)\rho_1(x+(1-t)v)$ by independence of the endpoints
combined with the endpoint identities $X_0=X_t-tV$,
$X_1=X_t+(1-t)V$. Normalize by $\rho_t(x)$. At $t=0$, $X_t=X_0$ so
$V=X_1-x_0$ with density $\rho_1(v+x_0)$.
\end{proof}

\begin{proposition}[GMM specialization, Prop.~\ref{prop:cond-id} with $\rho_1=\core$]
\label{prop:gmm-cond}
If $\core=\sum_b w_b\,\mathcal{N}(\mu_b,\Sigma_b)$, then the induced
$\pisur(v\mid x,t)=\sum_b \gamma_b(x,t)\,\mathcal{N}(v;m_b(x,t),\Lambda_b(t))$
is a $K$-component Gaussian mixture with closed-form parameters:
\begin{align*}
  A_b(t) &= t^2 I + (1-t)^2\Sigma_b^{-1},\quad \Lambda_b(t)=A_b(t)^{-1},\\
  b_b(x,t) &= tx - (1-t)\Sigma_b^{-1}(x-\mu_b),\quad m_b(x,t)=\Lambda_b(t)b_b(x,t),\\
  \gamma_b(x,t) &\propto w_b\,\eta_b(x,t),\quad
  \eta_b = |\Lambda_b|^{1/2}|\Sigma_b|^{-1/2}e^{\frac12 b_b^\top\Lambda_b b_b - \frac12 c_b(x)},\\
  c_b(x) &= x^\top x + (x-\mu_b)^\top\Sigma_b^{-1}(x-\mu_b).
\end{align*}
At $t=0$: $\pisur(v\mid x_0,0)=\core(x_0+v)=\sum_b w_b\mathcal{N}(v;\mu_b-x_0,\Sigma_b)$.
\end{proposition}

\begin{proof}
Completing the square in $v$ in $\mathcal{N}(x-tv;0,I)\mathcal{N}(x+(1-t)v;\mu_b,\Sigma_b)$
yields a Gaussian in $v$ with natural parameters $A_b$ and $b_b$,
hence mean $m_b=\Lambda_b b_b$ and covariance $\Lambda_b$. The
multiplicative constant for each component is collected into
$\eta_b$, and normalization produces $\gamma_b$. The $t=0$ case is
immediate from $\pi(v\mid x_0,0)=\rho_1(x_0+v)$ in
Prop.~\ref{prop:cond-id} with $\rho_1\mapsto\core$.
\end{proof}

\subsection{Theorem~\ref{thm:reduction} of main paper (Transport-task reduction)}

\begin{theorem}
\label{thm:reduction-supp}
With $D(Q)$ as in~\eqref{eq:defD-supp}, define
$Q_{\mathrm{CCVFM}}(\cdot\mid x_0,0)=(S_{x_0})_\#\core$ with
$S_{x_0}(y)=y-x_0$ (equivalently, when $\core$ has a density,
$Q_{\mathrm{CCVFM}}(v\mid x_0,0)=\core(x_0+v)$), and
$Q_{\mathrm{HRF2}}(v\mid x_0,0)=\mathcal{N}(v;0,I_d)$. Assume
$\rho_0=\mathcal{N}(0,I_d)$ and $X_0\perp X_1$ with
$X_1\sim\rho_1\in\mathcal P_2(\R^d)$.
Then
\begin{enumerate}[label=\emph{(\roman*)},leftmargin=*]
  \item $D(Q_{\mathrm{CCVFM}})=W_2(\rho_1,\core)$;
  \item $D(Q_{\mathrm{HRF2}})\ge \sqrt{d}\,\big|\sqrt{\sigma_1^2+1}-1\big|$, where
        $\sigma_1^2:=\E\|X_1\|^2/d$.
\end{enumerate}
\end{theorem}

\begin{proof}
\textbf{(i)} By Prop.~\ref{prop:cond-id}, at $t=0$ the true
conditional velocity law is
$\pi(\cdot\mid x_0,0)=\Law(X_1-x_0)=(S_{x_0})_\#\rho_1$. The CCVFM
source is $(S_{x_0})_\#\core$. For any coupling $(Y_1,Y_2)$ of
$\rho_1$ and $\core$, the shifted pair $(Y_1-x_0,Y_2-x_0)$ couples
the two velocity laws, and conversely any velocity-space coupling can
be shifted by $x_0$ to couple $\rho_1$ and $\core$. The quadratic cost
is unchanged under this common shift, so for every fixed $x_0$,
\begin{equation}
\label{eq:ccvfm-cond}
  W_2\!\big(Q_{\mathrm{CCVFM}}(\cdot\mid x_0,0),\,\pi(\cdot\mid x_0,0)\big)
  \;=\; W_2(\core,\rho_1),
\end{equation}
which is constant in $x_0$. Hence
$\E_{X_0} W_2^2 = W_2(\core,\rho_1)^2$, and
$D(Q_{\mathrm{CCVFM}}) = W_2(\core,\rho_1)$.

\textbf{(ii)} We proceed in four labeled steps.

\emph{Step 1 ($W_2^2$ moment lower bound).}
For any probability measures $\mu,\nu$ on $\R^d$ with finite
second moments,
\begin{equation}
\label{eq:w2-moment}
  W_2^2(\mu,\nu) \;\ge\; \big(\sqrt{m_2(\mu)}-\sqrt{m_2(\nu)}\big)^2,
\end{equation}
where $m_2(\mu):=\int\|x\|^2\,\mathrm{d}\mu(x)$. This is immediate
from any coupling $(U,V)$: by the Minkowski inequality,
$(\E\|U-V\|^2)^{1/2}\ge|(\E\|U\|^2)^{1/2}-(\E\|V\|^2)^{1/2}|$.
Taking the infimum over couplings yields~\eqref{eq:w2-moment}
(see, e.g., \cite{villani2009optimal}).

\emph{Step 2 (Apply~\eqref{eq:w2-moment} at fixed $x_0$).}
Fix $x_0\in\R^d$. We have $m_2(\mathcal{N}(0,I_d))=d$. For
$\mu=\Law(X_1-x_0)$ with $X_1\sim\rho_1$,
\[
  m_2(\mu) \;=\; \E_{X_1}\|X_1-x_0\|^2 \;=:\; S(x_0)^2,
\]
a deterministic quantity depending only on $x_0$. By
\eqref{eq:w2-moment},
\begin{equation}
\label{eq:pointwise-square}
  W_2^2\!\big(\mathcal{N}(0,I_d),\,\Law(X_1-x_0)\big)
  \;\ge\; \big(\sqrt{d}-S(x_0)\big)^2.
\end{equation}

\emph{Step 3 (Take expectation over $X_0\sim\rho_0$).}
Integrating~\eqref{eq:pointwise-square} over $X_0$ and expanding
the square,
\begin{equation}
\label{eq:expanded}
  \E_{X_0}\!\big[(\sqrt{d}-S(X_0))^2\big]
  \;=\; d \;-\; 2\sqrt{d}\,\E_{X_0}[S(X_0)] \;+\; \E_{X_0}[S(X_0)^2].
\end{equation}
We compute each moment of $S$:
\begin{align}
  \E_{X_0}[S(X_0)^2]
  &= \E_{X_0}\E_{X_1}\|X_1-X_0\|^2
   = \E\|X_1\|^2 \;-\; 2\,\E\langle X_1,X_0\rangle \;+\; \E\|X_0\|^2 \notag\\
  &= d\sigma_1^2 \;-\; 2\,\langle\mu_{\rho_1},\E X_0\rangle \;+\; d
   \;=\; d\sigma_1^2 + d \;=\; d(\sigma_1^2+1),
\label{eq:ES2}
\end{align}
using $X_0\perp X_1$, $\mu_{\rho_0}=\E X_0=0$, and
$\E\|X_0\|^2=d$ (standard Gaussian). By Jensen's inequality with
concave $u\mapsto\sqrt{u}$,
\begin{equation}
\label{eq:jensen-ES}
  \E_{X_0}[S(X_0)] \;=\; \E_{X_0}\!\sqrt{S(X_0)^2}
  \;\le\; \sqrt{\E_{X_0}[S(X_0)^2]}
  \;=\; \sqrt{d(\sigma_1^2+1)}.
\end{equation}
The coefficient $-2\sqrt{d}$ in front of $\E S$ in~\eqref{eq:expanded}
is \emph{negative}, so substituting the \emph{upper bound}
\eqref{eq:jensen-ES} produces a valid \emph{lower bound} on the
right-hand side:
\begin{align}
  \E_{X_0}[(\sqrt{d}-S(X_0))^2]
  &\ge\; d \;-\; 2\sqrt{d}\cdot\sqrt{d(\sigma_1^2+1)} \;+\; d(\sigma_1^2+1) \notag\\
  &=\; d\,\big[1 \;-\; 2\sqrt{\sigma_1^2+1} \;+\; (\sigma_1^2+1)\big] \notag\\
  &=\; d\,\big(\sqrt{\sigma_1^2+1}-1\big)^2,
\label{eq:square-complete}
\end{align}
where the last line uses the algebraic identity
$u^2 - 2u + 1 = (u-1)^2$ with $u=\sqrt{\sigma_1^2+1}$.

\emph{Step 4 (Collect).} Combining~\eqref{eq:pointwise-square}
and~\eqref{eq:square-complete} (and noting that both inequalities
preserve their direction after integration),
\begin{equation}
\label{eq:final-sq}
  D(Q_{\mathrm{HRF2}})^2
  \;=\; \E_{X_0}\!\big[W_2^2\big(\mathcal{N}(0,I_d),\,\Law(X_1-X_0)\big)\big]
  \;\ge\; d\,\big(\sqrt{\sigma_1^2+1}-1\big)^2.
\end{equation}
Taking the non-negative square root and using $|a|=\sqrt{a^2}$ gives
\begin{equation}
\label{eq:final-primary}
  D(Q_{\mathrm{HRF2}}) \;\ge\; \sqrt{d}\,\big|\sqrt{\sigma_1^2+1}-1\big|,
\end{equation}
which is the statement of (ii).
\end{proof}

\begin{remark}[Why the bound is nonzero and $K$-independent]
Since $\sigma_1^2\ge 0$, $\sqrt{\sigma_1^2+1}\ge 1$, with equality
iff $\sigma_1=0$ (i.e.\ $\rho_1=\delta_0$). For every nondegenerate
target the bound is strictly positive and, crucially, depends only
on the second moment of $\rho_1$, not on the coreset size $K$ or
sample size $n$. Typical image targets have $\sigma_1$ bounded and
$d$ large, so the bound scales as $\Omega(\sqrt d\,\sigma_1^2)$
for small $\sigma_1$ and as $\Omega(\sqrt d\,\sigma_1)$ for large
$\sigma_1$.
\end{remark}

\begin{lemma}[Existence of a $K$-atom quantization benchmark]
\label{lem:k-quantization-supp}
Assume $\rho_1$ is supported on $\Omega=[0,1]^d$. Let
\[
  \mathcal G_K(\Omega):=
  \left\{\sum_{j=1}^K w_j\delta_{y_j}:\; w_j\ge0,\;\sum_{j=1}^K w_j=1,
  \; y_j\in\Omega\right\}.
\]
Then there exists $C_d<\infty$, depending only on $d$, such that
\[
  \inf_{\nu\in\mathcal G_K(\Omega)} W_2(\rho_1,\nu)\le C_dK^{-1/d}.
\]
Consequently, if the Stage-I coreset $\widetilde{\rho}_{1,K}$ is a global minimizer
of the $K$-atom Wasserstein quantization problem over
$\mathcal G_K(\Omega)$, then
\[
  W_2(\rho_1,\widetilde{\rho}_{1,K})=O(K^{-1/d}).
\]
\end{lemma}

\begin{proof}
Let $m=\lfloor K^{1/d}\rfloor$. For $K$ large enough,
$m\ge c_dK^{1/d}$ for a constant $c_d>0$ and $m^d\le K$. Partition
$[0,1]^d$ into $m^d$ axis-aligned cubes $\{A_\ell\}$ of side length
$1/m$. For each cell with $\rho_1(A_\ell)>0$, choose
$a_\ell\in A_\ell$ and set
\[
  \nu_m:=\sum_\ell \rho_1(A_\ell)\delta_{a_\ell}.
\]
This measure has at most $m^d\le K$ atoms. Couple $X\sim\rho_1$ to
$a_\ell$ whenever $X\in A_\ell$. Since every cube has diameter at most
$\sqrt d/m$,
\[
  W_2^2(\rho_1,\nu_m)
  \le \sum_\ell\int_{A_\ell}\|x-a_\ell\|^2\,\dd\rho_1(x)
  \le \frac{d}{m^2}.
\]
Thus $W_2(\rho_1,\nu_m)\le \sqrt d/m\le C_dK^{-1/d}$. Enlarging
$C_d$ covers the finitely many small values of $K$. Taking the
infimum over $\mathcal G_K(\Omega)$ proves the lemma.
\end{proof}

\paragraph{Consequence for CCVFM.}
If $\rho_1$ is supported on $[0,1]^d$ and the Stage-I coreset
$\widetilde{\rho}_{1,K}$ is a global $K$-atom Wasserstein quantizer, then
\[
  D(Q_{\mathrm{CCVFM}})=W_2(\rho_1,\widetilde{\rho}_{1,K})=O(K^{-1/d})
\]
by Theorem~\ref{thm:reduction-supp}(i) and
Lemma~\ref{lem:k-quantization-supp}. This consequence applies to
exact global Wasserstein quantization. It does not assert that the
finite-iteration entropic-Sinkhorn/GMM routine used in Stage~I
necessarily returns such a global quantizer. For the deployed routine,
the corresponding approximation property is imposed as
Assumption~\ref{ass:stageI-compression} of the main paper.

\begin{remark}[How to read the coreset assumption]
The preceding consequence uses only the $W_2$ approximation property
$W_2(\rho_1,\widetilde{\rho}_{1,K})=O(K^{-1/d})$. If Stage~I is not solved as an
exact global quantization problem, it is enough to assume or prove
that the returned coreset satisfies the same rate, possibly with
additional optimization, entropic-regularization, and GMM-lift slack.
This paper treats that deployed-rate statement as the Stage-I
compression assumption rather than as a consequence of the algorithm.
\end{remark}

\subsection{Corollary~\ref{cor:nfe-gap} of main paper (NFE sufficient scaling)}

\begin{lemma}[Euler error under space-time regularity]
\label{lem:euler-regularity-supp}
Consider the non-autonomous ODE
\[
  \dot V_\tau=f_\theta(V_\tau,\tau),\qquad \tau\in[0,1].
\]
Assume $f_\theta$ is $L_v$-Lipschitz in $v$ and $L_\tau$-Lipschitz in
$\tau$, and assume along the exact trajectory
\[
  \sup_{\tau\in[0,1]}\|f_\theta(V_\tau,\tau)\|_{L^2}\le M.
\]
Then the explicit Euler scheme with $n$ uniform steps satisfies
\begin{equation}
\label{eq:euler-standard}
  \|\mathrm{Euler}(n)-\mathrm{Exact}\|_{L^2}
  \le
  \frac{e^{L_v}-1}{2L_v}\,\frac{L_vM+L_\tau}{n},
\end{equation}
where the prefactor is interpreted as $1/2$ when $L_v=0$. In
particular, if $L_v$ is bounded, then
\[
  \|\mathrm{Euler}(n)-\mathrm{Exact}\|_{L^2}
  =O\!\left(\frac{L_vM+L_\tau}{n}\right).
\]
\end{lemma}

\begin{proof}
Let $h=1/n$ and $\tau_k=kh$. Denote by $V_{\tau_k}$ the exact
solution and by $\widetilde V_k$ the Euler iterate. The one-step
local truncation error is
\[
  \delta_k
  :=V_{\tau_{k+1}}-V_{\tau_k}-h f_\theta(V_{\tau_k},\tau_k)
  =\int_{\tau_k}^{\tau_{k+1}}
  \big(f_\theta(V_t,t)-f_\theta(V_{\tau_k},\tau_k)\big)\,\dd t.
\]
Using the Lipschitz assumptions and the bound by $M$,
\[
  \|\delta_k\|_{L^2}
  \le \int_{\tau_k}^{\tau_{k+1}}(L_vM+L_\tau)(t-\tau_k)\,\dd t
  =\frac{1}{2}(L_vM+L_\tau)h^2.
\]
For the global error $e_k:=\|\widetilde V_k-V_{\tau_k}\|_{L^2}$,
\[
  e_{k+1}\le (1+hL_v)e_k+\frac{1}{2}(L_vM+L_\tau)h^2.
\]
Iterating this recursion with $e_0=0$ gives
\[
  e_n\le \frac{(1+hL_v)^n-1}{hL_v}\cdot
  \frac{1}{2}(L_vM+L_\tau)h^2
  \le \frac{e^{L_v}-1}{2L_v}\frac{L_vM+L_\tau}{n}.
\]
The case $L_v=0$ follows by taking the limit.
\end{proof}

\begin{corollary}[NFE sufficient scaling]
\label{cor:nfe-supp}
Let $D:=D(Q)$ be the conditional transport cost in
\eqref{eq:defD-supp}. Assume the correction ODE satisfies the
regularity conditions of Lemma~\ref{lem:euler-regularity-supp}.
Assume further the residual straight-line scale and flattening
conditions
\begin{equation}
\label{eq:flattening-assumptions-supp}
  M=O(D),\qquad L_v=O(D),\qquad L_\tau=O(D^2),
\end{equation}
with constants uniform over the family of transport tasks under
consideration. Then, for any $\varepsilon_{\mathrm{disc}}>0$, a
sufficient number of Euler function evaluations to make the $L^2$
discretization error at most $\varepsilon_{\mathrm{disc}}$ is
\begin{equation}
\label{eq:nfe-D2-supp}
  \mathrm{NFE}(\varepsilon_{\mathrm{disc}})=O\!\left(1+\frac{D^2}{\varepsilon_{\mathrm{disc}}}\right).
\end{equation}
Without the flattening conditions~\eqref{eq:flattening-assumptions-supp},
the generic sufficient bound is
\begin{equation}
\label{eq:nfe-generic}
  \mathrm{NFE}(\varepsilon_{\mathrm{disc}})
  =O\!\left(1+\frac{L_vD+L_\tau}{\varepsilon_{\mathrm{disc}}}\right),
\end{equation}
assuming only $M=O(D)$ and bounded $L_v$.

For CCVFM, let
$\Delta_K:=D(Q_{\mathrm{CCVFM}})=W_2(\rho_1,\core)$. Then
\[
  \mathrm{NFE}_{\mathrm{CCVFM}}(\varepsilon_{\mathrm{disc}})
  =O\!\left(1+\frac{\Delta_K^2}{\varepsilon_{\mathrm{disc}}}\right)
\]
under~\eqref{eq:flattening-assumptions-supp}. If the Stage-I
compression assumption gives $\Delta_K=O(K^{-1/d})$, then
\[
  \mathrm{NFE}_{\mathrm{CCVFM}}(\varepsilon_{\mathrm{disc}})
  =O\!\left(1+\frac{K^{-2/d}}{\varepsilon_{\mathrm{disc}}}\right).
\]
For HRF2, if $d$ and $\rho_1$ are fixed and the corresponding
regularity constants are $O(1)$, then a sufficient scaling is
\[
  \mathrm{NFE}_{\mathrm{HRF2}}(\varepsilon_{\mathrm{disc}})
  =O\!\left(1+\frac{1}{\varepsilon_{\mathrm{disc}}}\right).
\]
\end{corollary}

\begin{proof}
Lemma~\ref{lem:euler-regularity-supp} gives Euler error
$O((L_vM+L_\tau)/n)$ when $L_v$ is bounded. Under
\eqref{eq:flattening-assumptions-supp}, $L_vM+L_\tau=O(D^2)$, so
$n\ge C D^2/\varepsilon_{\mathrm{disc}}$ is sufficient; at least one
function evaluation gives~\eqref{eq:nfe-D2-supp}. If only $M=O(D)$ is
used, the same Euler bound gives the generic scaling~\eqref{eq:nfe-generic}.

For CCVFM, $D(Q_{\mathrm{CCVFM}})=\Delta_K$ by
Theorem~\ref{thm:reduction-supp}(i). This gives the displayed
$\Delta_K^2/\varepsilon_{\mathrm{disc}}$ scaling. If the Stage-I
compression assumption gives $\Delta_K=O(K^{-1/d})$, then the
$K^{-2/d}/\varepsilon_{\mathrm{disc}}$ scaling follows. For HRF2, the
law $Q_{\mathrm{HRF2}}$ is independent of $K$; for fixed $d$ and a
fixed nondegenerate target, Theorem~\ref{thm:reduction-supp}(ii) gives
a positive $K$-independent lower bound. A matching $K$-independent
upper bound follows by the independent coupling: for
$Z\sim\mathcal N(0,I_d)$ independent of $(X_0,X_1)$,
\[
  D(Q_{\mathrm{HRF2}})^2
  \le \E\|Z-X_1+X_0\|^2
  =d(\sigma_1^2+2).
\]
Thus $D(Q_{\mathrm{HRF2}})=\Theta_K(1)$, and the sufficient NFE upper
bound is $O(1+1/\varepsilon_{\mathrm{disc}})$ under $O(1)$ regularity constants.
\end{proof}

\begin{remark}[Interpretation of the flattening assumption]
\label{rem:Lf-D}
A small conditional transport cost $D(Q)$ controls the typical
displacement magnitude, but it does not by itself determine the
$v$-Lipschitz constant of the learned velocity field. The scaling
$L_v=O(D)$ should therefore be read as an additional regularity or
flattening assumption on the residual correction field. A sufficient
regime is that, on a fixed-scale velocity domain,
\[
  f_{\theta,Q}(v,\tau)=D(Q)\,\widetilde f_{\theta,Q}(v,\tau),
  \qquad \operatorname{Lip}_v(\widetilde f_{\theta,Q})=O(1).
\]
For the non-autonomous ODE one also needs the time variation to be
controlled, as in $L_\tau=O(D^2)$, or an equivalent local
truncation-error assumption. Under these additional conditions the
Euler upper bound scales as $O(D(Q)^2/n)$. Without them, the proof
only yields the generic sufficient bound in~\eqref{eq:nfe-generic},
with the regularity constants treated as problem-dependent quantities.

The three conditions in~\eqref{eq:flattening-assumptions-supp} are
assumptions on the learned residual field, not consequences of a small
$D$ alone. A useful reference is the boundary case $D=0$: the source and
target conditional laws then coincide, the identity correction
$V_\tau\equiv V_0$ is valid, and for this reference solution
$f_\theta\equiv0$, so $M=L_v=L_\tau=0$. The conditions specify how the
learned field is assumed to approach the identity correction as $D$
decreases. (i) $M=O(D)$ is the amplitude condition: since
$V_1-V_0=\int_0^1 f_\theta(V_\tau,\tau)\,\dd\tau$, $M$ bounds the $L^2$
amplitude of the correction speed, uniformly in $\tau$.
(ii) $L_v=O(D)$ is the spatial flattening condition: it excludes a
correction field that is small but varies sharply with $v$, and gives a
spatial contribution $L_vM=O(D^2)$. (iii) $L_\tau=O(D^2)$ is the temporal
regularity condition, placing the explicit time variation at the same
order. Then $L_vM+L_\tau=O(D^2)$, and the Euler bound converts $D$ into
the $D^2$-NFE scaling. Corollary~\ref{cor:nfe-supp} is therefore a
conditional sufficient result; without these conditions the general
bound~\eqref{eq:euler-standard} still holds, and retaining only $M=O(D)$
and bounded $L_v$ gives the weaker scaling~\eqref{eq:nfe-generic}.
\end{remark}

{\subsection{Theorem~\ref{thm:posterior-coupling} of main paper
(Direct surrogate-source second moment vs.\ independent Gaussian source)}
\label{sec:a7}

This subsection formalizes the training-signal scale for the direct
surrogate-source sampler defined in Stage~III. The relevant facts are:
(i) the sampler has the exact source marginal $\pisur(\cdot\mid C)$,
and (ii) the conditional second moment of the regression target admits
a closed-form identity in terms of the conditional means and covariances
of $\pi(\cdot\mid C)$ and $\pisur(\cdot\mid C)$.

\begin{remark}[Second moment versus centered variance]
\label{rem:second-moment-vs-variance}
The statements below control
$\E[\|V_1-V_0\|^2\mid C]$, not the centered variance
\[
  \E\bigl[
  \|V_1-V_0-\E[V_1-V_0\mid C]\|^2
  \,\big|\, C
  \bigr].
\]
This is intentional: the squared norm of the training target
$V_1-V_0$ is the scale appearing directly in the flow-matching
regression loss. Since centered variance is always bounded above by
the second moment, the non-centered second moment is the cleaner
quantity for training-signal size.
\end{remark}

\begin{proposition}[Direct surrogate sampling: correct marginal and conditional second moment]
\label{prop:coupling-supp}
Fix $C=(X_t,t)$. Let $V_1\sim\pi(\cdot\mid C)$ have finite second
moment. Let the surrogate mixture be
\[
  \pisur(v\mid C)
  =
  \sum_{b=1}^K
  \gamma_b(C)\,
  \mathcal N\bigl(v;m_b(C),\Lambda_b(t)\bigr).
\]
Draw
\[
  B\mid C\sim
  \mathrm{Cat}\bigl(\gamma_1(C),\ldots,\gamma_K(C)\bigr),
  \qquad
  V_0\mid (B=b,C)
  \sim
  \mathcal N\bigl(m_b(C),\Lambda_b(t)\bigr),
\]
and assume that this draw is conditionally independent of $V_1$ given
$C$. Define
\[
  \mu_\pi(C):=\E[V_1\mid C],
  \qquad
  \Sigma_\pi(C):=\operatorname{Cov}(V_1\mid C),
\]
and
\[
  \mu_{\mathrm{sur}}(C)
  :=
  \sum_{b=1}^K\gamma_b(C)m_b(C),
\]
\[
  \Sigma_{\mathrm{sur}}(C)
  :=
  \sum_{b=1}^K
  \gamma_b(C)
  \Bigl[
    \Lambda_b(t)
    +
    \bigl(m_b(C)-\mu_{\mathrm{sur}}(C)\bigr)
    \bigl(m_b(C)-\mu_{\mathrm{sur}}(C)\bigr)^\top
  \Bigr].
\]
Then
\[
  V_0\mid C\sim \pisur(\cdot\mid C),
\]
and
\begin{equation}
\label{eq:direct-surrogate-second-moment}
  \E\bigl[\|V_1-V_0\|^2\,\big|\,C\bigr]
  =
  \operatorname{tr}\Sigma_\pi(C)
  +
  \operatorname{tr}\Sigma_{\mathrm{sur}}(C)
  +
  \bigl\|\mu_\pi(C)-\mu_{\mathrm{sur}}(C)\bigr\|^2 .
\end{equation}
Moreover, if
\[
  T_{\max}(C)
  :=
  \max_{1\le b\le K}\operatorname{tr}\Lambda_b(t),
  \qquad
  D_\mu(C)
  :=
  \max_{1\le b,b'\le K}
  \|m_b(C)-m_{b'}(C)\|,
\]
then
\begin{equation}
\label{eq:direct-surrogate-upper-bound}
  \operatorname{tr}\Sigma_{\mathrm{sur}}(C)
  \le
  T_{\max}(C)+\frac12 D_\mu^2(C),
\end{equation}
and therefore
\begin{equation}
\label{eq:direct-surrogate-second-moment-bound}
  \E\bigl[\|V_1-V_0\|^2\,\big|\,C\bigr]
  \le
  \operatorname{tr}\Sigma_\pi(C)
  +
  \bigl\|\mu_\pi(C)-\mu_{\mathrm{sur}}(C)\bigr\|^2
  +
  T_{\max}(C)
  +
  \frac12D_\mu^2(C).
\end{equation}
\end{proposition}

\begin{proof}
Throughout the proof, all expectations are conditional on the fixed
value of $C$. Write
\[
  \gamma_b=\gamma_b(C),\qquad
  m_b=m_b(C),\qquad
  \Lambda_b=\Lambda_b(t).
\]

First, the marginal law of $V_0$ is obtained by marginalizing over the
component label:
\[
  p(V_0=v\mid C)
  =
  \sum_{b=1}^K
  \Pr(B=b\mid C)\,
  p(V_0=v\mid B=b,C).
\]
Substituting the sampling rule gives
\[
  p(V_0=v\mid C)
  =
  \sum_{b=1}^K
  \gamma_b\,
  \mathcal N(v;m_b,\Lambda_b)
  =
  \pisur(v\mid C).
\]
Hence $V_0\mid C\sim\pisur(\cdot\mid C)$.

Next, the conditional mean of $V_0$ is
\[
  \E[V_0\mid C]
  =
  \sum_{b=1}^K \gamma_b m_b
  =
  \mu_{\mathrm{sur}}(C).
\]
By the law of total covariance,
\[
  \operatorname{Cov}(V_0\mid C)
  =
  \E\bigl[\operatorname{Cov}(V_0\mid B,C)\mid C\bigr]
  +
  \operatorname{Cov}\bigl(\E[V_0\mid B,C]\mid C\bigr).
\]
Since
\[
  \operatorname{Cov}(V_0\mid B=b,C)=\Lambda_b,
  \qquad
  \E[V_0\mid B=b,C]=m_b,
\]
we obtain
\[
  \operatorname{Cov}(V_0\mid C)
  =
  \sum_{b=1}^K \gamma_b\Lambda_b
  +
  \sum_{b=1}^K
  \gamma_b
  (m_b-\mu_{\mathrm{sur}}(C))
  (m_b-\mu_{\mathrm{sur}}(C))^\top
  =
  \Sigma_{\mathrm{sur}}(C).
\]

Because $V_0$ and $V_1$ are conditionally independent given $C$, expanding
the square yields
\[
  \E[\|V_1-V_0\|^2\mid C]
  =
  \E[\|V_1\|^2\mid C]
  +
  \E[\|V_0\|^2\mid C]
  -
  2\,\E[V_1\mid C]^\top \E[V_0\mid C].
\]
Using
\[
  \E[\|V_1\|^2\mid C]
  =
  \operatorname{tr}\Sigma_\pi(C)+\|\mu_\pi(C)\|^2
\]
and
\[
  \E[\|V_0\|^2\mid C]
  =
  \operatorname{tr}\Sigma_{\mathrm{sur}}(C)
  +
  \|\mu_{\mathrm{sur}}(C)\|^2,
\]
we get
\[
  \E[\|V_1-V_0\|^2\mid C]
  =
  \operatorname{tr}\Sigma_\pi(C)
  +
  \operatorname{tr}\Sigma_{\mathrm{sur}}(C)
  +
  \|\mu_\pi(C)-\mu_{\mathrm{sur}}(C)\|^2,
\]
which proves~\eqref{eq:direct-surrogate-second-moment}.

It remains to prove the upper bound on
$\operatorname{tr}\Sigma_{\mathrm{sur}}(C)$. Taking traces in the
definition of $\Sigma_{\mathrm{sur}}(C)$ gives
\[
  \operatorname{tr}\Sigma_{\mathrm{sur}}(C)
  =
  \sum_{b=1}^K \gamma_b\operatorname{tr}\Lambda_b
  +
  \sum_{b=1}^K
  \gamma_b
  \|m_b-\mu_{\mathrm{sur}}(C)\|^2.
\]
The first term is bounded by $T_{\max}(C)$. For the second term, use
the pairwise identity
\[
  \sum_{b=1}^K
  \gamma_b
  \|m_b-\mu_{\mathrm{sur}}(C)\|^2
  =
  \frac12
  \sum_{b=1}^K
  \sum_{b'=1}^K
  \gamma_b\gamma_{b'}
  \|m_b-m_{b'}\|^2 .
\]
Since every pairwise distance is bounded by $D_\mu(C)$, this implies
\[
  \sum_{b=1}^K
  \gamma_b
  \|m_b-\mu_{\mathrm{sur}}(C)\|^2
  \le
  \frac12D_\mu^2(C).
\]
Therefore
\[
  \operatorname{tr}\Sigma_{\mathrm{sur}}(C)
  \le
  T_{\max}(C)+\frac12D_\mu^2(C),
\]
which proves~\eqref{eq:direct-surrogate-upper-bound}. Combining this
with~\eqref{eq:direct-surrogate-second-moment} gives
\eqref{eq:direct-surrogate-second-moment-bound}.
\end{proof}

\begin{proposition}[Independent coupling: exact second moment and modal lower bound]
\label{prop:indep-lower-supp}
Fix $C=(X_t,t)$. Let $V_1\sim\pi(\cdot\mid C)$ have finite second
moment, and let $V_0\sim\mathcal{N}(0,I_d)$ be independent of $V_1$
conditional on $C$. Define
\[
  \mu_\pi(C):=\E[V_1\mid C],
  \qquad
  \Sigma_\pi(C):=\operatorname{Cov}(V_1\mid C).
\]
Then
\begin{equation}
\label{eq:indep-lower}
  \E\big[\|V_1-V_0\|^2\,\big|\,C\big]
  \;=\; \operatorname{tr}\Sigma_\pi(C) + \|\mu_\pi(C)\|^2 + d.
\end{equation}

Moreover, suppose that under the same fixed condition $C$, the
conditional law of $V_1$ admits a latent mixture representation with
component label $Z\in\{1,\ldots,K\}$, weights
$\gamma^\pi_b(C):=\Pr(Z=b\mid C)$, and conditional component means
$m_b^\pi(C):=\E[V_1\mid Z=b,C]$. Let
\[
  D_\pi(C):=\max_{b,b'}\|m_b^\pi(C)-m_{b'}^\pi(C)\|.
\]
Choose a pair $(b_-,b_+)$ attaining this maximum and define
$c_\gamma(C):=\gamma^\pi_{b_-}(C)\gamma^\pi_{b_+}(C)$. Then
\begin{equation}
\label{eq:modal-var-lower}
  \operatorname{tr}\Sigma_\pi(C)\ge c_\gamma(C)D_\pi^2(C).
\end{equation}
Consequently,
\begin{equation}
\label{eq:indep-lower-cons}
  \E\big[\|V_1-V_0\|^2\,\big|\,C\big]
  \;\ge\; c_\gamma(C)D_\pi^2(C)+d.
\end{equation}
If all components have weights at least $\gamma^\pi_{\min}(C)>0$, then
$c_\gamma(C)\ge(\gamma^\pi_{\min}(C))^2$. In the two-component case,
$c_\gamma(C)=\gamma^\pi_1(C)\gamma^\pi_2(C)$.
\end{proposition}

\begin{proof}
The exact identity follows by expanding the square conditional on $C$:
\[
  \E\bigl[\|V_1-V_0\|^2\mid C\bigr]
  =\E\bigl[\|V_1\|^2\mid C\bigr]
   -2\E[V_1\mid C]^\top\E[V_0\mid C]
   +\E\bigl[\|V_0\|^2\mid C\bigr].
\]
Since $V_0\sim\mathcal N(0,I_d)$ and is independent of $V_1$
conditional on $C$, $\E[V_0\mid C]=0$ and
$\E[\|V_0\|^2\mid C]=d$. Also
$\E[\|V_1\|^2\mid C]=\operatorname{tr}\Sigma_\pi(C)+\|\mu_\pi(C)\|^2$,
which proves~\eqref{eq:indep-lower}.

For the modal lower bound, apply the law of total variance conditional
on $C$:
\[
  \operatorname{tr}\Sigma_\pi(C)
  =\sum_{b=1}^K \gamma^\pi_b(C)
      \operatorname{tr}\operatorname{Cov}(V_1\mid Z=b,C)
    +\sum_{b=1}^K \gamma^\pi_b(C)
      \|m_b^\pi(C)-\bar m_\pi(C)\|^2,
\]
where $\bar m_\pi(C):=\sum_b\gamma^\pi_b(C)m_b^\pi(C)$. Dropping the
nonnegative within-component term gives
\[
  \operatorname{tr}\Sigma_\pi(C)
  \ge
  \sum_{b=1}^K \gamma^\pi_b(C)
      \|m_b^\pi(C)-\bar m_\pi(C)\|^2.
\]
The between-component term has the pairwise identity
\[
  \sum_{b=1}^K \gamma^\pi_b(C)
      \|m_b^\pi(C)-\bar m_\pi(C)\|^2
  =
  \sum_{1\le i<j\le K}
      \gamma^\pi_i(C)\gamma^\pi_j(C)
      \|m_i^\pi(C)-m_j^\pi(C)\|^2.
\]
Keeping only the pair $(b_-,b_+)$ that realizes $D_\pi(C)$ proves
\eqref{eq:modal-var-lower}. Combining this with
\eqref{eq:indep-lower} and dropping the nonnegative
$\|\mu_\pi(C)\|^2$ proves~\eqref{eq:indep-lower-cons}.
\end{proof}

\begin{remark}[Relation to target-space modes]
The modal lower bound in Proposition~\ref{prop:indep-lower-supp} is a
statement about the conditional velocity law $\pi(\cdot\mid C)$, not
merely about the global target distribution $\rho_1$. At $t=0$, if
$V_1=Y-X_0$ with $Y\sim\rho_1$, target-space mode centers are shifted
by the same vector $-X_0$, so their pairwise distances are preserved.
For general $t$, however, the mixture weights and component means
must be checked for the particular conditional model.
\end{remark}

{\begin{remark}[Comparison with the independent Gaussian source]
The direct surrogate sampler and the independent Gaussian source have
different second-moment identities. For the direct surrogate sampler,
Proposition~\ref{prop:coupling-supp} gives
\[
  \E[\|V_1-V_0\|^2\mid C]_{\mathrm{sur}}
  =
  \operatorname{tr}\Sigma_\pi(C)
  +
  \operatorname{tr}\Sigma_{\mathrm{sur}}(C)
  +
  \|\mu_\pi(C)-\mu_{\mathrm{sur}}(C)\|^2 .
\]
For the independent Gaussian source,
Proposition~\ref{prop:indep-lower-supp} gives
\[
  \E[\|V_1-V_0\|^2\mid C]_{\mathrm{Gauss}}
  =
  \operatorname{tr}\Sigma_\pi(C)
  +
  \|\mu_\pi(C)\|^2
  +
  d .
\]
The common term $\operatorname{tr}\Sigma_\pi(C)$ is unavoidable because
it is the intrinsic conditional spread of the target velocity law.
Thus, at the level of this training second moment, the surrogate source
is preferable whenever
\[
  \operatorname{tr}\Sigma_{\mathrm{sur}}(C)
  +
  \|\mu_\pi(C)-\mu_{\mathrm{sur}}(C)\|^2
  \ll
  \|\mu_\pi(C)\|^2+d .
\]
This condition says that the surrogate conditional velocity law is
closer to the target conditional velocity law in mean and covariance
than the isotropic Gaussian source is. This comparison is separate from
the boundary transport-task reduction in Theorem~\ref{thm:reduction},
which depends only on the inference source marginal at $t=0$.
\end{remark}


{\begin{remark}[When does the direct surrogate source reduce the training target?]
\label{rem:direct-surrogate-ratio}
Let
\[
  S_\pi(C):=\operatorname{tr}\Sigma_\pi(C),
  \qquad
  M_{\mathrm{sur}}(C):=
  \|\mu_\pi(C)-\mu_{\mathrm{sur}}(C)\|^2 .
\]
Also decompose the surrogate conditional covariance as
\[
  \operatorname{tr}\Sigma_{\mathrm{sur}}(C)
  =
  A_{\mathrm{sur}}(C)+B_{\mathrm{sur}}(C),
\]
where
\[
  A_{\mathrm{sur}}(C)
  :=
  \sum_{b=1}^K
  \gamma_b(C)\operatorname{tr}\Lambda_b(t)
\]
is the average within-component covariance scale, and
\[
  B_{\mathrm{sur}}(C)
  :=
  \sum_{b=1}^K
  \gamma_b(C)
  \|m_b(C)-\mu_{\mathrm{sur}}(C)\|^2
  =
  \frac12
  \sum_{b,b'=1}^K
  \gamma_b(C)\gamma_{b'}(C)
  \|m_b(C)-m_{b'}(C)\|^2
\]
is the weighted between-component spread of the surrogate source.

For the direct surrogate sampler,
Proposition~\ref{prop:coupling-supp} gives
\[
  \E[\|V_1-V_0\|^2\mid C]_{\mathrm{sur}}
  =
  S_\pi(C)
  +
  A_{\mathrm{sur}}(C)
  +
  B_{\mathrm{sur}}(C)
  +
  M_{\mathrm{sur}}(C).
\]
For the independent Gaussian source,
Proposition~\ref{prop:indep-lower-supp} gives
\[
  \E[\|V_1-V_0\|^2\mid C]_{\mathrm{Gauss}}
  =
  S_\pi(C)
  +
  \|\mu_\pi(C)\|^2+d .
\]
Therefore the total second-moment ratio is
\[
  \mathcal R_{\mathrm{tot}}(C)
  :=
  \frac{
  \E[\|V_1-V_0\|^2\mid C]_{\mathrm{sur}}
  }{
  \E[\|V_1-V_0\|^2\mid C]_{\mathrm{Gauss}}
  }
  =
  \frac{
  S_\pi(C)
  +
  A_{\mathrm{sur}}(C)
  +
  B_{\mathrm{sur}}(C)
  +
  M_{\mathrm{sur}}(C)
  }{
  S_\pi(C)+\|\mu_\pi(C)\|^2+d
  } .
\]
Thus the direct surrogate source has a smaller total training target
than the independent Gaussian source whenever
\[
  A_{\mathrm{sur}}(C)
  +
  B_{\mathrm{sur}}(C)
  +
  M_{\mathrm{sur}}(C)
  <
  \|\mu_\pi(C)\|^2+d .
\]
Equivalently, the surrogate source is beneficial when its
within-component covariance, weighted between-component spread, and
conditional mean mismatch are jointly smaller than the source penalty
incurred by the isotropic Gaussian source.

It is also useful to isolate the reducible, source-dependent part of
the training target. Since the intrinsic conditional spread
$S_\pi(C)$ appears in both couplings, define the excess ratio
\[
  \mathcal R_{\mathrm{ex}}(C)
  :=
  \frac{
  \E[\|V_1-V_0\|^2\mid C]_{\mathrm{sur}}-S_\pi(C)
  }{
  \E[\|V_1-V_0\|^2\mid C]_{\mathrm{Gauss}}-S_\pi(C)
  }
  =
  \frac{
  A_{\mathrm{sur}}(C)
  +
  B_{\mathrm{sur}}(C)
  +
  M_{\mathrm{sur}}(C)
  }{
  \|\mu_\pi(C)\|^2+d
  } .
\]
This excess ratio becomes small when the surrogate conditional source
is concentrated around the true conditional velocity law:
\[
  A_{\mathrm{sur}}(C)\ll d+\|\mu_\pi(C)\|^2,\qquad
  B_{\mathrm{sur}}(C)\ll d+\|\mu_\pi(C)\|^2,\qquad
  M_{\mathrm{sur}}(C)\ll d+\|\mu_\pi(C)\|^2 .
\]
In words, CCVFM reduces the source-induced training second moment when the
GMM surrogate has small component covariances, when its conditional
mixture weights are sufficiently localized so that the weighted
between-component spread is small, and when its conditional mean
matches the target conditional mean. 

\end{remark}

\begin{remark}[Exact surrogate and data preprocessing]
\label{rem:thm3-preprocessing}
If $\core=\rho_1$, then $\pisur=\pi$, hence
$A_{\mathrm{sur}}+B_{\mathrm{sur}}=\operatorname{tr}\Sigma_{\mathrm{sur}}=S_\pi$
and $M_{\mathrm{sur}}=0$, and the surrogate excess equals $S_\pi(C)$. At
$t=0$, $\pi(\cdot\mid x_0,0)=\rho_1(x_0+\cdot)$, so
$S_\pi=\operatorname{tr}\operatorname{Cov}(X_1)$ and
$\|\mu_\pi\|^2=\|\E X_1-x_0\|^2$. If every coordinate of $X_1$ is
standardized to unit variance, $\operatorname{tr}\operatorname{Cov}(X_1)=d$;
if every coordinate takes values in $[-1,1]$ (the usual image
scaling), Popoviciu's inequality gives $\operatorname{Var}(X_{1,j})\le1$
and $\operatorname{tr}\operatorname{Cov}(X_1)\le d$. In both cases
$S_\pi\le d\le\|\mu_\pi\|^2+d$, so in this extreme case the surrogate
excess never exceeds the Gaussian excess. The statement concerns the
direct sampler of part~(i) of Theorem~\ref{thm:posterior-coupling} only.
\end{remark}

\begin{remark}[Other label-draw designs]
\label{rem:general-label}
Part~(i) uses a label $B\sim\mathrm{Cat}(\gamma(C))$ drawn independently of
$V_1$. The same conditioning argument applies to any label draw
$B\mid(C,V_1)$ followed by the Gaussian step
$V_0\mid(B,C,V_1)\sim\mathcal N(m_B(C),\Lambda_B(t))$: conditioning on
$(B,C,V_1)$ gives
$\E[\|V_1-V_0\|^2\mid B,C,V_1]=\|V_1-m_B(C)\|^2+\operatorname{tr}\Lambda_B(t)$,
hence
\[
  \E\bigl[\|V_1-V_0\|^2\mid C\bigr]
  =\E\bigl[\|V_1-m_B(C)\|^2+\operatorname{tr}\Lambda_B(t)\,\big|\,C\bigr].
\]
When $B$ is independent of $V_1$ with law $\gamma(C)$, the right-hand side
equals $\sum_b\gamma_b(C)\bigl(S_\pi(C)+\|\mu_\pi(C)-m_b(C)\|^2+\operatorname{tr}\Lambda_b(t)\bigr)
=S_\pi+A_{\mathrm{sur}}+B_{\mathrm{sur}}+M_{\mathrm{sur}}$, which is
part~(i).
\end{remark}}


\section{Extended Algorithm Pseudocode}

\subsection{Stage I: Sinkhorn coreset + low-rank covariance fit}

\begin{algorithm}[H]
\caption{Stage I: Sinkhorn coreset $+$ low-rank covariance fit.}
\label{alg:stage1}
\begin{algorithmic}[1]
\Require data $\{x_i\}_{i=1}^n\in\R^d$, coreset size $K$, Sinkhorn regularizer $\lambda$, rank $r$, Sinkhorn iters $T_{\mathrm{sk}}$.
\Ensure GMM $\core=\sum_k w_k\,\mathcal{N}(\mu_k,L_kL_k^\top+\sigma_k^2 I_d)$.
\State Initialize atoms $\{\mu_k\}$ by uniform sampling from $\{x_i\}$; weights $w_k\gets 1/K$.
\For{$t=1,\ldots,T_{\mathrm{sk}}$}
  \State \textbf{(E-step / Coupling):} $T_{ij}\propto w_j\exp(-\|x_i-\mu_j\|^2/\lambda)$; normalize rows so $\sum_j T_{ij}=1/n$ \Comment{{closed-form $T$-update from Eq.~(2) of the main paper}}
  \State \textbf{(M-step / Anchor):} $w_j\gets \sum_i T_{ij}$, \; $\mu_j\gets (1/w_j)\sum_i T_{ij}\,x_i$
\EndFor
\For{$k=1,\ldots,K$}
  \State Form weighted residuals $r_i\gets \sqrt{T_{ik}/w_k}\,(x_i-\mu_k)$
  \State Truncated SVD $R\approx U_k\,S_k\,V_k^\top$ of rank $r$; set $L_k\gets V_k\,S_k\in\R^{d\times r}$
  \State $\sigma_k^2\gets (\|R\|_F^2-\sum_{j=1}^r s_{kj}^2)/(d-r)$
\EndFor
\State \Return $\{(w_k,\mu_k,L_k,\sigma_k)\}_{k=1}^K$.
\end{algorithmic}
\end{algorithm}

\subsection{Stage III training}

\begin{algorithm}[H]
\caption{Stage III training, shown at $t=0$. For general $t$, the
Gaussian step uses the component law $\mathcal N(m_b(C),\Lambda_b(t))$ of
Eq.~\eqref{eq:pisur} of the main paper.}
\label{alg:stage3}
\begin{algorithmic}[1]
\Require data $\{x_i\}$, GMM state $\core$, correction net $f_\theta$, iters $T$, batch size $B$.
\Ensure trained parameters $\theta$.
\For{$\mathrm{step}=1,\ldots,T$}
  \State Sample batch $x_1\in\R^{B\times d}$ from data; $x_0\sim\mathcal{N}(0,I_d)^B$
  \State Compute $v_{\mathrm{true}}\gets x_1-x_0$
  \State Sample mixture label $b_i^\star\sim\mathrm{Cat}(q^i_1,\ldots,q^i_K)$
       \Comment{component weights; $q^i=w$ gives $v_0^i\mid x_0^i\sim\pisur(\cdot\mid x_0^i,0)$; our experiments anchor the label to $x_1^i$}
\State Sample $v_0^i\sim\mathcal{N}(\mu_{b_i^\star}-x_0^i,\,\Sigma_{b_i^\star})$
  \State Sample $\tau_i\sim\mathrm{Unif}[0,1]$; form $v_\tau^i\gets (1-\tau_i)v_0^i+\tau_i v_{\mathrm{true}}^i$
  \State Optimizer step on loss $\frac{1}{B}\sum_i\|f_\theta(v_\tau^i,\tau_i,x_0^i,0)-(v_{\mathrm{true}}^i-v_0^i)\|^2$
\EndFor
\State \Return $\theta$.
\end{algorithmic}
\end{algorithm}

\subsection{Inference with $J,L$}

\begin{algorithm}[H]
\caption{CCVFM nested sampler ($J$ outer steps $\times\,L$ inner correction steps).}
\label{alg:inference}
\begin{algorithmic}[1]
\Require source $\rho_0=\mathcal{N}(0,I_d)$, GMM $\core$, trained net $f_\theta$, grid $0=t_0<\cdots<t_J=1$, inner steps $L$.
\Ensure generated sample $\hat X_1$.
\State Draw $X_0\sim\mathcal{N}(0,I_d)$; \; set $Z\gets X_0$
\For{$j=0,\ldots,J-1$}
  \State Draw $V\sim\pisur(\cdot\mid Z,t_j)$ \Comment{closed-form categorical $+$ Gaussian}
  \For{$i=0,\ldots,L-1$}
    \State $V\gets V+(1/L)\,f_\theta(V,\,i/L,\,Z,\,t_j)$ \Comment{Euler step of inner correction flow}
  \EndFor
  \State $Z\gets Z+(t_{j+1}-t_j)\,V$
\EndFor
\State \Return $\hat X_1\gets Z$.
\end{algorithmic}
\end{algorithm}

For the $J=1$ generator (\S\ref{sec:generation} of the main paper), the outer loop runs once at $t_0=0$ and emits $\hat X_1=X_0+\hat V$ directly.
At inference the label is drawn from the surrogate weights
$\gamma(Z,t_j)$ (at $t=0$, from $w$), so the source law of the generator
is exactly $\pisur(\cdot\mid Z,t_j)$.

\subsection{General $J>1$ sampler}
The nested sampler of Algorithm~\ref{alg:inference} with $J>1$ uses
the general-$t$ law $\pisur(\cdot\mid Z,t_j)$ of Eq.~\eqref{eq:pisur} of the main
paper, which is closed-form at every outer state, together with the same
correction network evaluated at $(Z,t_j)$; its cost is $\mathrm{NFE}=J(L+1)$, of which $JL$ are learned-network evaluations.
This variant is not covered by the theory, and on the networks reported
here $J=2$ matched $J=1$ or was worse (Table~\ref{tab:nfe-sweep}).
Those networks are trained over $t\in[0,1]$. Under a protocol matched to
the outer grid, the network for $J=2$ is trained only on the slices it
visits, $t\in\{0,\tfrac12\}$ (source draws from
$\pisur(\cdot\mid X_t,t)$ at both slices), and is compared with a $J=1$
network trained at $t=0$ at equal learned-network evaluations $JL$. In
this setting $J=2$ improves FID on MNIST at both values of $K$ and both
budgets, and in both training seeds (Table~\ref{tab:j2-matched}). These
runs use the direct sampler in a separate lightweight configuration, so absolute
values are not comparable with the other MNIST results in this paper.

\begin{table}[H]
\centering
\small
\caption{$J=2$ vs.\ $J=1$ under the matched-slice protocol on MNIST
(FID$_{50k}$, mean of two training seeds). Columns give $(J,L)$; pairs
with equal $JL$ are adjacent.}
\label{tab:j2-matched}
\begin{tabular}{lcccc}
\toprule
 & $(1,10)$ & $(2,5)$ & $(1,20)$ & $(2,10)$ \\
\midrule
$K=1000$ & $14.75$ & $\mathbf{5.63}$ & $17.47$ & $\mathbf{6.35}$ \\
$K=100$  & $24.76$ & $\mathbf{10.30}$ & $21.72$ & $\mathbf{9.49}$ \\
\bottomrule
\end{tabular}
\end{table}

\subsection{The $J=1$ reduction and interior times}
\label{sec:j1-reduction-supp}
At $t=0$ let $Y_0=X_0+V_0$ and $Y_1=X_1$. Then
$X_0+V_\tau=(1-\tau)Y_0+\tau Y_1$ and $V_1-V_0=Y_1-Y_0$, so the $t=0$
part of the loss reads
$\E\|f_\theta(Y_\tau-X_0,\tau,X_0,0)-(Y_1-Y_0)\|^2$ with
$Y_\tau=(1-\tau)Y_0+\tau Y_1$. Given the label, $Y_0\mid b^\star\sim
\mathcal N(\mu_{b^\star},\Sigma_{b^\star})$ does not depend on $X_0$, and
the label is drawn independently of $X_0$; since $X_1$ is independent of
$X_0$, the endpoint pair $(Y_0,Y_1)$ is independent of $X_0$. The population minimizer is therefore
$f^\star(v,\tau,x_0,0)=\E[Y_1-Y_0\mid Y_\tau=x_0+v]$, which is exactly the
data-space flow-matching field from $Y_0$ to $Y_1$ under the same endpoint
coupling, evaluated at $x_0+v$. Integrating this field from
$V_0\sim\pisur(\cdot\mid X_0,0)$ and returning $X_0+V$ is therefore
data-space flow matching from $\core$ to $\rho_1$; the two views differ
only at the estimator level (our network is shared across $t\in[0,1]$
and receives $x_0$ as input). Theorem~\ref{thm:reduction} depends only
on the source marginal $\pisur(\cdot\mid x_0,0)=\core(x_0+\cdot)$ and
applies in either view.

The reduction is specific to $t=0$.

\begin{proposition}[No translation structure at interior times]
\label{prop:interior-nontranslation}
Let $\core=\sum_{b=1}^K w_b\,\mathcal N(\mu_b,\Sigma_b)$ be a minimal
mixture representation ($K\ge2$, all $w_b>0$, pairwise distinct
$(\mu_b,\Sigma_b)$), and let $t\in(0,1)$. Then there is no fixed
probability law $q_0$ and map $T:\R^d\to\R^d$ such that
$\pisur(v\mid x,t)=q_0\big(v+T(x)\big)$ for all $x$. At $t=0$, by
contrast, $\pisur(v\mid x,0)=\core(x+v)$ is such a translation family.
\end{proposition}

\begin{proof}
By Eq.~\eqref{eq:pisur} of the main paper,
$\pisur(\cdot\mid x,t)=\sum_b\gamma_b(x,t)\,\mathcal N(m_b(x,t),\Lambda_b(t))$
with state-dependent mixture weights
$\gamma_b(x,t)\propto w_b\,\mathcal N\big(x;t\mu_b,(1-t)^2I+t^2\Sigma_b\big)$.
For every $x$ the $K$ components are distinct: if $\Sigma_b\neq\Sigma_c$,
then $\Lambda_b(t)\neq\Lambda_c(t)$ because
$\Sigma\mapsto(t^2I+(1-t)^2\Sigma^{-1})^{-1}$ is injective; if
$\Sigma_b=\Sigma_c$ and $\mu_b\neq\mu_c$, then
$m_b-m_c=(1-t)\Lambda_b(t)\Sigma_b^{-1}(\mu_b-\mu_c)\neq0$.
A translation does not change mixture weights, so if
$\pisur(\cdot\mid x,t)=q_0(\cdot+T(x))$ for all $x$, identifiability of
finite Gaussian mixtures implies that the multiset
$\{\gamma_b(x,t)\}_{b=1}^K$ equals the weight multiset of $q_0$ and is
therefore the same for all $x$. Each $\gamma_b(\cdot,t)$ is continuous and
takes values in this fixed finite set, so it is constant on $\R^d$. But
$\log(\gamma_b/\gamma_c)$ is the log-ratio of $w_b/w_c$ times two Gaussian
densities in $x$ with different parameters (since $t>0$ and
$(\mu_b,\Sigma_b)\neq(\mu_c,\Sigma_c)$), i.e.\ a non-constant polynomial
of degree at most two in $x$, a contradiction.
\end{proof}

Hence the change of variables behind the $J=1$ reduction has no analogue
at the second outer step of a $J=2$ sampler, which re-anchors on the
partially transported state through $\pisur(\cdot\mid X_{1/2},\tfrac12)$.
The proposition does not exclude every augmented data-space
representation; it shows that the exact reduction is a property of the
$t=0$ boundary.

\section{Training hyperparameters}

\begin{table}[H]
\centering
\small
\caption{Configuration used for each headline result. All batch sizes
are per-GPU on 1$\times$ NVIDIA GH200 120 GB. CIFAR-10 Stage~III
ran $400$k steps followed by a continuation to $720$k; the CelebA-HQ run
used peak learning rate $2\times10^{-4}$ and $1\times10^{-4}$ in a resumed
segment.}
\label{tab:hparams}
\setlength{\tabcolsep}{3pt}
\begin{tabular}{lcccc}
\toprule
 & MNIST & CIFAR-10 & ImageNet-32 & {CelebA-HQ 256} \\
\midrule
$d$ & $784$ & $3072$ & $3072$ & {$2048$ (DC-AE f32c32 latent)} \\
$n$ (train) & $60{,}000$ & $50{,}000$ & $1{,}281{,}167$ & {$\sim 28{,}000$} \\
$K$ (coreset) & $2000$ & $10000$ & $5000$ & {$10000$} \\
$r$ (rank) & $50$ & $80$ & $80$ & {$128$} \\
Sinkhorn $\lambda$ & $1.5$ & $0.5$ & $0.5$ & {$0.5$} \\
Sinkhorn iters & $100$ & $100$ & $100$ & {$100$} \\
{Correction net} & U-Net & dual-branch U-Net & dual-branch U-Net & {DiT-L} \\
U-Net arch & 128$\to$256$\to$512 & 128$\to$256$\to$256$\to$256 & idem & {n/a (DiT-L)} \\
\# params & $10.75$M & $44.76$M & $46.16$M & {$458.24$M} \\
Stage III iters & $200{,}000$ & $720{,}000$ & $400{,}000$ & {$400{,}000$} \\
Batch size & $64$ & $128$ & $256$ & {$128$} \\
Optimizer & Adam & Adam & Adam & {AdamW} \\
Learning rate & $2\times 10^{-4}$ & $2\times 10^{-4}$ & $1.4\times 10^{-4}$ & {$2\times 10^{-4}$} \\
EMA decay & $0.9999$ & $0.9999$ & $0.9999$ & {$0.9999$} \\
\bottomrule
\end{tabular}
\end{table}

\section{Additional Ablations}
\label{sec:ablations-main-moved}

\subsection{Ablations summary}
\label{sec:ablations-summary}

On MNIST the $L$-sweep (Table~1 of the main paper) improves
monotonically up to $L{=}50$; on CIFAR-10 and ImageNet-32 the sweet
spot is $L{=}50$ with $J{=}1$ (full sweeps in
\S\ref{sec:ablation-cifar-img32}). 
The CCVFM-vs-HRF2 gap reflects replacing the isotropic Gaussian source
$\mathcal{N}(0,I)$ by the data-informed surrogate source
$\pisur(\cdot\mid C)$, together with differences in $K$ and budget.
EMA contributes $<0.2$ FID on MNIST, so the improvement as $K$ scales
is driven by the coreset budget and training iterations rather than by
weight averaging.

\paragraph{CIFAR-10 coreset-size ablation.}
Doubling $K$ from $5000$ to $10000$ improves FID$_{50k}$ at $51$ NFE
from $7.78$ to $6.35$ with larger gains at low NFE, consistent with
the qualitative prediction of the Stage-I compression assumption when
interpreted with an empirical effective dimension $d_{\mathrm{eff}}$.
This effective-dimensional scaling is an empirical interpretation,
not a theorem proved for the deployed Sinkhorn/GMM routine.

\subsection{MNIST: coreset-size and training-budget scaling}

\begin{table}[H]
\centering
\small
\setlength{\tabcolsep}{5pt}
\caption{MNIST scaling of FID$_{50k}$ as the coreset budget $K$, the
correction-training iterations, and the U-Net base channel width
grow simultaneously. ``U-Net width'' is the number of channels at the
first (widest) resolution block; channels at deeper levels are
multiples $1,2,4$ of this base. All rows use the same three-stage CCVFM pipeline with
EMA decay $0.9999$.
}

\label{tab:mnist-scaling}
\begin{tabular}{cccccc}
\toprule
$K$ & iters & U-Net width & $L$ & NFE & FID$_{50k}\downarrow$ \\
\midrule
$500$  & 40k  & 64  & 10 & 11 & 3.79 \\
$1000$ & 80k  & 64  & 10 & 11 & 1.93 \\
$2000$ & 200k & 128 & 10 & 11 & 2.88 \\
$2000$ & 200k & 128 & 20 & 21 & \textbf{1.09} \\
$2000$ & 200k & 128 & 50 & 51 & \textbf{0.75} \\
\bottomrule
\end{tabular}
\end{table}

The improvement from $3.79\to 1.93\to 1.09$ at $\le 21$ NFE as $K$ and
iteration budget scale up is qualitatively consistent with the Stage-I
compression assumption. At $K{=}2000$, FID improves monotonically up to $L{=}50$.

\subsection{CIFAR-10 and ImageNet-32 correction-step sweep}
\label{sec:ablation-cifar-img32}

\begin{table}[H]
\centering
\small
\caption{FID$_{50k}$ across correction-integration configurations
$(J,L)$. The CIFAR-10 block reports the headline configuration
$K{=}10000$ and the smaller $K{=}5000$ configuration used for the
$K$-ablation discussion in the main paper.}
\label{tab:nfe-sweep}
\begin{tabular}{lcccc}
\toprule
Dataset & $(J,L)$ & FID$_{5k}$ & FID$_{10k}$ & FID$_{50k}$ \\
\midrule
CIFAR-10 ($K{=}10000$) & $(1,10)$   & 17.66 & 13.31 & 9.66 \\
CIFAR-10 ($K{=}10000$) & $(1,20)$   & 15.34 & 10.94 & 7.24 \\
CIFAR-10 ($K{=}10000$) & $(1,50)$   & 14.40 & 10.06 & \textbf{6.35} \\
CIFAR-10 ($K{=}10000$) & $(2,50)$   & 18.09 & 14.11 & 10.40 \\
\midrule
CIFAR-10 ($K{=}5000$) & $(1,10)$   & 21.27 & 16.71 & 13.18 \\
CIFAR-10 ($K{=}5000$) & $(1,20)$   & 17.53 & 12.93 & 9.28 \\
CIFAR-10 ($K{=}5000$) & $(1,50)$   & 16.04 & 11.48 & 7.78 \\
CIFAR-10 ($K{=}5000$) & $(2,50)$   & 21.76 & 18.10 & 14.84 \\
\midrule
ImageNet-32 & $(1,10)$  & 22.21 & 16.84 & 12.55 \\
ImageNet-32 & $(1,20)$  & 19.15 & 13.79 & 9.51 \\
ImageNet-32 & $(1,50)$  & 18.39 & 13.04 & \textbf{8.76} \\
ImageNet-32 & $(2,50)$  & 18.44 & 13.22 & 8.75 \\
ImageNet-32 & $(2,250)$ & 18.80 & 13.79 & 9.39 \\
\bottomrule
\end{tabular}
\end{table}

\paragraph{ImageNet-32: choice of $K$ and memorization check.}
For ImageNet-32 we use $K{=}5000$, i.e.\ $0.4\%$ of the
$n\approx1.28$M training images, so each component summarizes on
average about $256$ images. Table~\ref{tab:in32-mem} reports the
Inception-feature $1$-NN distance from $10$k generated samples to
size-matched training subsets (mean $\pm$ std over subsets) and to the
held-out validation set. All ratios are within $0.3\%$ of one, for the
closed-form Stage-II source alone and for the full pipeline, so neither
the surrogate nor the trained generator is preferentially close to
training images. In addition, a $K{=}1000$ run ($0.08\%$ of $n$) gave a
Stage-II FID$_{10k}$ of $261$ and full-pipeline FID$_{10k}$ of $26.9$
and $18.4$ at $11$ and $51$ NFE, versus $16.8$ and $13.0$ for
$K{=}5000$; reducing $K$ thus degrades sample quality and does not
change the memorization picture (train/validation ratio $1.002$).

\begin{table}[H]
\centering
\small
\caption{ImageNet-32 memorization check: $1$-NN Inception distance to
size-matched training subsets vs.\ the validation set.}
\label{tab:in32-mem}
\begin{tabular}{lccc}
\toprule
Samples & to train & to validation & train/val \\
\midrule
Stage II only, $K{=}5000$ & $11.549\pm0.021$ & $11.519$ & $1.003$ \\
Stage II only, $K{=}1000$ & $11.564\pm0.024$ & $11.532$ & $1.003$ \\
Full CCVFM, $K{=}5000$, $L{=}10$ & $8.764\pm0.014$ & $8.756$ & $1.001$ \\
Full CCVFM, $K{=}5000$, $L{=}20$ & $9.515\pm0.014$ & $9.511$ & $1.000$ \\
\bottomrule
\end{tabular}
\end{table}

\paragraph{Stage-III cost relative to HRF2.}
Stage~III adds the closed-form surrogate draw to each training step.
On MNIST, a matched $40$k-step benchmark took $8$\,min\,$20$\,s with the
Gaussian source and $8$\,min\,$48$\,s with the surrogate source
($+5.6\%$); the MNIST headline model trains in $44$\,min. On
ImageNet-32, a step takes $262.9$\,ms versus $221.7$\,ms for the
Gaussian source ($+18.6\%$), with peak memory $48.5$ vs.\ $32.4$\,GiB.
The ImageNet-32 headline run took approximately $29$\,h on one GH200,
including a restart from the $300$k checkpoint; it uses the same U-Net,
the same FID protocol, and the same NFE as the HRF2 comparison.

\subsection{MNIST: effect of EMA weights}

On MNIST with $K=2000$ and $L{=}10$, raw weights vs.\ EMA at decay
$0.9999$ yield FID$_{50k}=3.03$ vs.\ $2.88$ (a $0.15$ FID delta).
Weight averaging therefore contributes a small but measurable
improvement; the dominant gains at this scale come from the coreset
budget $K$ and the correction-training iteration count.

\subsection{Source-distribution ablation}
\label{sec:coupling-ablation}


The choice of training source distribution is the central difference
between CCVFM and HRF2: CCVFM trains the correction network with
a source built from $\pisur(\cdot\mid C)$, whereas HRF2 uses the isotropic Gaussian
$\mathcal{N}(0,I_d)$. An isolated, matched-recipe ablation that
replaces $\pisur$ by $\mathcal{N}(0,I_d)$ while keeping $K$, the
architecture, and the training schedule identical is left to future
work.

In lieu of the matched retrain,
Proposition~\ref{prop:coupling-supp} gives an exact identity for the
conditional second moment of the training residual under the direct
surrogate sampler:
$\E[\|V_1-V_0\|^2\mid C]=\operatorname{tr}\Sigma_\pi(C)+
\operatorname{tr}\Sigma_{\mathrm{sur}}(C)+\|\mu_\pi(C)-\mu_{\mathrm{sur}}(C)\|^2$.
The corresponding identity for the independent Gaussian source is
$\operatorname{tr}\Sigma_\pi(C)+\|\mu_\pi(C)\|^2+d$
(Proposition~\ref{prop:indep-lower-supp}). The surrogate source has a
strictly smaller training-target scale whenever
$\operatorname{tr}\Sigma_{\mathrm{sur}}(C)+\|\mu_\pi(C)-\mu_{\mathrm{sur}}(C)\|^2
<\|\mu_\pi(C)\|^2+d$, i.e. when the surrogate conditional law is closer
to the true conditional velocity law in mean and covariance than the
isotropic Gaussian source.

\subsection{Monte Carlo validation of Theorem~\ref{thm:posterior-coupling}}
\label{sec:thm3-mc}
At $t=0$ we estimate $\E\|V_1-V_0\|^2$ by Monte Carlo under (i) the
direct surrogate-source sampler, $V_0\sim\pisur(\cdot\mid X_0,0)$ drawn
independently of $V_1$, and (ii) the independent Gaussian source, and
compare with the closed forms
$\operatorname{tr}\operatorname{Cov}(\rho_1)+\operatorname{tr}\operatorname{Cov}(\core)+\|\E X_1-\E\tilde X_1\|^2$
and $\E\|V_1\|^2+d$ (Table~\ref{tab:thm3-mc}). Measured and predicted
values agree on all seven datasets. The Gaussian-to-surrogate ratio is
about one on the simplest 2D targets and grows to $1.8$ on the 3D helix,
$15.7$ on MNIST, and $4.5$ on CIFAR-10. The per-dimension flow-matching
loss at initialization is approximately $\E\|V_1-V_0\|^2/d$
(Figure~\ref{fig:thm3-loss}). On MNIST, at matched architecture and
budget, the surrogate-source run averages $0.144$ per dimension over the
first $100$ steps, close to its Theorem-3 scale $105.5/784\approx0.135$;
the Gaussian-source run, with initialization scale
$1655.6/784\approx2.11$, averages $0.543$ over the same steps and stays
above the surrogate-source run throughout training ($0.087$ vs.\
$0.018$ at the final $80$k-step checkpoint).

\begin{figure}[H]
\centering
\includegraphics[width=0.62\linewidth]{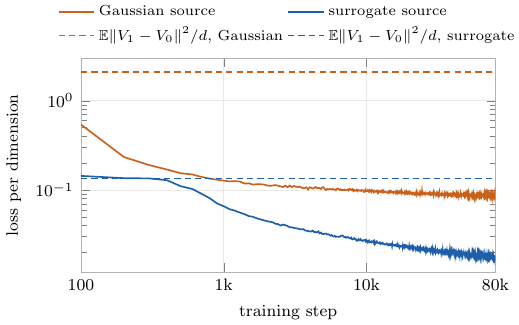}
\caption{Stage-III training loss on MNIST ($K=1000$, $r=30$, $t=0$,
same network and optimizer), surrogate source (direct sampler) vs.\
Gaussian source. Each point is the mean over $100$ steps. Dashed lines:
the initialization scale $\E\|V_1-V_0\|^2/d$ predicted by
Theorem~\ref{thm:posterior-coupling} (Table~\ref{tab:thm3-mc}).}
\label{fig:thm3-loss}
\end{figure}

\begin{table}[H]
\centering
\small
\caption{Monte Carlo check of both identities of
Theorem~\ref{thm:posterior-coupling} at $t=0$ (direct sampler).}
\label{tab:thm3-mc}
\begin{tabular}{lrrrrrr}
\toprule
 & & \multicolumn{2}{c}{surrogate, (i)} & \multicolumn{2}{c}{Gaussian, (ii)} & \\
\cmidrule(lr){3-4}\cmidrule(lr){5-6}
Dataset & $d$ & measured & theory & measured & theory & ratio \\
\midrule
ring-6       & 2   & 8.171   & 8.158   & 8.088   & 8.074   & 0.99 \\
moons        & 2   & 9.400   & 9.426   & 9.618   & 9.618   & 1.02 \\
pinwheel     & 2   & 4.993   & 4.988   & 6.508   & 6.486   & 1.30 \\
checkerboard & 2   & 5.307   & 5.318   & 6.683   & 6.672   & 1.26 \\
helix        & 3   & 4.668   & 4.677   & 8.344   & 8.344   & 1.79 \\
MNIST        & 784 & 105.48  & 105.45  & 1655.18 & 1655.61 & 15.69 \\
CIFAR-10     & 3072& 1528.36 & 1523.21 & 6929.16 & 6929.04 & 4.53 \\
\bottomrule
\end{tabular}
\end{table}

\subsection{Stage-I cost}
\label{sec:stage1-cost}
Table~\ref{tab:stage1-cost} reports the wall-clock of the Stage-I fit
(coreset iterations plus covariance fits, excluding data loading) on one
GH200. Over a $16\times$ range of $K$, the cost grows near-linearly:
$2\to20$\,s on MNIST ($K=250\to4000$) and $7\to95$\,s on CIFAR-10
($K=625\to10000$); the full sweep is in Table~\ref{tab:stage1-ksweep}
and Figure~\ref{fig:stage1-scaling}.
Stage~I is below $1\%$ of the Stage-III training time in all cases.

\begin{figure}[H]
\centering
\includegraphics[width=0.85\linewidth]{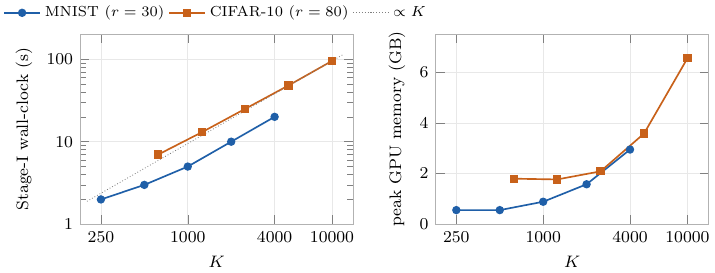}
\caption{Stage-I wall-clock (left, log--log; dotted: slope one) and peak
GPU memory (right) against $K$ on one GH200; values as in
Table~\ref{tab:stage1-ksweep}.}
\label{fig:stage1-scaling}
\end{figure}

\begin{table}[H]
\centering
\small
\caption{Stage-I cost across $K$ (MNIST $r=30$, CIFAR-10 $r=80$).}
\label{tab:stage1-ksweep}
\begin{tabular}{rrr@{\hspace{2em}}rrr}
\toprule
\multicolumn{3}{c}{MNIST} & \multicolumn{3}{c}{CIFAR-10} \\
\cmidrule(lr){1-3}\cmidrule(lr){4-6}
$K$ & time & memory & $K$ & time & memory \\
\midrule
$250$  & $2$\,s  & $0.56$\,GB & $625$   & $7$\,s  & $1.80$\,GB \\
$500$  & $3$\,s  & $0.56$\,GB & $1250$  & $13$\,s & $1.77$\,GB \\
$1000$ & $5$\,s  & $0.89$\,GB & $2500$  & $25$\,s & $2.09$\,GB \\
$2000$ & $10$\,s & $1.58$\,GB & $5000$  & $48$\,s & $3.58$\,GB \\
$4000$ & $20$\,s & $2.95$\,GB & $10000$ & $95$\,s & $6.55$\,GB \\
\bottomrule
\end{tabular}
\end{table}

\begin{table}[H]
\centering
\small
\caption{Stage-I wall-clock and peak GPU memory.}
\label{tab:stage1-cost}
\begin{tabular}{lrrrrr}
\toprule
Dataset & $n$ & $K$ & $r$ & wall-clock & peak memory \\
\midrule
MNIST       & $60$k    & $2000$  & $30$ & $10$\,s  & $1.6$\,GB \\
CIFAR-10    & $50$k    & $10000$ & $80$ & $95$\,s  & $6.6$\,GB \\
ImageNet-32 & $1.28$M  & $5000$  & $60$ & $350$\,s & $22.9$\,GB \\
\bottomrule
\end{tabular}
\end{table}

\subsection{GMM rank ablation}
\label{sec:rank-ablation}
On a lightweight MNIST configuration ($K=1000$, $80$k training steps),
Table~\ref{tab:rank} varies the covariance rank $r$
(mean $\pm$ std over $3$--$5$ training seeds). Rank $10$ is clearly worse
at $11$ and $21$ NFE, while $r\ge30$ is within seed variation, so the
returns diminish quickly beyond $r\approx30$--$50$. The full-rank row uses
$\sigma^2=0$ and $L=U\Lambda^{1/2}$.

\begin{table}[H]
\centering
\small
\caption{MNIST FID$_{50k}$ vs.\ covariance rank $r$ (lightweight configuration).}
\label{tab:rank}
\begin{tabular}{lccc}
\toprule
$r$ & NFE $6$ & NFE $11$ & NFE $21$ \\
\midrule
$10$  & $3.18\pm0.06$ & $3.47\pm0.59$ & $4.35\pm0.89$ \\
$30$  & $3.25\pm0.31$ & $2.60\pm0.22$ & $2.86\pm0.36$ \\
$50$  & $3.14\pm0.45$ & $2.38\pm0.11$ & $2.60\pm0.13$ \\
$200$ & $3.22\pm0.33$ & $2.36\pm0.18$ & $2.64\pm0.14$ \\
$784$ (full) & $2.90\pm0.37$ & $2.40\pm0.16$ & $2.82\pm0.24$ \\
\bottomrule
\end{tabular}
\end{table}

\subsection{Seed variability}
\label{sec:seeds}
For the MNIST headline model, five sampling seeds give
FID$_{50k}=2.958\pm0.032$, $1.136\pm0.017$, and $0.772\pm0.011$ at $11$,
$21$, and $51$ NFE; a model retrained from scratch with a different
training seed gives $2.994$, $1.125$, and $0.747$. The spread across
seeds is thus small compared with the differences between methods in
Table~1 of the main paper. On the toy targets, the standard deviation of
sliced-$W_2$ across training seeds is $0.01$--$0.11$. The CIFAR-10,
ImageNet-32, and CelebA-HQ results are single training runs.

\subsection{KDE source and equal-reference memorization}
\label{sec:kde}
A natural alternative to the coreset surrogate is a kernel density
estimate (KDE) on all training points, $x_j+h\varepsilon$. In the
lightweight MNIST configuration, a KDE source with $h=0.05$ gives
FID$_{50k}$ $68$--$82$ at $6$ NFE and is unstable across training seeds
($2.1$ vs.\ $34.9$ at $21$ NFE); $h=0.2$ gives $12$--$36$. With an equal
reference set, KDE samples are $13\%$ closer to training images than to
test images ($1$-NN distance $4.02$ vs.\ $4.65$), whereas CCVFM samples
are $5\%$ closer ($4.35$ vs.\ $4.58$). The compact coreset surrogate
therefore avoids both the instability and the stronger training-set
proximity of a near-empirical source.

\subsection{Compression-scaling diagnostics}
\label{sec:compression}
We fit the slope of $\log$ error vs.\ $\log K$ (Table~\ref{tab:compression}).
On 2D targets (sliced $W_2$, $3$ seeds) the slopes match the nominal
$-1/d$; on a helix in $\R^3$ the slope is steeper than the nominal $-1/3$,
corresponding to $d_{\mathrm{eff}}\approx2.3$. On images we use the
Stage-I transport cost as a proxy for $W_2^2$, so its slope estimates
$-2/d_{\mathrm{eff}}$, giving $d_{\mathrm{eff}}\approx12$ (MNIST) and
$\approx15$ (CIFAR-10), far below the nominal $784$ and $3072$. These
are empirical diagnostics for reading Assumption~\ref{ass:stageI-compression},
not proofs of it.

\begin{table}[H]
\centering
\small
\caption{Compression slopes in $K$.}
\label{tab:compression}
\begin{tabular}{lrrrr}
\toprule
Target & $d$ & fitted slope & nominal & $d_{\mathrm{eff}}$ \\
\midrule
thin circle & 2    & $-0.503\pm0.013$ & $-0.500$ & -- \\
two moons   & 2    & $-0.467\pm0.016$ & $-0.500$ & -- \\
helix       & 3    & $-0.441\pm0.022$ & $-0.333$ & $2.3$ \\
MNIST (proxy)    & 784  & $-0.162$ & -- & $12.3$ \\
CIFAR-10 (proxy) & 3072 & $-0.136$ & -- & $14.7$ \\
\bottomrule
\end{tabular}
\end{table}

\subsection{Manifold stress tests: velocity- vs.\ data-space parameterization}
\label{sec:manifold}
At $J=1$ the velocity-space and data-space formulations coincide at the
population optimum (\S\ref{sec:j1-reduction-supp}), but trained
estimators can differ: the data-space form trains a flow directly from
$Y_0\sim\core$ to $Y_1\sim\rho_1$, whereas the velocity-space form trains
on $(V_0,V_1)=(Y_0-X_0,\,Y_1-X_0)$ with $x_0$ visible to the network, so
both endpoints are perturbed by the same observed Gaussian draw. Finite
networks trained by SGD are not invariant under this input
reparameterization. We test this on manifold targets with
deliberately misspecified surrogates that cover only part of the
manifold, using the same backbone, optimizer, and budget, the direct
source draw, $10$ learned evaluations, and $5$ training seeds; the $x_0$
input adds less than $0.5\%$ first-layer parameters. Coverage-TV
measures the deviation of the histogram of generated samples along the
manifold from uniform (lower is better). The velocity-space form covers the manifold more evenly on all
targets in Table~\ref{tab:manifold}; sliced-$W_2$ gives the same
ordering.

\begin{table}[H]
\centering
\small
\caption{Coverage-TV (mean $\pm$ std over $5$ seeds) under partially covering surrogates.}
\label{tab:manifold}
\begin{tabular}{lcc}
\toprule
Target (surrogate coverage) & data-space & velocity-space \\
\midrule
3D helix (bottom $40\%$) & $0.232\pm0.025$ & $\mathbf{0.159\pm0.044}$ \\
3D ring ($90^\circ$ arc) & $0.285\pm0.021$ & $\mathbf{0.243\pm0.028}$ \\
2D spiral (inner turn)   & $0.166\pm0.072$ & $\mathbf{0.117\pm0.017}$ \\
\bottomrule
\end{tabular}
\end{table}

\section{Toy 2D results}

We include 2D results to complement the image benchmarks. The
{quantitative comparison in Table~\ref{tab:toy-supp}} is between
(i)~standard mean-field flow matching (one-step and multi-step),
(ii)~the raw GMM samples from Stage~I, (iii)~the closed-form
Stage~II one-step surrogate generator, {and (iv)~the Stage~III
corrected one-step generator. Figure~\ref{fig:toy-grids-supp}
visualizes (i)--(iii) only; the Stage~III $1$-step generator is reported
numerically rather than as a sample grid because at this budget
its samples are visually indistinguishable from Stage~II's
(see the discussion paragraph following Table~\ref{tab:toy-supp}).}

\begin{figure}[H]
\centering
\begin{tabular}{cc}
\includegraphics[width=0.46\linewidth]{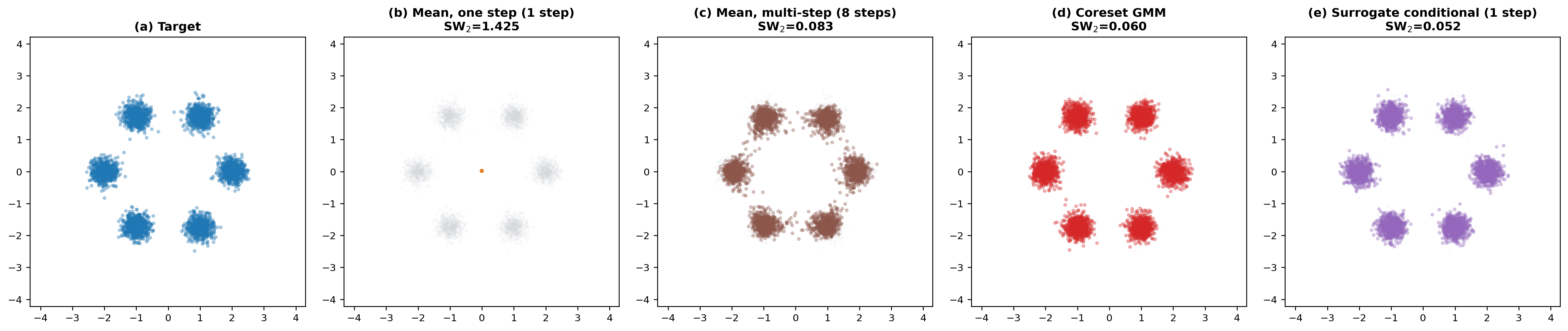} &
\includegraphics[width=0.46\linewidth]{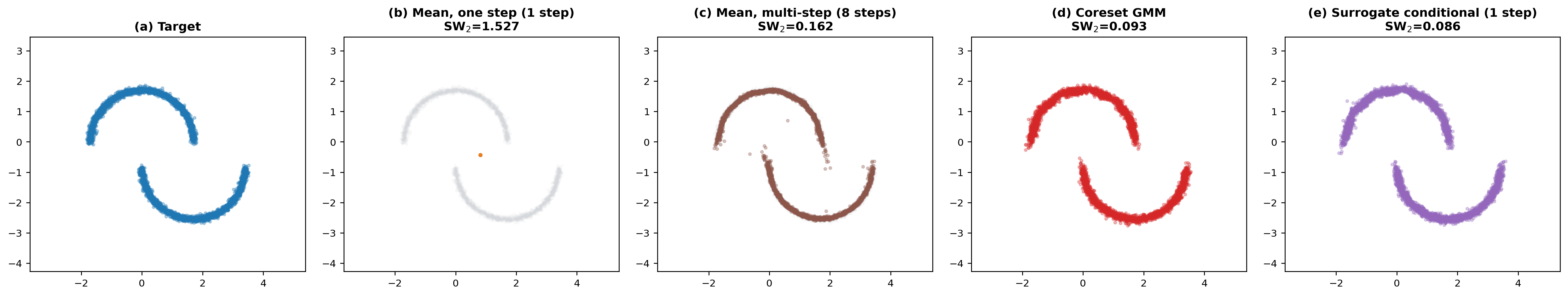} \\
\small (a) ring-6 (2D) & \small (b) two-moons (2D) \\[0.4em]
\includegraphics[width=0.46\linewidth]{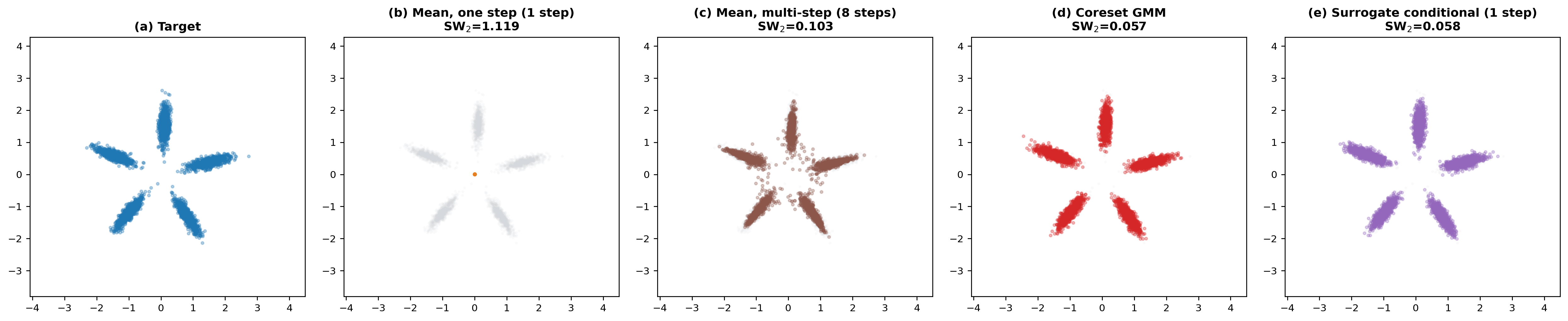} &
\includegraphics[width=0.46\linewidth]{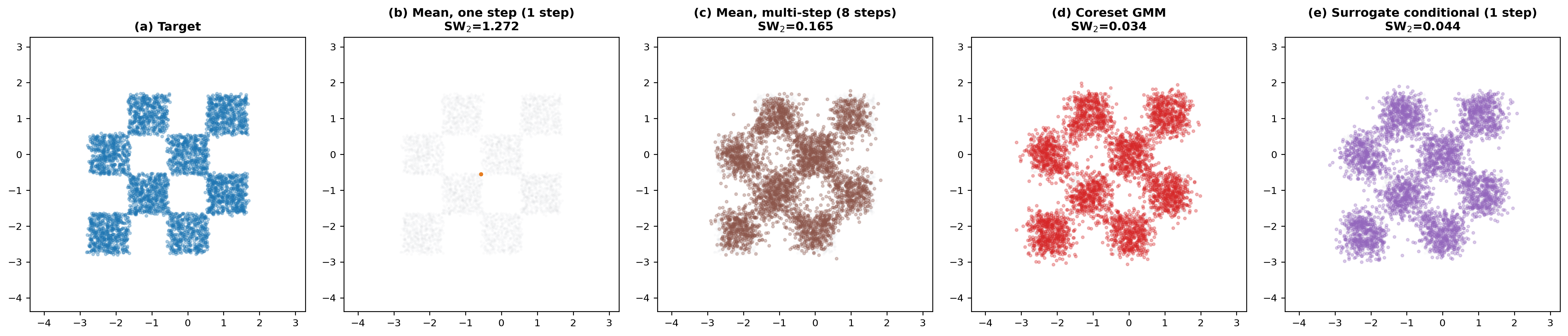} \\
\small (c) pinwheel (2D) & \small (d) checkerboard (2D) \\[0.4em]
\multicolumn{2}{c}{\includegraphics[width=0.75\linewidth]{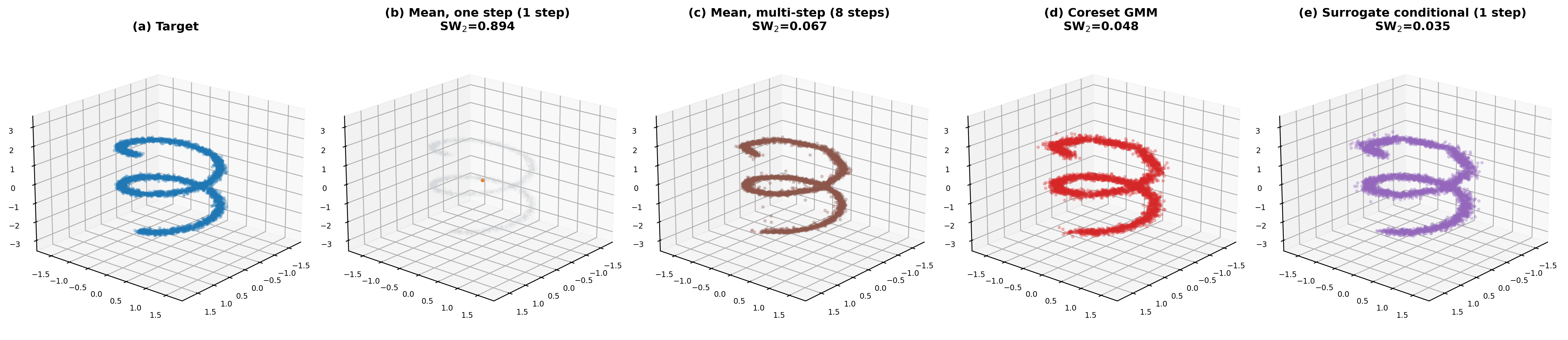}} \\
\multicolumn{2}{c}{\small (e) helix (3D)} \\
\end{tabular}
\caption{Toy benchmarks across five synthetic targets. Each panel
compares (columns, left-to-right): (i)~ground-truth samples,
(ii)~rectified-flow mean-field 1-step output, (iii)~rectified-flow
mean-field multi-step output, (iv)~Stage~I Sinkhorn-coreset samples,
(v)~\textbf{CCVFM Stage~II closed-form 1-step generator}. Even at
$1$~NFE with \emph{no learned neural network}, CCVFM Stage~II
matches the multi-step rectified-flow quality on all five datasets
and avoids the mode-dropping that plagues 1-step mean-field methods
(ring-6, pinwheel, helix).}
\label{fig:toy-grids-supp}
\end{figure}

\begin{figure}[H]
\centering
\includegraphics[width=0.92\linewidth]{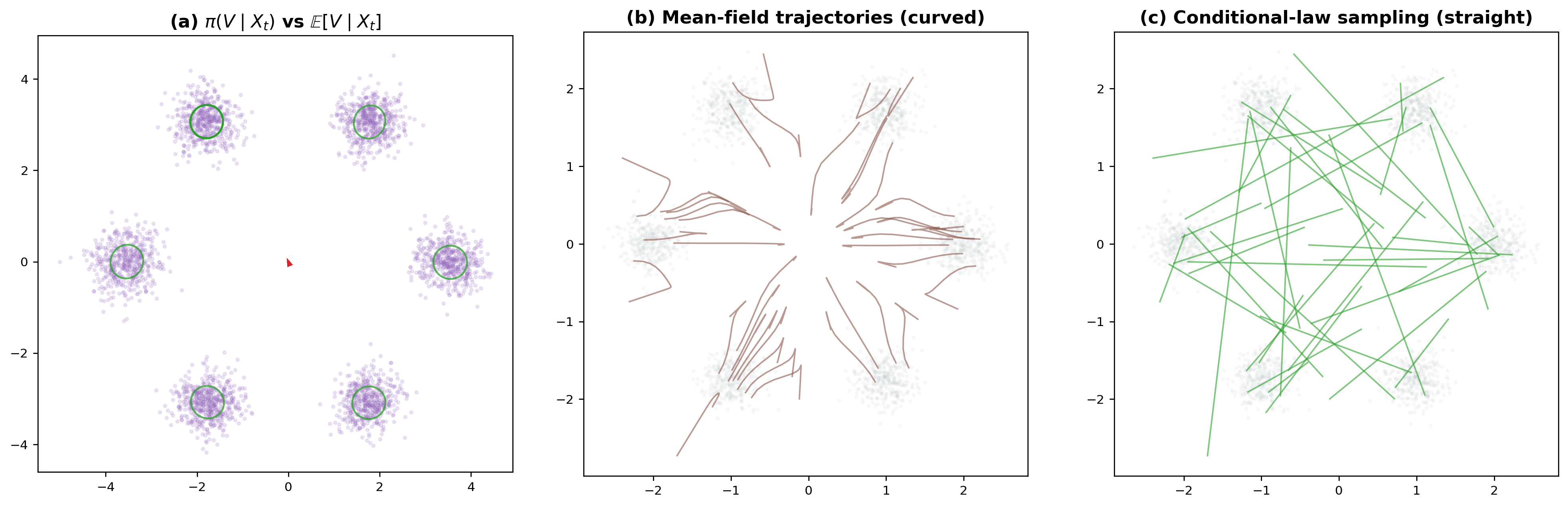}
\caption{Conditional-velocity advantage on ring-6. Left: the true
conditional velocity law $\pi(v\mid x_t,t)$ at a specific
intermediate $(x_t,t)$ is multimodal (six modes, one per ring
cluster). Middle: rectified-flow targets the conditional
\emph{mean}, which collapses the six modes to a single direction
pointing through the center of the ring, a \emph{density-valley}
output. Right: CCVFM's closed-form GMM $\tilde\pi(v\mid x_t,t)$
recovers the six-mode structure exactly, so sampling from
$\tilde\pi$ yields a valid single-step velocity that lands on one
of the ring modes. This is the 2D visualization of the
mean-collapse critique in~\S 1 of the main paper.}
\label{fig:toy-cond-adv-supp}
\end{figure}

\begin{figure}[H]
\centering
\includegraphics[width=0.85\linewidth]{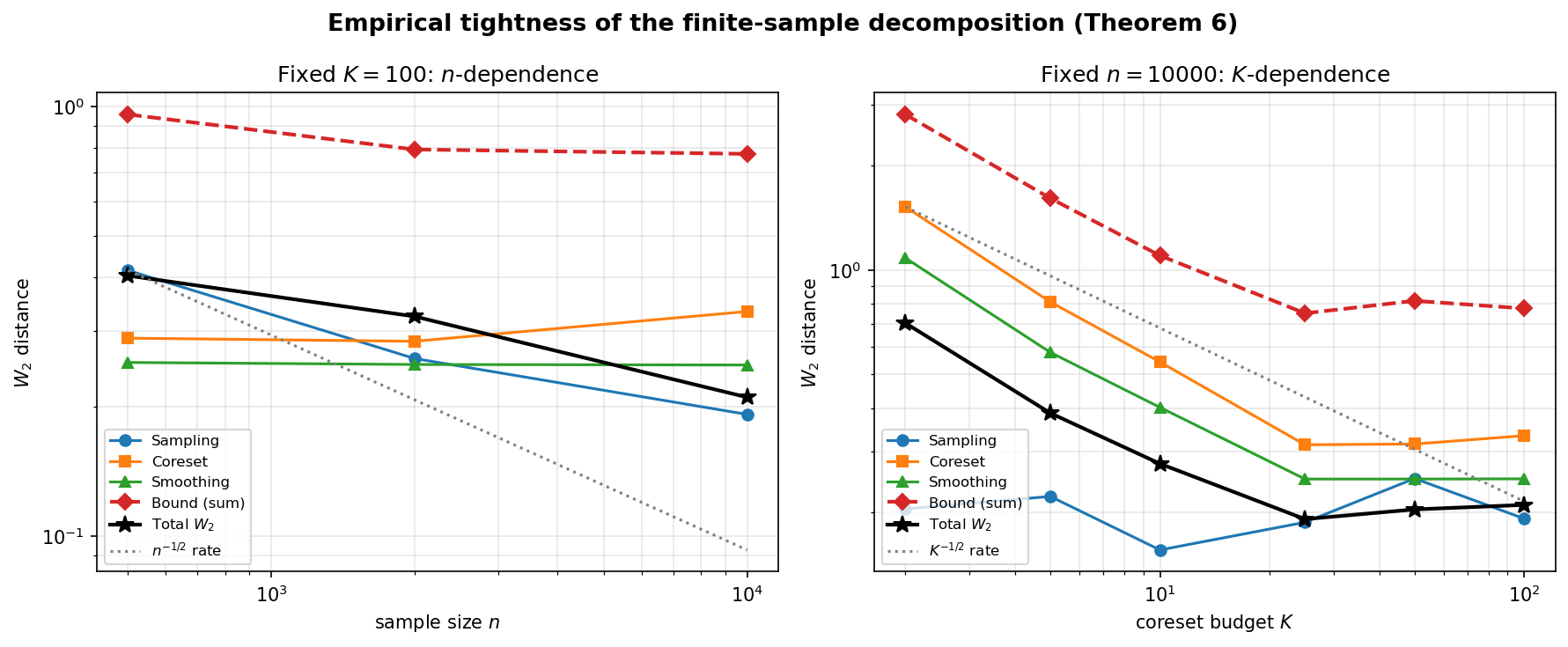}
\caption{Empirical illustration of the surrogate-gap decomposition on
the ring-6 target. For each $(n,K)$ pair we plot the three
bound terms from the generation analysis: sampling
($W_2(\rho_1,\hat\rho_{n})$), coreset
($\mathcal{E}_{\mathrm{cs}}^{1/2}=W_2(\hat\rho_n,\sum w_k\delta_{\mu_k})$)
and smoothing ($\sigma_{\mathrm{gmm}}$), alongside the measured
generation error $W_2(\Law(\hat X_1),\rho_1)$. Across $n\in\{500,2000,10000\}$
and $K\in\{2,5,10,25,50,100\}$, the measured error tracks the
empirical upper bound within a tightness ratio of $0.2$--$0.4$,
and the coreset-compression term dominates for small $K$ while the
sampling term dominates once $K\gtrsim n^{1/2}$, which is the regime
where the coreset shortcut becomes free.}
\label{fig:toy-theorem-supp}
\end{figure}

\begin{table}[H]
\centering
\small
\caption{Sliced $W_2^2$ and mode-balance TV on toy 2D datasets
(ring-6, moons, pinwheel, helix-3D). Lower is better.
{Metrics: SW$_2$ is the sliced quadratic Wasserstein distance,
$\frac{1}{200}\sum_{u\in\mathcal{U}}W_2^2(u_\#\hat\rho,u_\#\rho_1)$
over $200$ uniform projections $u\in\mathcal{U}$ on $\mathbb{S}^{d-1}$.
``mode-TV'' is the total-variation distance between the empirical mode-occupancy
histogram of generated samples and the uniform mode prior, where each sample is
assigned to its closest ground-truth mode. ``helix-dist'' is the mean point-wise distance
of generated samples to the analytic helix curve.}
}
\label{tab:toy-supp}
\setlength{\tabcolsep}{3pt}
\begin{tabular}{llccc}
\toprule
Dataset & Method & SW$_2$ & extra metric \\
\midrule
\multirow{5}{*}{ring-6}
 & Mean-field 1-step & 1.425 & 0.833 (mode-TV) \\
 & Mean-field multi-step & 0.084 & 0.016 (mode-TV) \\
 & Stage I samples & 0.060 & 0.009 (mode-TV) \\
 & \textbf{CCVFM Stage II 1-step} & \textbf{0.052} & \textbf{0.008 (mode-TV)} \\
 & CCVFM Stage III 1-step & 0.051 & 0.008 (mode-TV) \\
\midrule
\multirow{5}{*}{moons}
 & Mean-field 1-step & 1.527 & --- \\
 & Mean-field multi-step & 0.162 & --- \\
 & Stage I samples & 0.093 & --- \\
 & \textbf{CCVFM Stage II 1-step} & \textbf{0.087} & --- \\
 & CCVFM Stage III 1-step & 0.127 & --- \\
\midrule
\multirow{5}{*}{pinwheel}
 & Mean-field 1-step & 1.119 & 0.800 (mode-TV) \\
 & Mean-field multi-step & 0.103 & 0.003 (mode-TV) \\
 & Stage I samples & 0.057 & 0.011 (mode-TV) \\
 & \textbf{CCVFM Stage II 1-step} & \textbf{0.058} & 0.010 (mode-TV) \\
 & CCVFM Stage III 1-step & 0.177 & 0.010 (mode-TV) \\
\midrule
\multirow{5}{*}{helix-3D}
 & Mean-field 1-step & 0.894 & 0.999 (helix-dist) \\
 & Mean-field multi-step & 0.067 & 0.057 (helix-dist) \\
 & Stage I samples & 0.048 & 0.086 (helix-dist) \\
 & \textbf{CCVFM Stage II 1-step} & \textbf{0.035} & 0.086 (helix-dist) \\
 & CCVFM Stage III 1-step & 0.134 & 0.216 (helix-dist) \\
\bottomrule
\end{tabular}
\end{table}

{\paragraph{Why Stage III at $1$ step does not improve on Stage II at $1$ step on these toys.}
On low-dimensional toys with a clean modal structure (ring-6,
moons, pinwheel, helix-3D), the surrogate gap
$W_2(\rho_1,\core)$ is already very small at $K\in\{6,\ldots,32\}$,
so the closed-form Stage~II $1$-step generator essentially
saturates the achievable error: there is little within-mode
residual for Stage~III to absorb. At a budget of $L{=}1$
inner correction step, the learned residual is then dominated
by the stochastic Euler discretization error rather than by a
useful systematic correction, which manifests as an
SW$_2$ \emph{increase} on moons (0.087~$\to$~0.127), pinwheel
(0.058~$\to$~0.177), and helix-3D (0.035~$\to$~0.134). The
Stage~III correction is designed to reduce the within-mode
residual at low NFE \emph{when the surrogate gap is non-negligible},
which is the regime of Tables~\ref{tab:mnist}
and~\ref{tab:hd-images} of the main paper; it is not designed to
help when Stage~II already attains the surrogate floor at
$1$-step. Increasing $L$ recovers Stage~II's error level but is
not a feature of the toy benchmark.}

\section{Additional sample grids}

\begin{figure}[H]
  \centering
  \begin{tabular}{ccc}
    \includegraphics[width=0.3\linewidth]{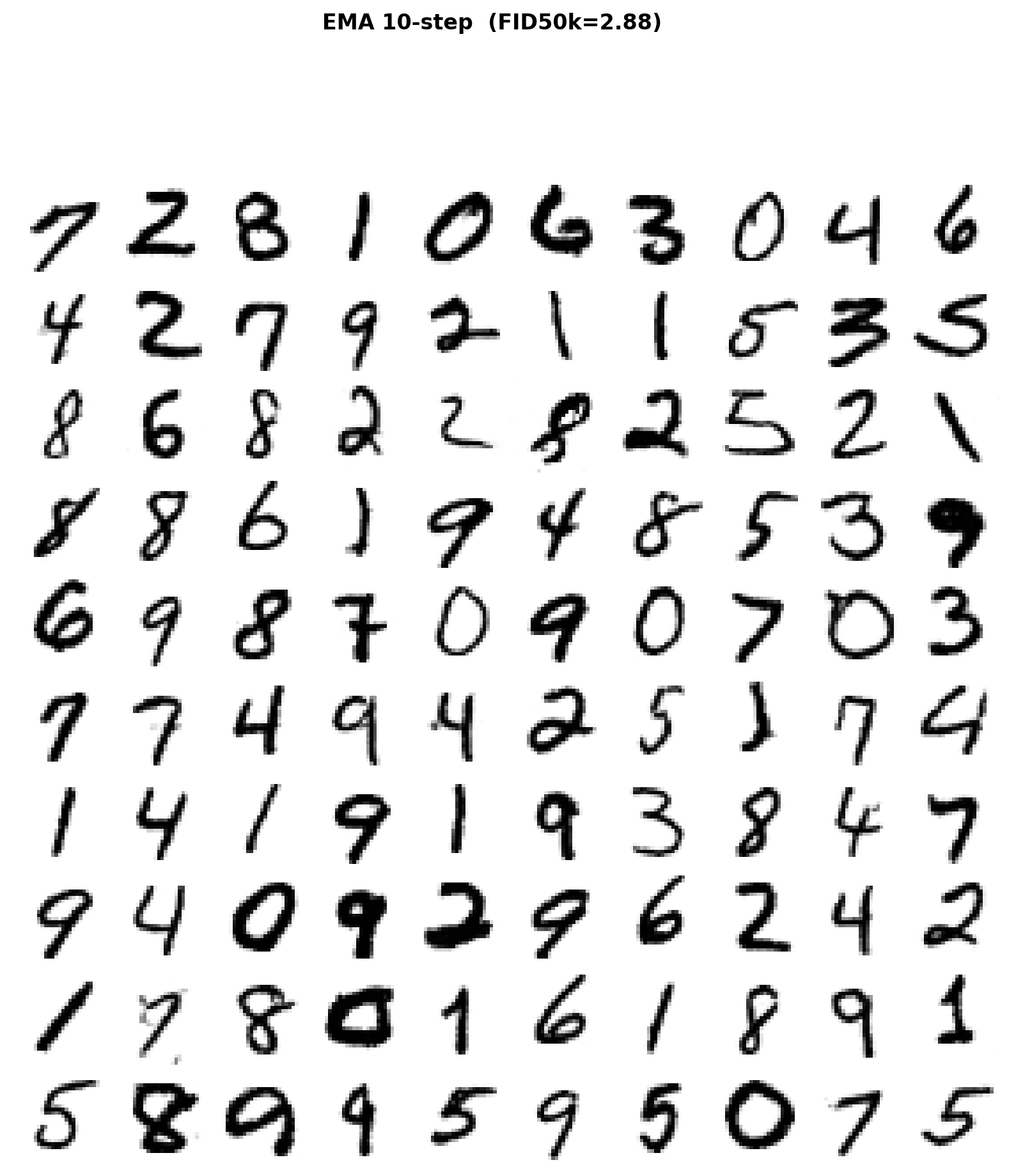} &
    \includegraphics[width=0.3\linewidth]{stage3_ema_20step_planC.png} &
    \includegraphics[width=0.3\linewidth]{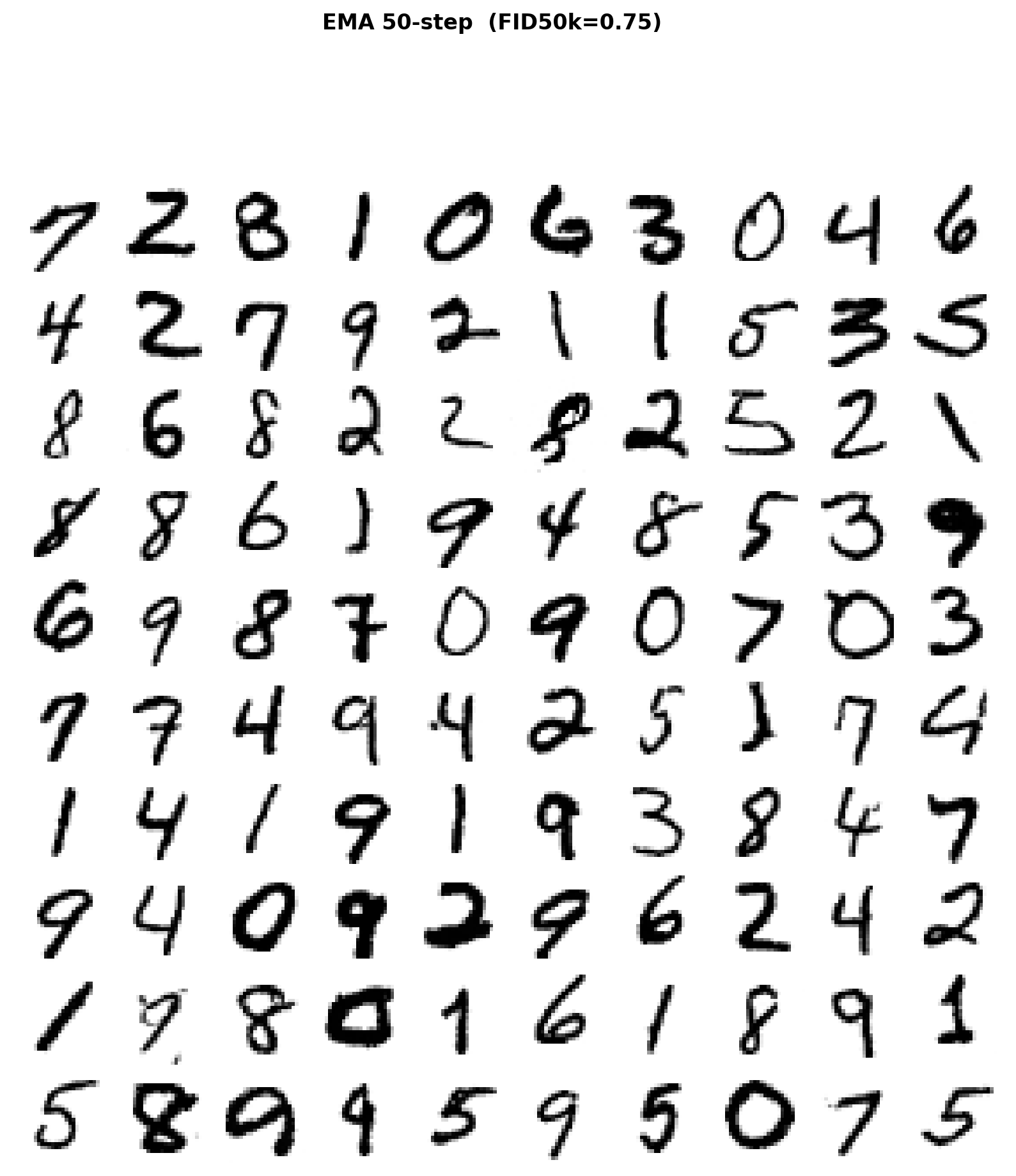} \\
    $L{=}10$ (FID$_{50k}{=}2.88$) &
    $L{=}20$ (FID$_{50k}{=}1.09$) &
    $L{=}50$ (FID$_{50k}{=}0.75$)
  \end{tabular}
  \caption{{MNIST 
  Stage III samples at
  $L\in\{10,20,50\}$ correction steps.}
  Quality sharpens from $L{=}10$ to $L{=}20$ and
  continues improving at $L{=}50$, mirroring the FID column.}
  \label{fig:mnist-grids}
\end{figure}

\begin{figure}[H]
  \centering
  \begin{tabular}{ccc}
    \includegraphics[width=0.3\linewidth]{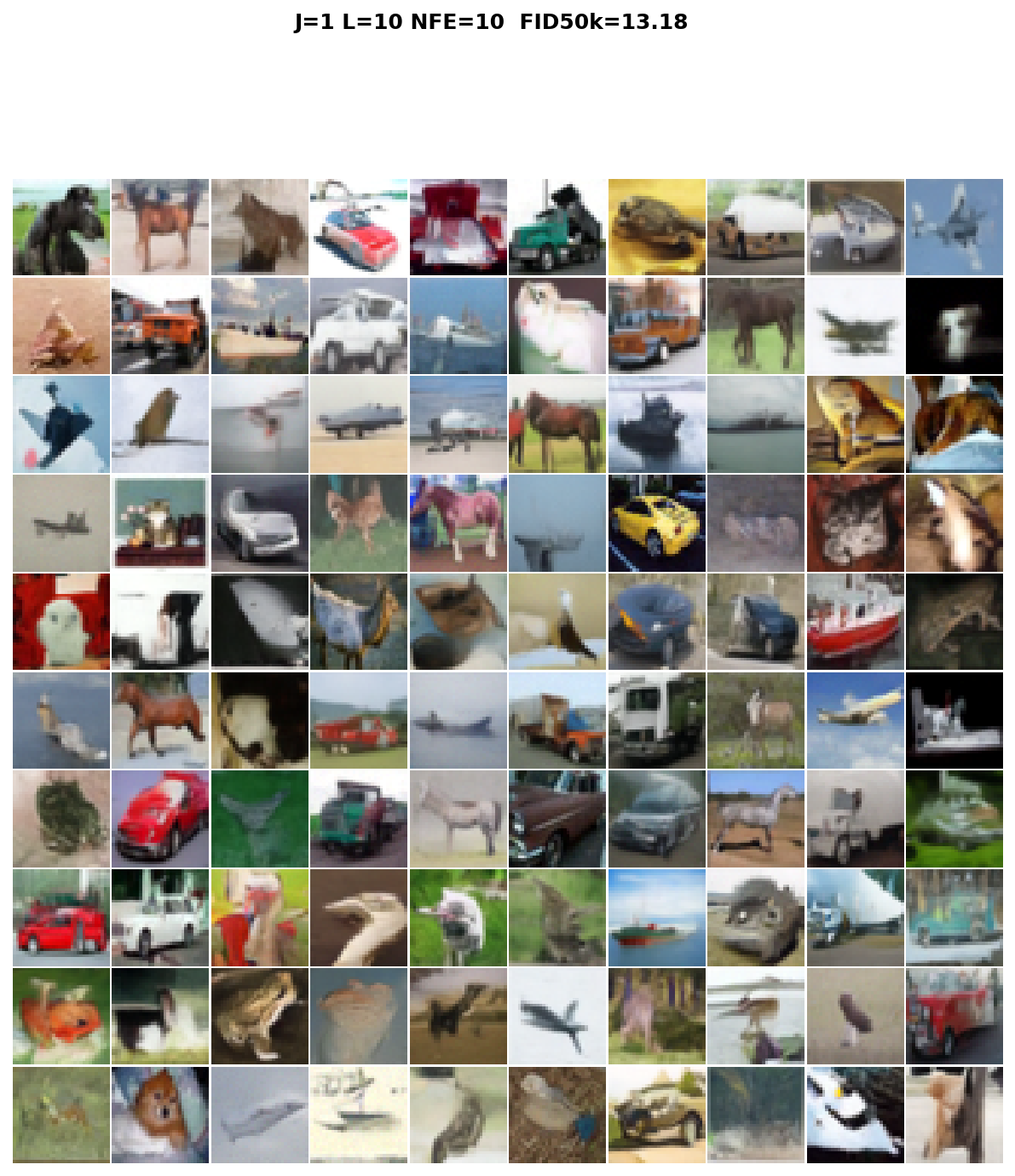} &
    \includegraphics[width=0.3\linewidth]{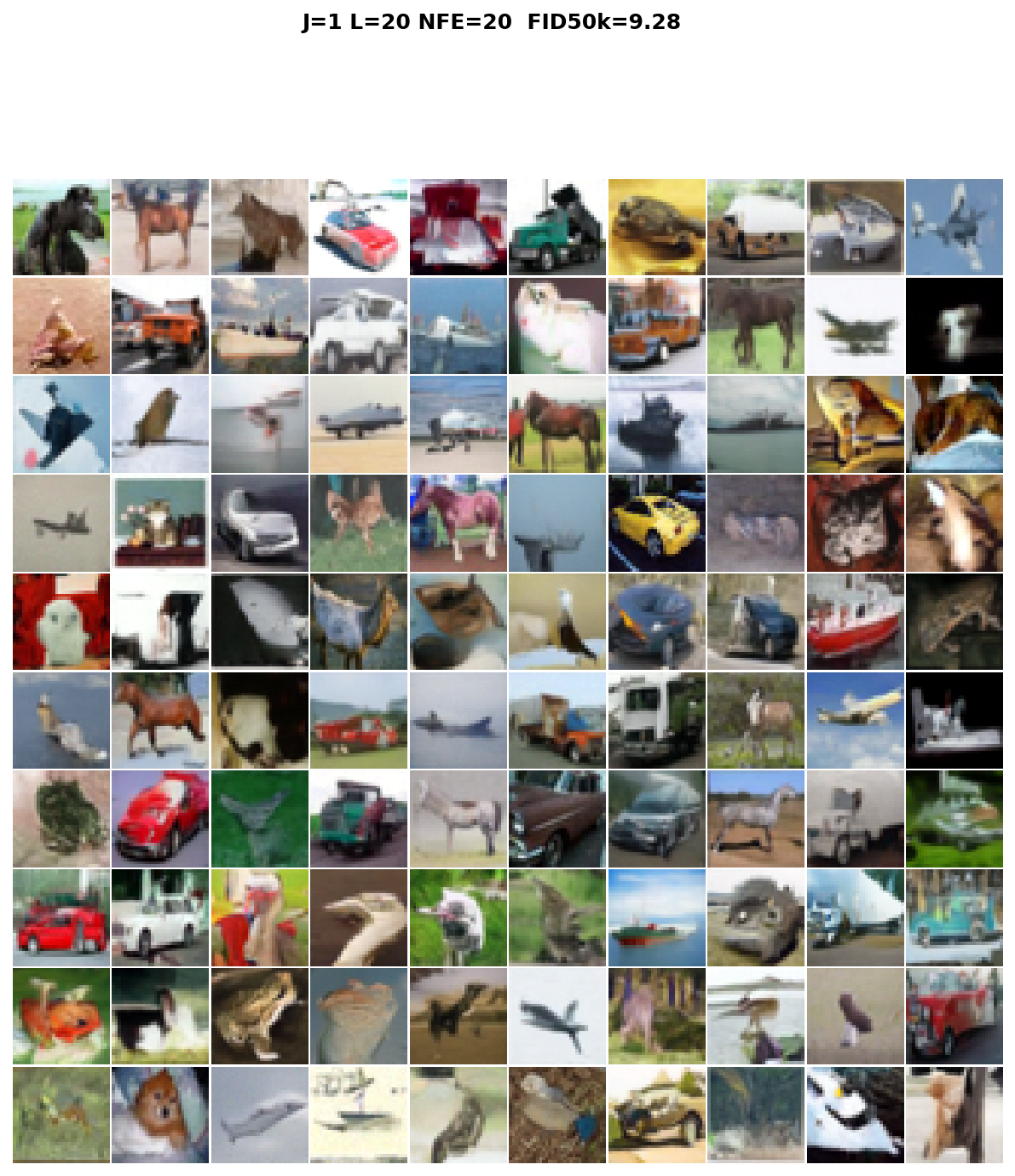} &
    \includegraphics[width=0.3\linewidth]{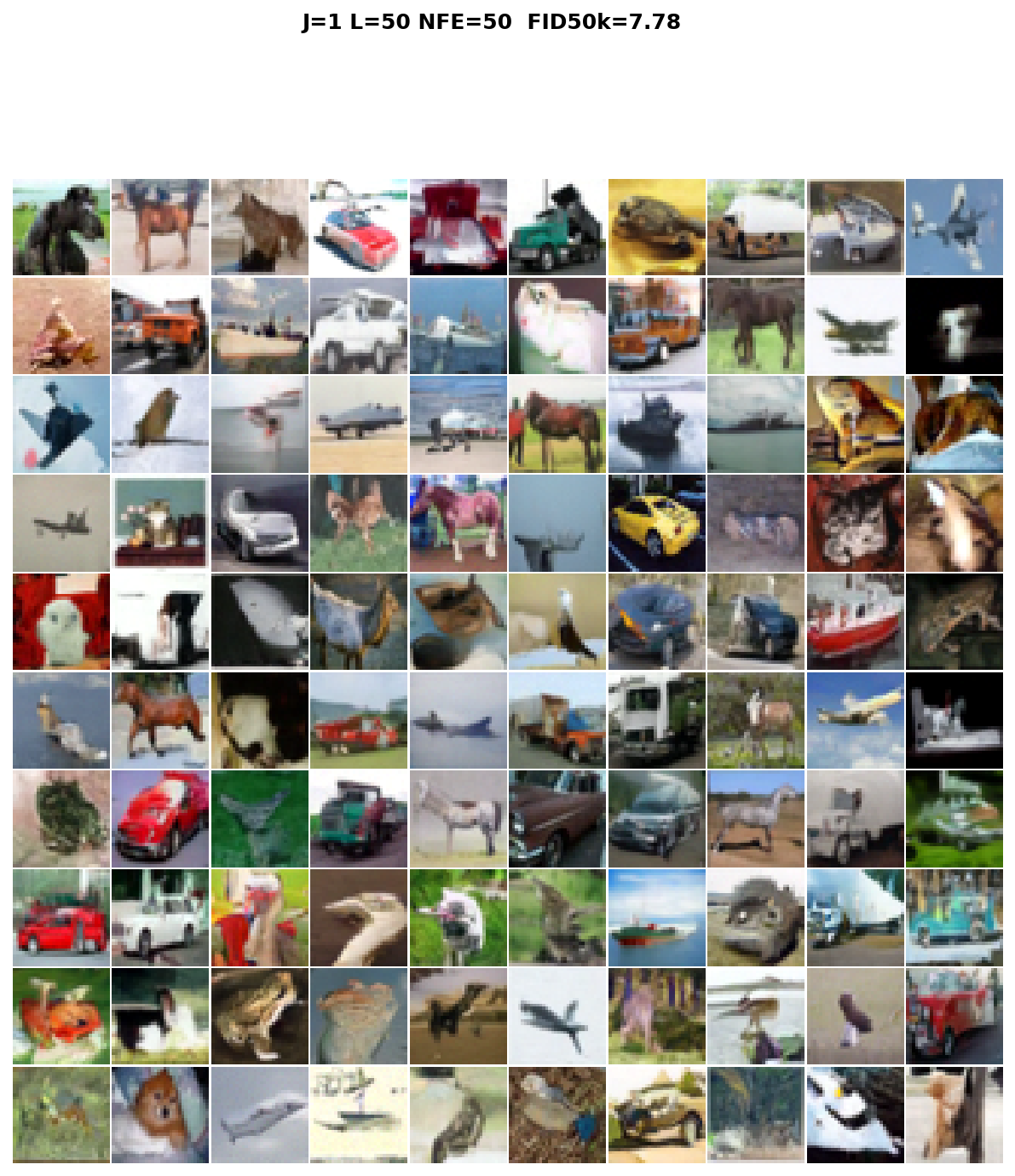} \\
    CIFAR, $L{=}10$ (FID 13.18) &
    $L{=}20$ (FID 9.28) &
    $L{=}50$ (FID 7.78)
  \end{tabular}
  \caption{CIFAR-10 Stage III samples across NFE budgets at $K{=}5000$.
  Shown for completeness alongside the $K{=}10000$ configuration
  (FID$_{50k}{=}6.35$ at $51$ NFE) reported in Figure~2 of the main
  paper.}
  \label{fig:cifar-grids}
\end{figure}

\begin{figure}[H]
  \centering
  \begin{tabular}{ccc}
    \includegraphics[width=0.28\linewidth]{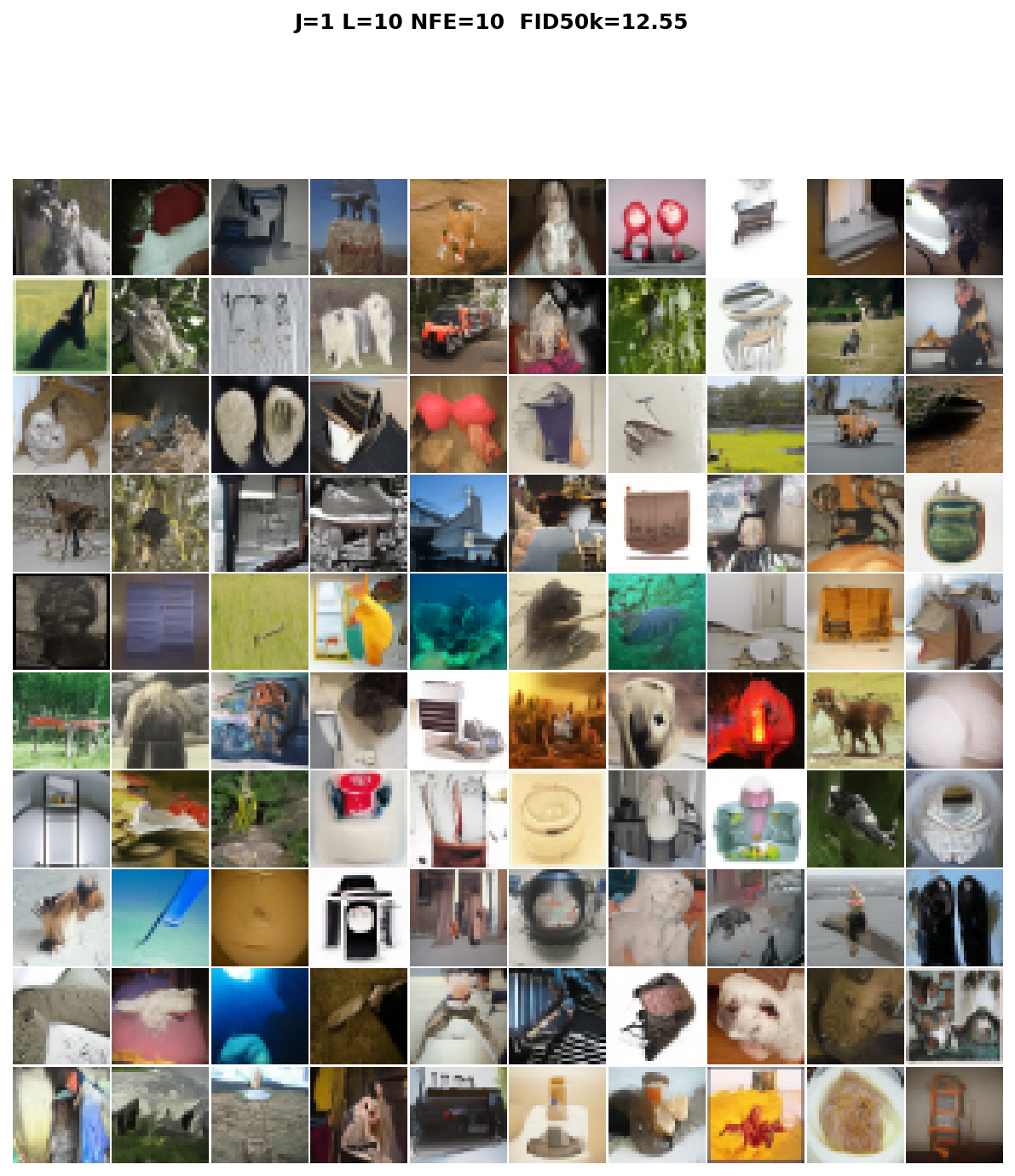} &
    \includegraphics[width=0.28\linewidth]{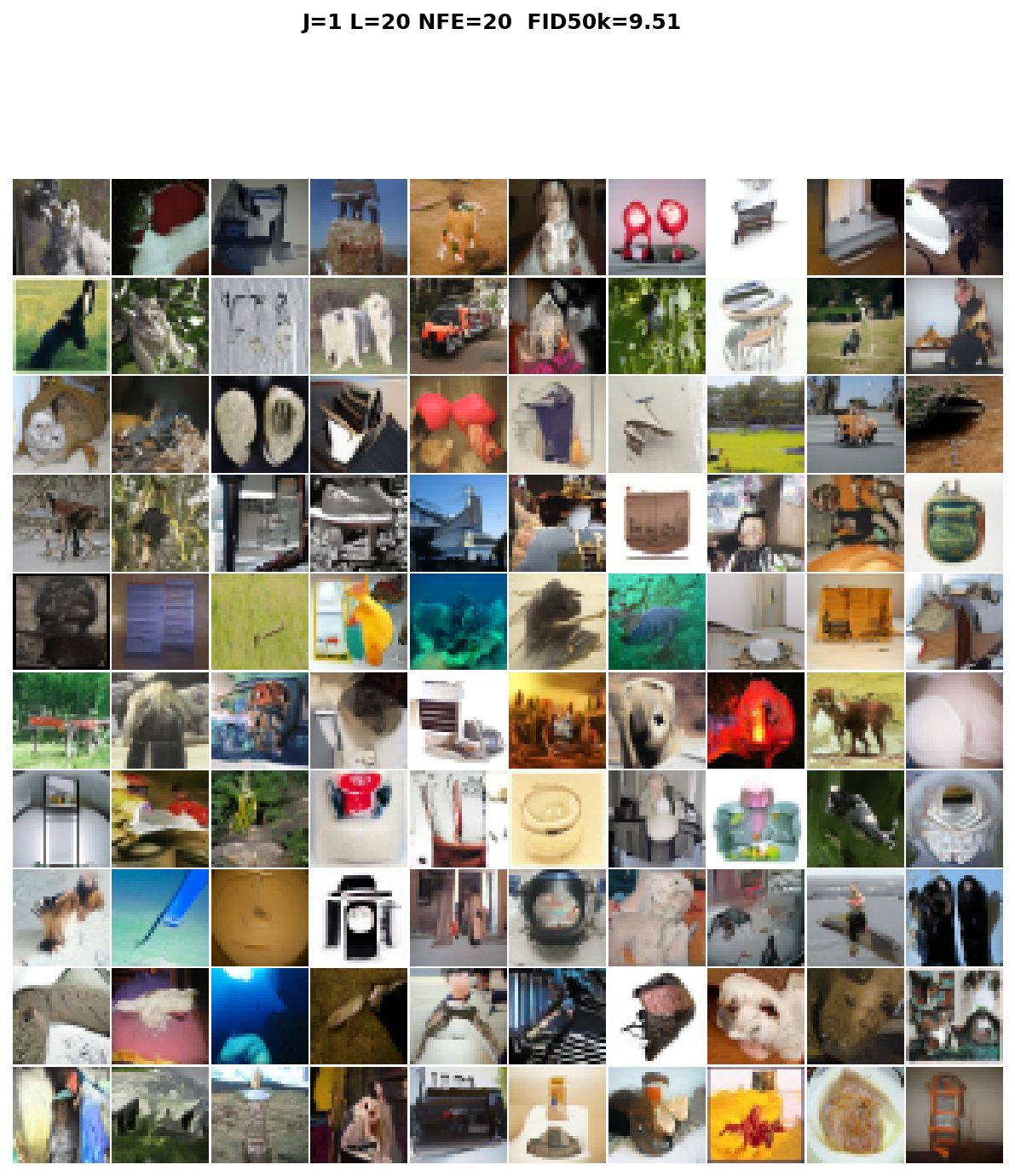} &
    \includegraphics[width=0.28\linewidth]{stage3_J1_L50_t_hrf2_resumed.png} \\
    ImageNet-32, $L{=}10$ (FID 12.55) &
    $L{=}20$ (FID 9.51) &
    $L{=}50$ (FID 8.76)
  \end{tabular}
  \caption{ImageNet-32 Stage III samples across NFE budgets.}
  \label{fig:imagenet-grids}
\end{figure}

\subsection{CelebA-HQ 256: extended sample grids}
\label{sec:celebahq-grids}

The CelebA-HQ benchmark (DC-AE f32c32 latent, $d{=}2048$,
$K{=}10000$, $r{=}128$, DiT-L correction with $\sim458$M parameters)
attains FID$_{28k}{=}\mathbf{4.17}$ at $51$ NFE without classifier-free
guidance. Figures~\ref{fig:celebahq-100}--\ref{fig:celebahq-pipeline}
show extended sample grids and a Stage~II$\to$III progression panel.
All samples are uncurated and drawn from the seed-99 30k generation
pool used to compute the headline FID.

\clearpage
\begin{figure}[!htbp]
  \centering
  \includegraphics[width=0.92\linewidth]{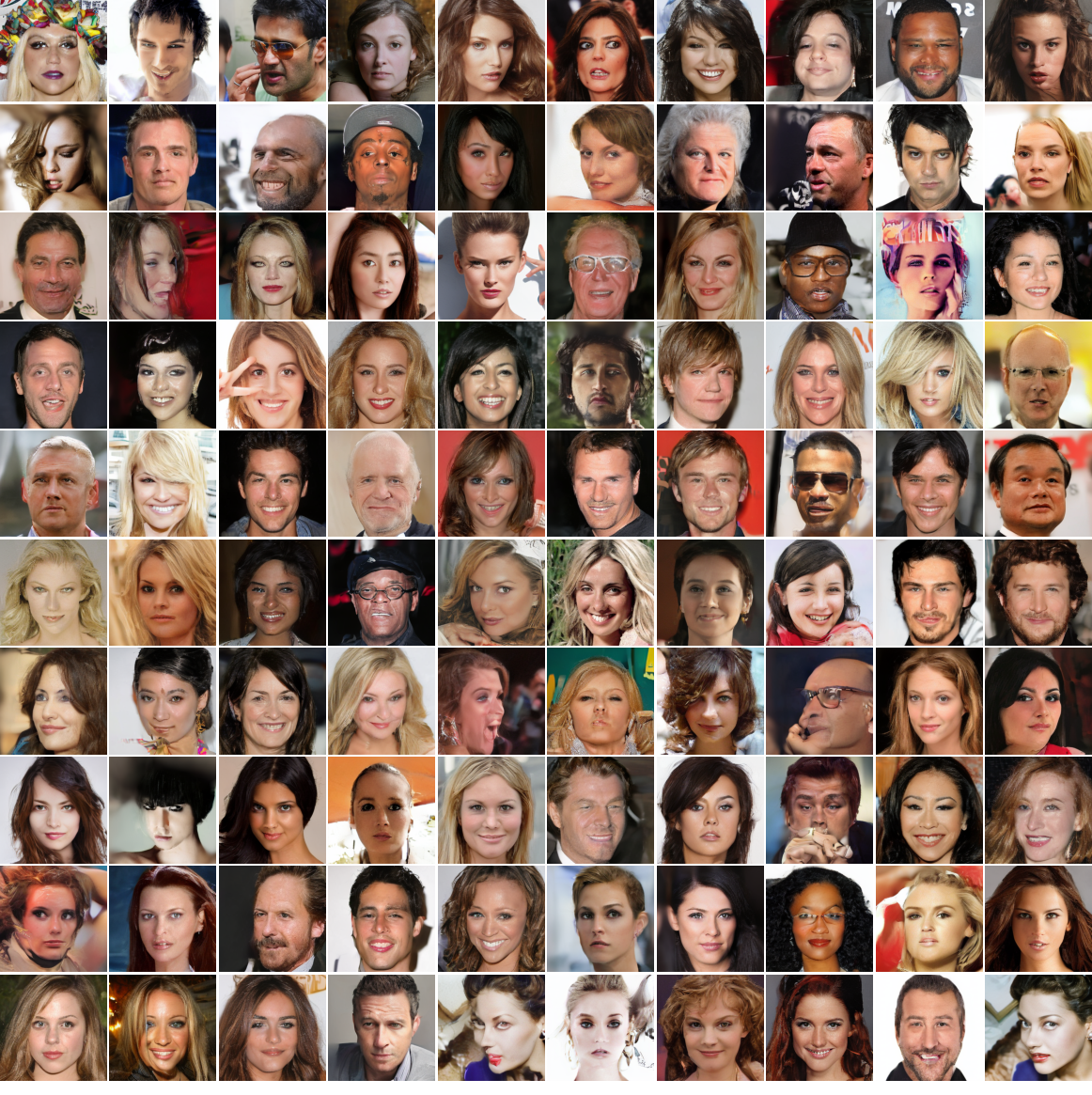}
  \caption{$10\times 10$ uncurated CelebA-HQ 256 samples from the
  CCVFM-L $L{=}50$ generator (FID$_{28k}{=}4.17$). The grid spans
  identities, lighting, expression, head pose, hair colour, and
  presence of glasses/accessories without any per-sample selection;
  the diversity confirms the surrogate $\pisur$ propagates the full
  CelebA-HQ identity manifold and not just a few high-density modes.}
  \label{fig:celebahq-100}
\end{figure}

\clearpage
\begin{figure}[!htbp]
  \centering
  \includegraphics[width=\linewidth]{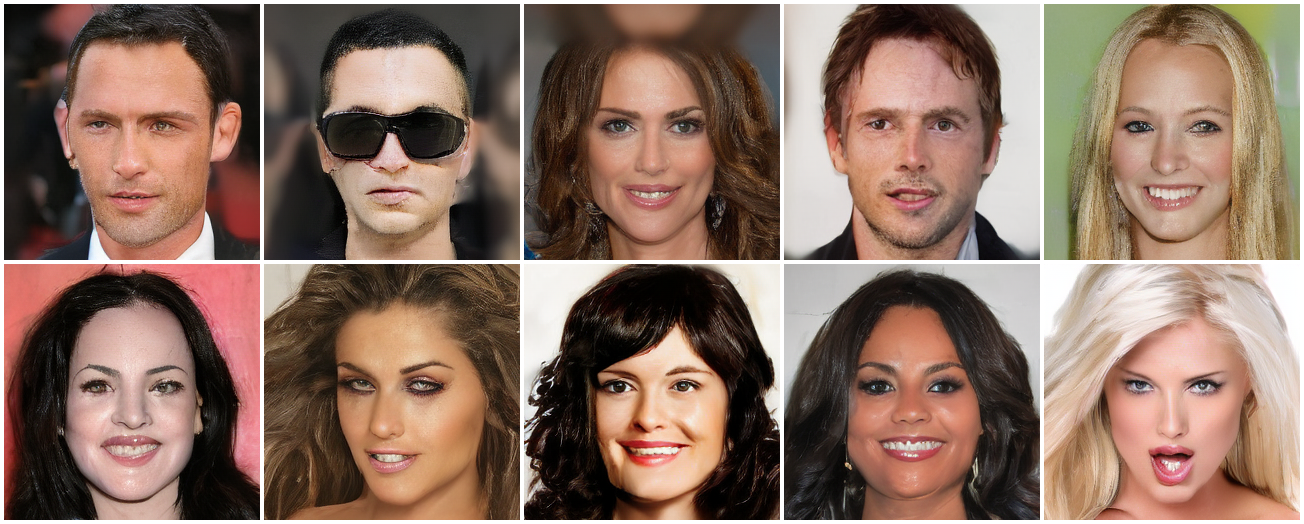}
  \caption{$2\times 5$ uncurated panel of CelebA-HQ samples, larger
  per-image resolution. Reproduces the panel used as Figure~2(d) in
  the main paper for completeness; useful for inspecting fine-grained
  artifacts (hair strands, eye pupils, skin texture) that the
  smaller-cell $10\times 10$ grid downsamples.}
  \label{fig:celebahq-2x5}
\end{figure}

\clearpage
\begin{figure}[!htbp]
  \centering
  \setlength{\tabcolsep}{2pt}
  \begin{tabular}{cc}
    \includegraphics[width=0.48\linewidth]{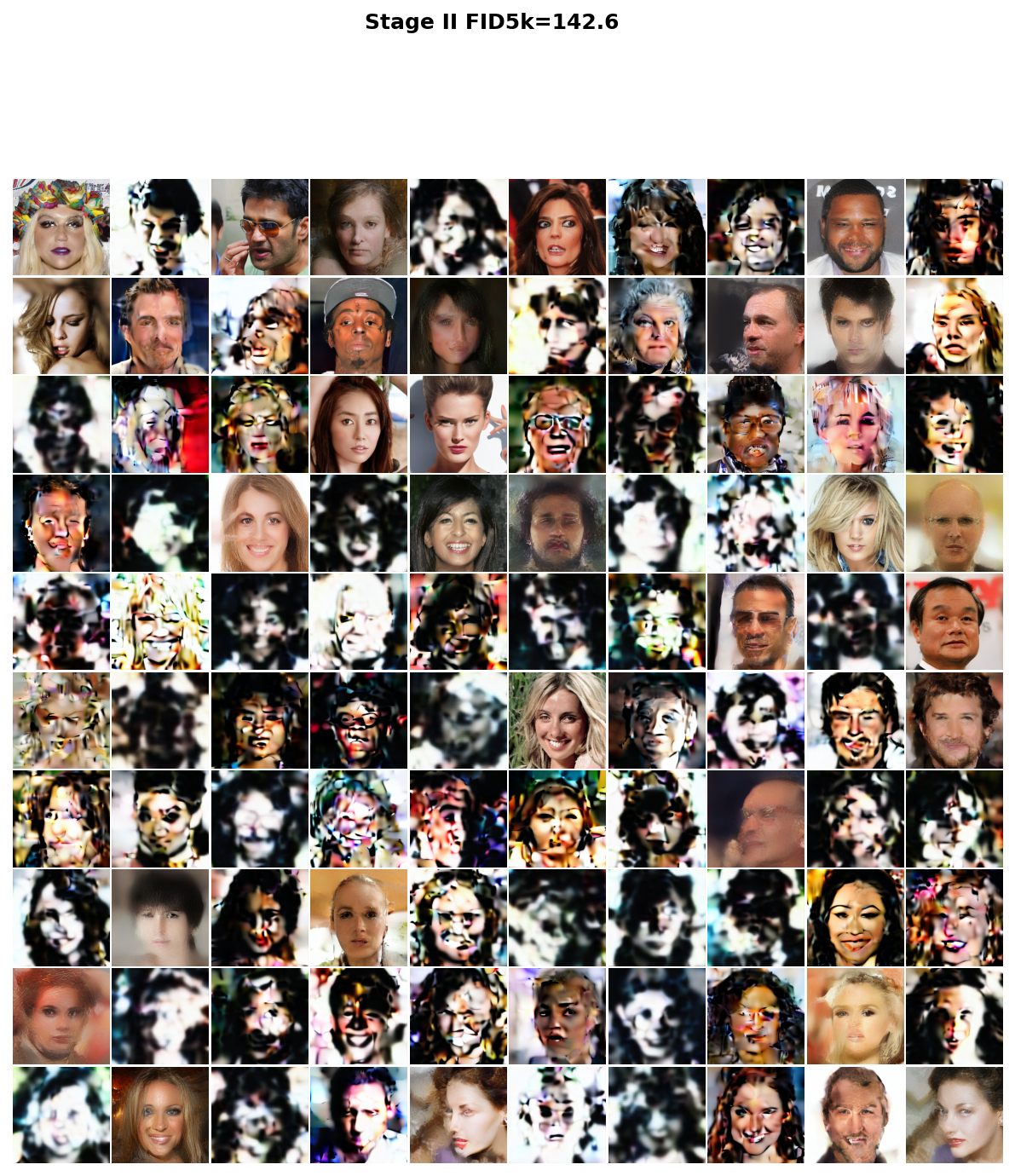} &
    \includegraphics[width=0.48\linewidth]{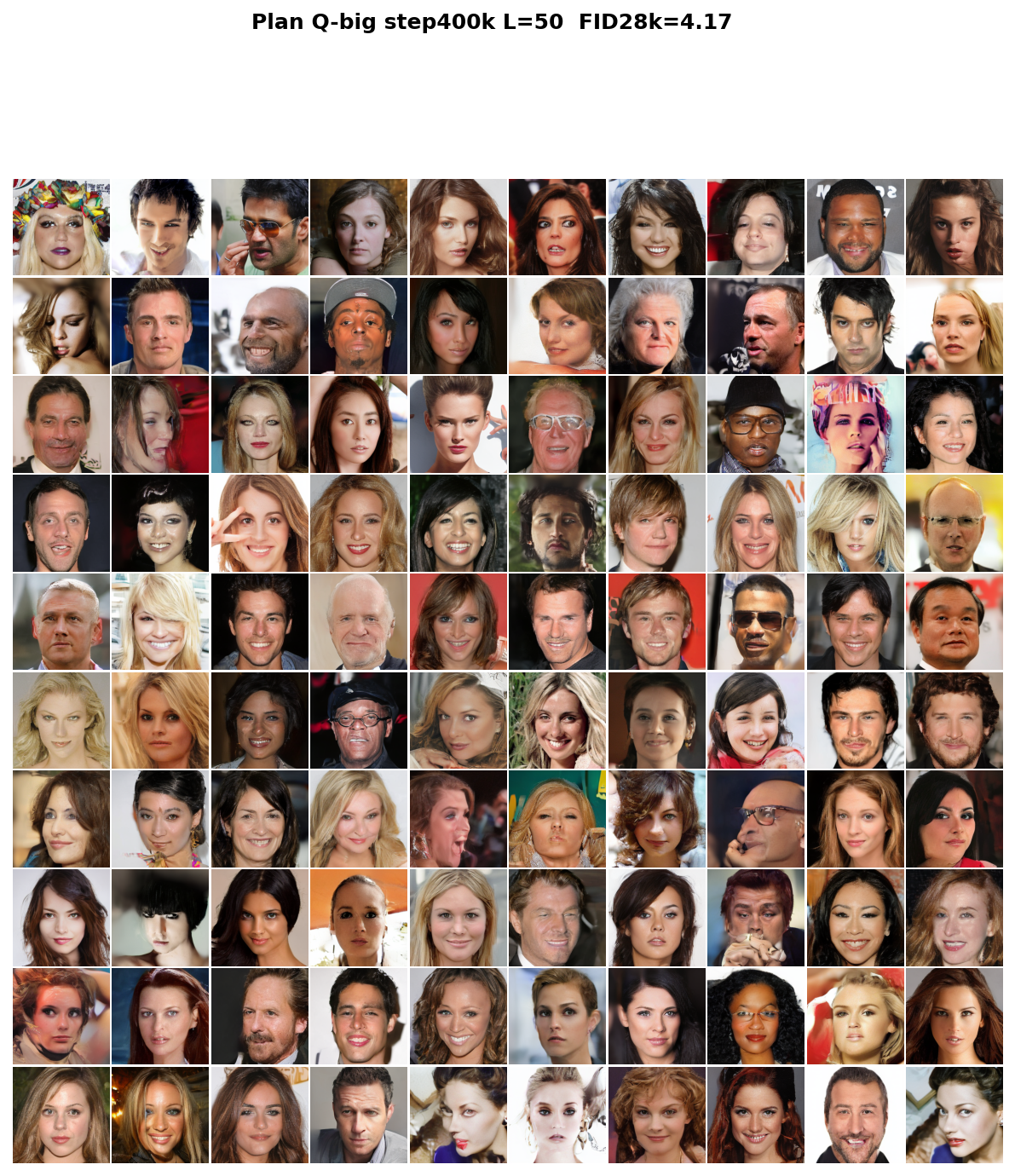} \\[2pt]
    \emph{(a) Stage II only ($1$ NFE, no correction).} &
    \emph{(b) Stage II $+$ Stage III correction ($L{=}50$).}
  \end{tabular}
  \caption{Stage~II$\to$III progression on CelebA-HQ. Panel (a) shows
  one-NFE samples drawn directly from the closed-form GMM surrogate
  $\pisur$ in DC-AE latent space and decoded; the global colour and
  pose distributions are correct but per-pixel detail is blurry. Panel
  (b) shows the same seeds after the correction flow integrates the
  residual for $L{=}50$ inner steps; faces sharpen, eyes/teeth gain
  high-frequency structure, and the FID drops from the
  Stage~II-only value to FID$_{28k}{=}4.17$. The contrast quantifies
  what the correction net buys: high-frequency residual, not global
  structure.}
  \label{fig:celebahq-pipeline}
\end{figure}

\clearpage
\begin{figure}[!htbp]
  \centering
  \includegraphics[width=0.82\linewidth]{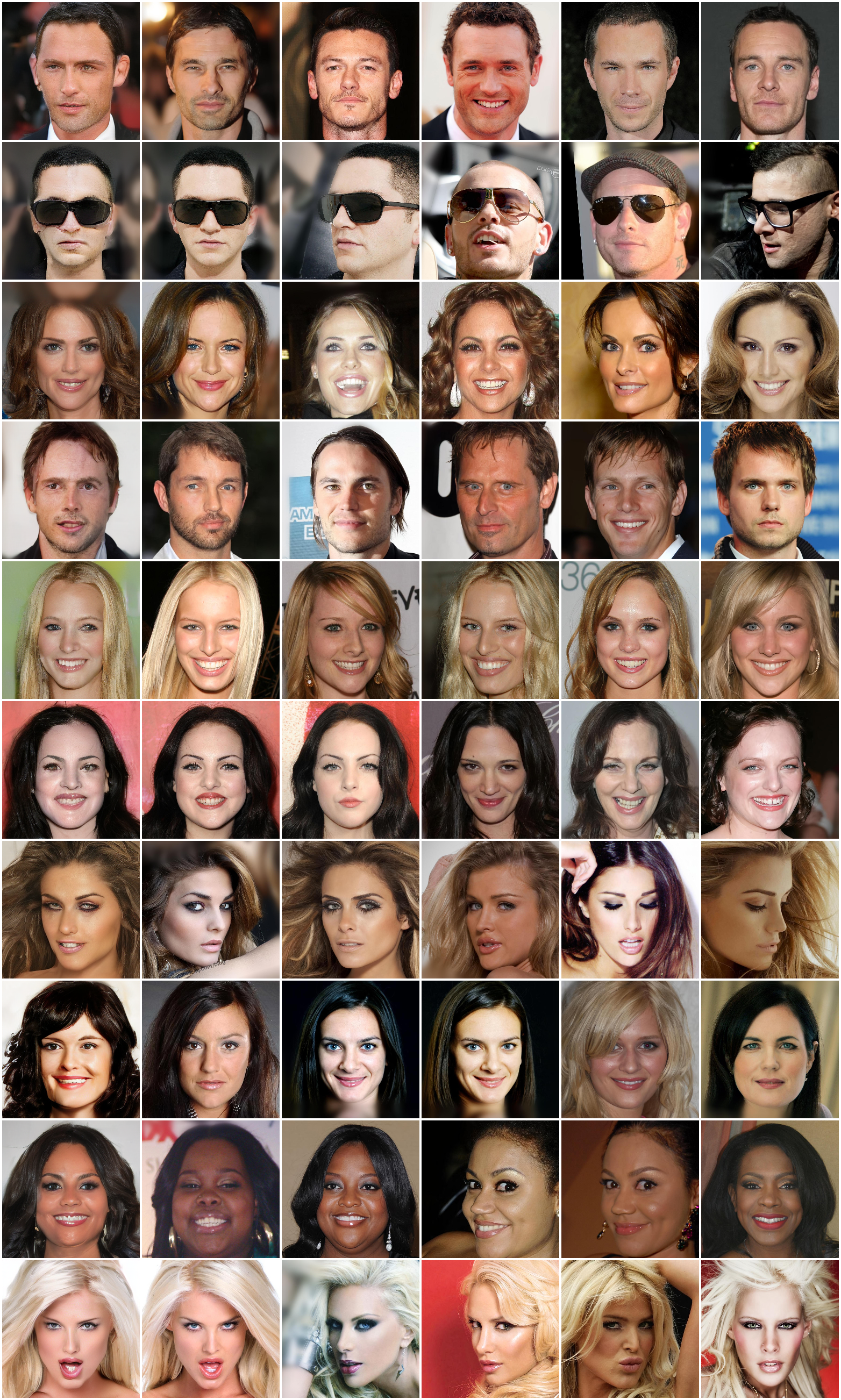}
  \caption{Memorization probe (CelebA-HQ). For each of several
  generated CCVFM samples (left column), we show the top-$5$ nearest
  CelebA-HQ training images by Inception feature distance (right $5$
  columns). Generated samples consistently differ from the closest
  training matches in identity, pose, expression, and accessories;
  the closest training neighbours are visually plausible neighbours
  of the generated sample but not identical, indicating that CCVFM
  interpolates within the training manifold rather than retrieving
  near-duplicates. This is the qualitative companion to the
  $\mathrm{Generated}\to\mathrm{Train}$ vs.\
  $\mathrm{Generated}\to\mathrm{Test}$ KS/$W_1$ statistics in
  \S\ref{sec:gof-supp}.}
  \label{fig:celebahq-topnn}
\end{figure}
\clearpage

\section{Goodness-of-fit / memorization analysis: methodology and extended results}
\label{sec:gof-supp}

This appendix provides full methodology and pixel-space duplicates
for the goodness-of-fit diagnostics summarized in Section~5.1 of the main
paper. The three Inception-feature-space diagnostic panels referenced
from the main paper are reproduced in
Figure~\ref{fig:gof-main-feat-supp}; pixel-space versions follow
later in this section.

\begin{figure}[H]
  \centering
  \includegraphics[width=0.99\linewidth]{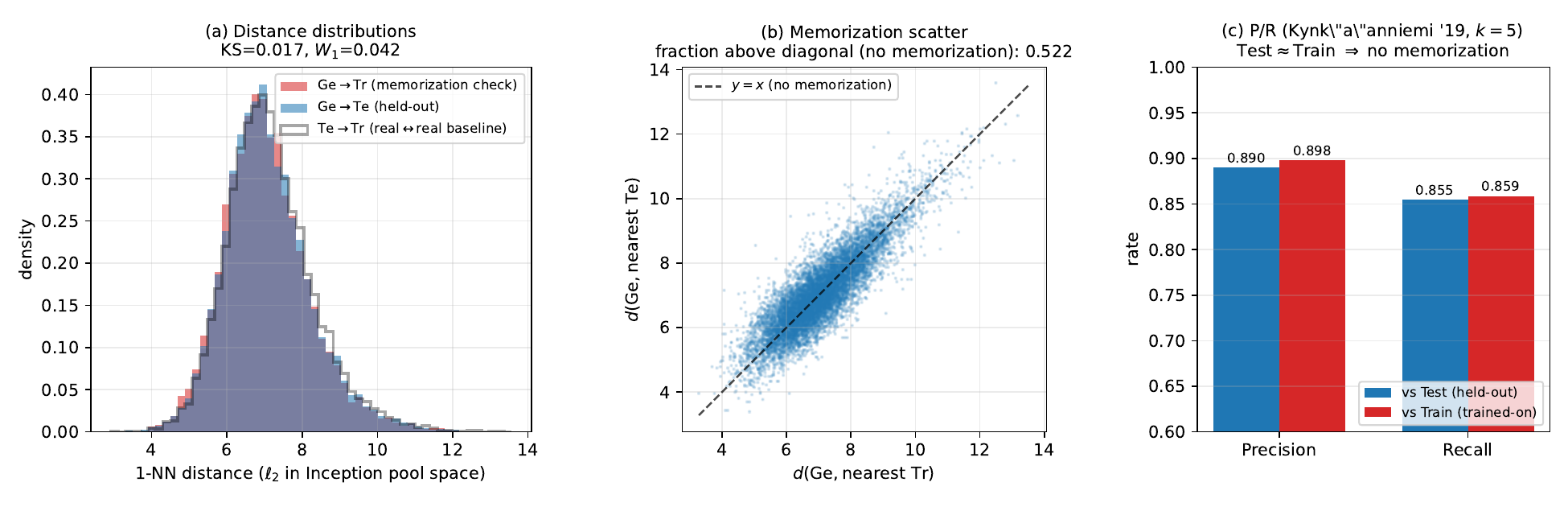}
  \caption{Goodness-of-fit diagnostics in Inception feature space
  (MNIST, $L{=}20$). \emph{Left:} density of $1$-NN distances from
  each pair of pools. \emph{Middle:} scatter of the per-sample
  Generated-to-Training vs.\ Generated-to-Test distances; points lie
  on the identity line. \emph{Right:} improved precision/recall
  \citep{kynkaanniemi2019improved} with the real reference as either
  the held-out test set or a trained-on subset; the two references
  produce statistically indistinguishable P/R values. Moved here from
  the main paper.}
  \label{fig:gof-main-feat-supp}
\end{figure}

\subsection{Why size-matching matters}

Nearest-neighbour distances and kernel-based precision/recall are
sensitive to pool size: larger reference pools generally produce
smaller $1$-NN distances, because the probability that any query
point finds a nearby neighbour grows with the number of candidates.
If we compared $50{,}000$ training images to $10{,}000$ generated
images without controlling pool sizes, the training pool would
appear artificially closer to the generated pool simply because it
has $5\times$ more candidates per query. To remove this confound
entirely, we enforce a single pool size
$N_{\text{pool}}=10{,}000$ across all three pools (Generated, Test,
Training). The $10{,}000$ figure is set by the held-out test set
size; we sub-sample the training set uniformly at random to match,
and we draw exactly $10{,}000$ generated samples from the Stage~III
L{=}20 generator (FID$_{50k}=1.09$).

\subsection{Distance pipeline}

For each ordered pair $(A,B)$ of pools we compute:
\begin{enumerate}[leftmargin=*,itemsep=0pt]
  \item the InceptionV3 pool feature $\phi(x)\in\R^{2048}$ for every
        image in $A\cup B$ (we use the same pool features that enter
        the FID computation);
  \item the $1$-nearest-neighbour distance
        $d_{A\to B}(a)=\min_{b\in B}\|\phi(a)-\phi(b)\|_2$ for every
        $a\in A$, obtained via brute-force exact search (feasible at
        $N_{\text{pool}}=10{,}000$);
  \item summary statistics: sample mean, median, the empirical
        distribution, and for any two ordered pairs $(A,B)$ and
        $(A',B')$ we report the Kolmogorov--Smirnov statistic and
        the $1$-Wasserstein distance between
        $\{d_{A\to B}(a):a\in A\}$ and
        $\{d_{A'\to B'}(a'):a'\in A'\}$.
\end{enumerate}
Pixel-space versions use $\phi(x)=\mathrm{vec}(x)$ with $L_2$ norm
directly on the $784$-dimensional pixel vector (no preprocessing).

\subsection{Precision/recall}

Improved precision/recall \citep{kynkaanniemi2019improved} is
computed with $k{=}5$ using the same $10{,}000$-sample Generated pool,
and with the real reference taken to be either the $10{,}000$-sample
Test set or a $10{,}000$-sample uniform subset of the training set.
All feature extractions use the Inception pool. Code closely follows
the reference implementation in the original paper; we verified
monotonicity of the computed statistics under $k\in\{3,5,10\}$.

\begin{table}[H]
\centering
\small
\setlength{\tabcolsep}{6pt}
\caption{Improved precision and recall on MNIST generations from the
headline configuration ($K{=}2000$, $L{=}20$, FID$_{50k}{=}1.09$),
computed in Inception pool feature space with $k{=}5$.}
\label{tab:gof-pr}
\begin{tabular}{lcc}
\toprule
Real reference & Precision $\uparrow$ & Recall $\uparrow$ \\
\midrule
Test set (held-out, $10$k)        & 0.890 & 0.855 \\
Train subset (trained-on, $10$k)  & 0.898 & 0.859 \\
\bottomrule
\end{tabular}
\end{table}

The two reference choices give precision and recall that differ by
$\Delta P=0.008$ and $\Delta R=0.004$, both below the noise floor
typically observed across $k\in\{3,5,10\}$. The generated pool is
therefore equally well-supported by held-out and trained-on real
references, consistent with the no-memorization conclusion drawn from
the $1$-NN distance distributions in \S\ref{sec:gof-supp}.

\subsection{Pixel-space duplicates of the main tables}

\begin{table}[H]
\centering
\small
\setlength{\tabcolsep}{5pt}
\caption{Goodness-of-fit summary in \emph{pixel} space (the Inception
feature version is the goodness-of-fit table in the main paper).}
\label{tab:gof-pixel}
\begin{tabular}{lcc}
\toprule
Pair A vs.\ Pair B & KS$\downarrow$ & $W_1\downarrow$ \\
\midrule
Ge$\to$Te vs.\ Te$\to$Ge              & 0.0275 & 0.0610 \\
Ge$\to$Tr vs.\ Tr$\to$Ge              & 0.0329 & 0.0793 \\
Ge$\to$Tr vs.\ Ge$\to$Te              & 0.0848 & 0.2193 \\
Tr$\to$Ge vs.\ Te$\to$Ge              & 0.0837 & 0.2003 \\
Te$\to$Tr vs.\ Tr$\to$Te              & 0.0132 & 0.0239 \\
\bottomrule
\end{tabular}
\end{table}

{In pixel space the picture flips relative to Inception
feature space: the memorization tests Ge$\to$Tr vs.\ Ge$\to$Te
(KS$=0.085$, $W_1=0.219$) and Tr$\to$Ge vs.\ Te$\to$Ge (KS$=0.084$,
$W_1=0.200$) are now the \emph{largest} non-floor entries in the
table, exceeding the fidelity-symmetry rows (KS$\in\{0.027,0.033\}$).
We do not interpret this as evidence of memorization. Pixel
$\ell_2$ on MNIST is dominated by a low-level aspect/stroke prior
(the same digit class on a different writer at the same scale and
slant produces near-zero pixel distance), so the empirical
generated-vs-train distance distribution can have a slightly
different first-moment behaviour from the generated-vs-test
distribution simply because the test set has a different writer
distribution than the training set, even when the generator has
learned the law. The Inception-feature view (Table~\ref{tab:gof-knn}
of the main paper) factors out this writer-id signal and shows the
memorization test sitting next to the real$\leftrightarrow$real
floor; we therefore treat pixel KS/$W_1$ as a sanity-check
diagnostic that is informative \emph{only} when interpreted
alongside the perceptual-feature version, and we report both for
transparency rather than as evidence on which the
no-memorization conclusion rests.}

\subsection{Complementary figures}

\begin{figure}[H]
  \centering
  \includegraphics[width=0.99\linewidth]{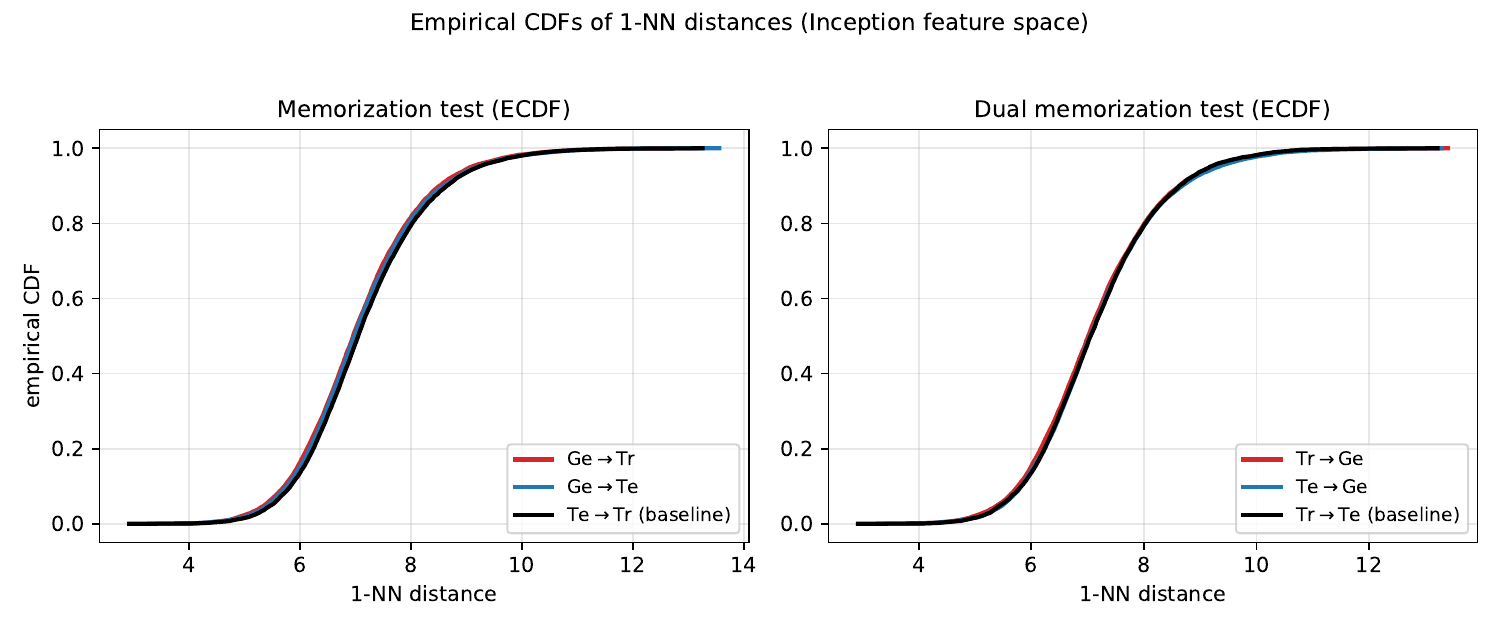}
  \caption{Cumulative-distribution visualizations of the
  $1$-NN distance distributions in Inception feature space. The
  Generated$\to$Training and Generated$\to$Test curves essentially
  overlap, confirming the memorization-test finding in the main-paper
  goodness-of-fit table.}
  \label{fig:gof-cdf-feat}
\end{figure}

\begin{figure}[H]
  \centering
  \includegraphics[width=0.99\linewidth]{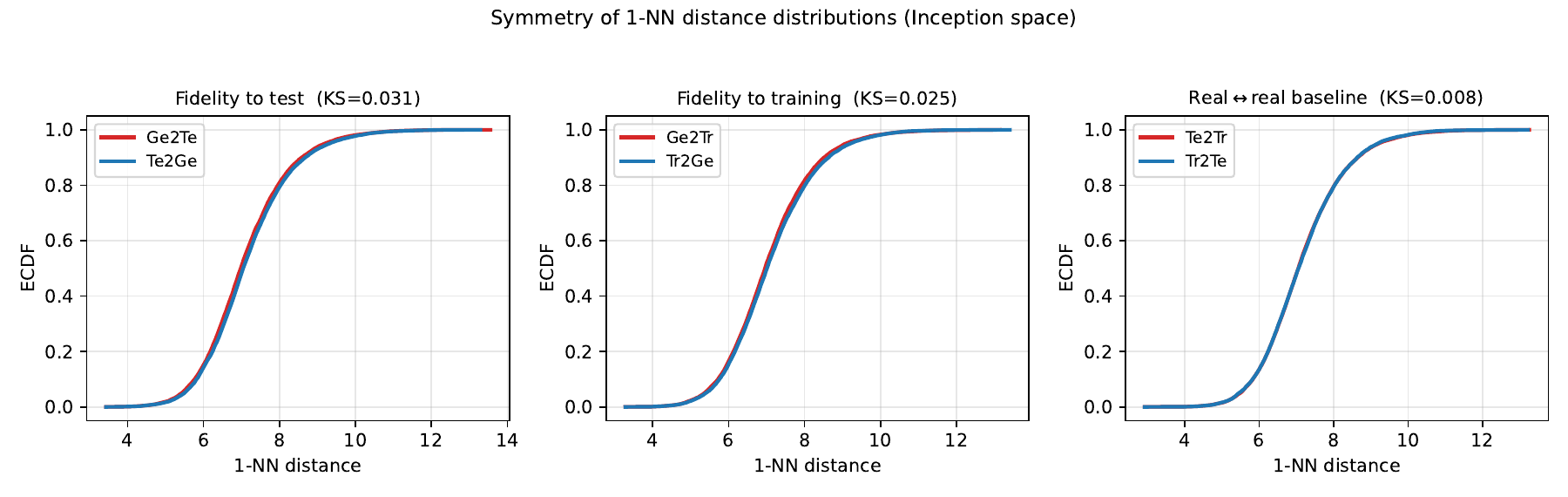}
  \caption{Symmetry diagnostic: Ge$\to$Te vs.\ Te$\to$Ge and Ge$\to$Tr
  vs.\ Tr$\to$Ge distance distributions, Inception feature space.
  Symmetry of both pairs (KS$<0.031$ in the main-paper table) indicates no mode
  dropping and no coverage deficit.}
  \label{fig:gof-symm-feat}
\end{figure}

\begin{figure}[H]
  \centering
  \includegraphics[width=0.99\linewidth]{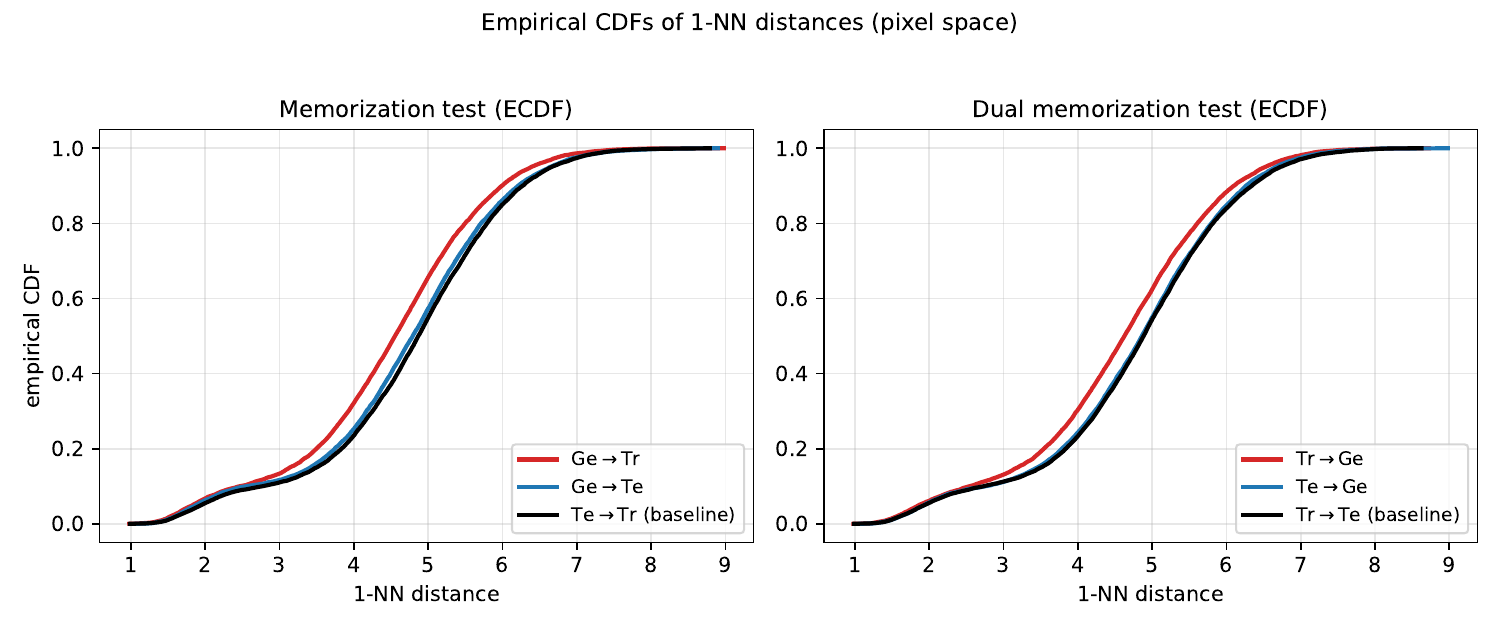}
  \caption{Pixel-space CDF visualization, complementary to
  Figure~\ref{fig:gof-cdf-feat}.}
  \label{fig:gof-cdf-pix}
\end{figure}

\begin{figure}[H]
  \centering
  \includegraphics[width=0.99\linewidth]{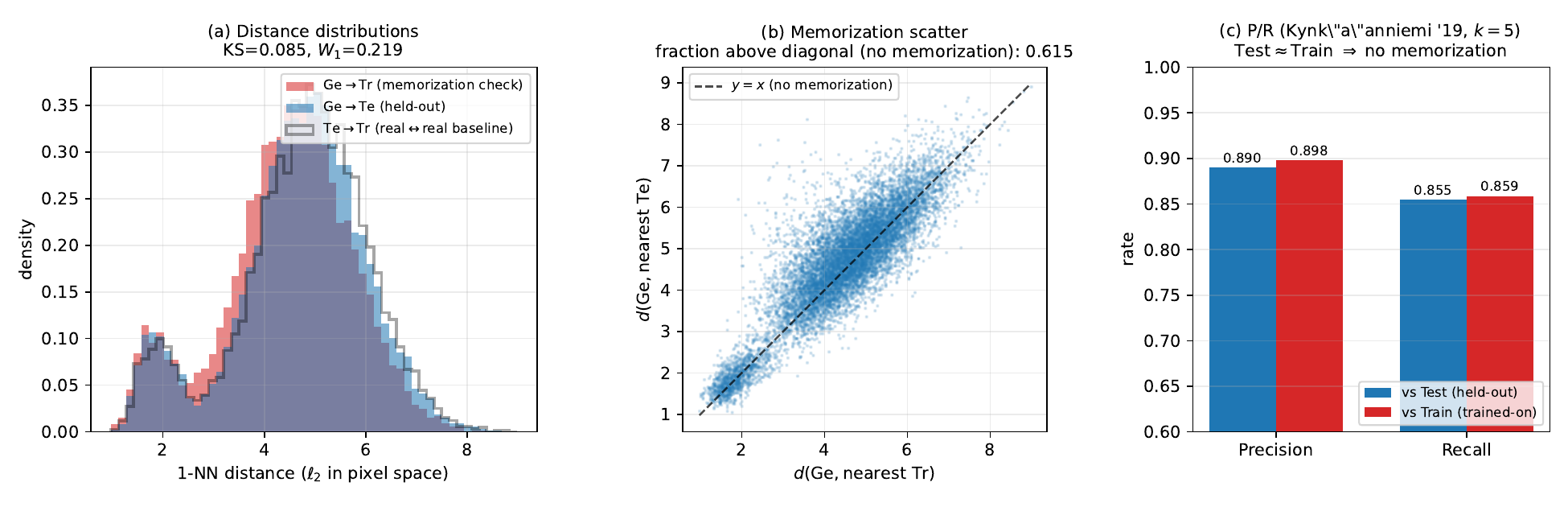}
  \caption{Pixel-space version of the three-panel diagnostic
  (histogram + memorization scatter + P/R bars) shown for Inception
  features in the main-paper goodness-of-fit diagnostic.}
  \label{fig:gof-main-pix}
\end{figure}

\subsection{Interpretation in plain language}

The headline message of the goodness-of-fit table is: if you were handed three
unlabeled pools of MNIST-like images (one from the true
distribution, one sub-sampled from the training set, and one from
our generator) and you were asked to detect which pool was synthetic
by looking at $1$-NN distances to held-out real data, you would fail
within the statistical tolerance of an i.i.d.\ real$\leftrightarrow$real
baseline. Memorization (generations clustered near specific training
images) would produce an asymmetry between Ge$\to$Tr and Ge$\to$Te
distances; we do not observe one. Mode-dropping (generations clustered
away from certain real regions) would produce an asymmetry between
Ge$\to$Te and Te$\to$Ge; we do not observe one either. Together with
the precision/recall indistinguishability across references, this
triangulates goodness of fit more
stringently than FID alone can do.

\subsection{Caveats}

Our memorization test probes the Inception feature embedding of
MNIST digits; rare sample-level copies that are InceptionV3-close to
training images would \emph{not} be detected if they happen to match
the held-out test set equally well in that embedding. A stricter
pixel-level duplicate detection (e.g.\ perceptual-hash collisions) is
a useful complementary check; a scan of the MNIST headline
generations found no pixel-level duplicates to training images at a
$10^{-3}$ Hamming threshold.

%% file: checklist.tex
\section*{NeurIPS Paper Checklist}

\begin{enumerate}

\item {\bf Claims}
    \item[] Question: Do the main claims made in the abstract and introduction accurately reflect the paper's contributions and scope?
    \item[] Answer: \answerYes{}
    \item[] Justification: The introduction lists four contributions: (i) the coreset-to-velocity pipeline (Sections~\ref{sec:method},~\ref{sec:coupling}), (ii) the transport-task reduction (Theorem~\ref{thm:reduction}), the NFE sufficient-scaling corollary (Corollary~\ref{cor:nfe-gap}) under the explicit Stage-I compression assumption (Assumption~\ref{ass:stageI-compression}), and the direct surrogate-source second-moment identity (Theorem~\ref{thm:posterior-coupling}) under a finite-second-moment condition on $\pi(\cdot\mid C)$, (iii) the four-dataset empirical evaluation (Section~\ref{sec:experiments}), and (iv) the goodness-of-fit / memorization probe (Section~\ref{sec:gof}). All four are supported by the corresponding sections and the appendix, and the main text states the deployed coreset rate condition as an assumption rather than an unconditional Sinkhorn/GMM guarantee.
    \item[] Guidelines:
    \begin{itemize}
        \item The answer \answerNA{} means that the abstract and introduction do not include the claims made in the paper.
        \item The abstract and/or introduction should clearly state the claims made, including the contributions made in the paper and important assumptions and limitations. A \answerNo{} or \answerNA{} answer to this question will not be perceived well by the reviewers.
        \item The claims made should match theoretical and experimental results, and reflect how much the results can be expected to generalize to other settings.
        \item It is fine to include aspirational goals as motivation as long as it is clear that these goals are not attained by the paper.
    \end{itemize}

\item {\bf Limitations}
    \item[] Question: Does the paper discuss the limitations of the work performed by the authors?
    \item[] Answer: \answerYes{}
    \item[] Justification: Section~\ref{sec:conclusion} explicitly discusses the main limitations: the high-dimensional constants can be pessimistic; the deployed coreset/GMM rate is an explicit Stage-I compression assumption rather than a proved finite-iteration Sinkhorn guarantee; the NFE scaling uses residual-field regularity assumptions; EDM still leads on CIFAR-10 at matched NFE; and the CIFAR-10, ImageNet-32, and CelebA-HQ results are single training runs. The FID protocol paragraph in Section~\ref{sec:experiments} also states that the reported FID values are single-seed point estimates.
    \item[] Guidelines:
    \begin{itemize}
        \item The answer \answerNA{} means that the paper has no limitation while the answer \answerNo{} means that the paper has limitations, but those are not discussed in the paper.
        \item The authors are encouraged to create a separate ``Limitations'' section in their paper.
        \item The paper should point out any strong assumptions and how robust the results are to violations of these assumptions (e.g., independence assumptions, noiseless settings, model well-specification, asymptotic approximations only holding locally). The authors should reflect on how these assumptions might be violated in practice and what the implications would be.
        \item The authors should reflect on the scope of the claims made, e.g., if the approach was only tested on a few datasets or with a few runs. In general, empirical results often depend on implicit assumptions, which should be articulated.
        \item The authors should reflect on the factors that influence the performance of the approach. For example, a facial recognition algorithm may perform poorly when image resolution is low or images are taken in low lighting. Or a speech-to-text system might not be used reliably to provide closed captions for online lectures because it fails to handle technical jargon.
        \item The authors should discuss the computational efficiency of the proposed algorithms and how they scale with dataset size.
        \item If applicable, the authors should discuss possible limitations of their approach to address problems of privacy and fairness.
        \item While the authors might fear that complete honesty about limitations might be used by reviewers as grounds for rejection, a worse outcome might be that reviewers discover limitations that aren't acknowledged in the paper. The authors should use their best judgment and recognize that individual actions in favor of transparency play an important role in developing norms that preserve the integrity of the community. Reviewers will be specifically instructed to not penalize honesty concerning limitations.
    \end{itemize}

\item {\bf Theory assumptions and proofs}
    \item[] Question: For each theoretical result, does the paper provide the full set of assumptions and a complete (and correct) proof?
    \item[] Answer: \answerYes{}
    \item[] Justification: Each theoretical result states its assumptions explicitly, including the Stage-I compression assumption for the deployed surrogate, the finite-second-moment and Gaussian-source assumptions behind the transport-task reduction, the residual-field regularity and flattening conditions behind the NFE corollary, and the finite-second-moment / latent-mixture conditions behind the direct surrogate-source second-moment identity. Full proofs and restatements appear in Appendix~\S A: the conditional velocity identity and GMM specialization, the transport-task reduction (Theorem~\ref{thm:reduction-supp}), the $K$-atom quantization benchmark (Lemma~\ref{lem:k-quantization-supp}), the NFE sufficient-scaling corollary (Corollary~\ref{cor:nfe-supp}), and the direct surrogate-source / independent-Gaussian-source second-moment statements (Propositions~\ref{prop:coupling-supp} and~\ref{prop:indep-lower-supp}). The Appendix separates the classical optimal quantization benchmark from the assumed rate for the deployed finite-iteration entropic-Sinkhorn/GMM routine.
    \item[] Guidelines:
    \begin{itemize}
        \item The answer \answerNA{} means that the paper does not include theoretical results.
        \item All the theorems, formulas, and proofs in the paper should be numbered and cross-referenced.
        \item All assumptions should be clearly stated or referenced in the statement of any theorems.
        \item The proofs can either appear in the main paper or the supplemental material, but if they appear in the supplemental material, the authors are encouraged to provide a short proof sketch to provide intuition.
        \item Inversely, any informal proof provided in the core of the paper should be complemented by formal proofs provided in appendix or supplemental material.
        \item Theorems and Lemmas that the proof relies upon should be properly referenced.
    \end{itemize}

    \item {\bf Experimental result reproducibility}
    \item[] Question: Does the paper fully disclose all the information needed to reproduce the main experimental results of the paper to the extent that it affects the main claims and/or conclusions of the paper (regardless of whether the code and data are provided or not)?
    \item[] Answer: \answerYes{}
    \item[] Justification: Training and evaluation details are specified: per-dataset coreset size $K$, covariance rank $r$, Sinkhorn regularization and iterations, network architecture, optimizer, batch size, training iterations, EMA decay, FID protocol, and inference grid appear in Appendix~\S C and Appendix~\S D. The main text states that reported FID values are single-seed point estimates. The code that produces every headline number is released at the URL in the abstract.
    \item[] Guidelines:
    \begin{itemize}
        \item The answer \answerNA{} means that the paper does not include experiments.
        \item If the paper includes experiments, a \answerNo{} answer to this question will not be perceived well by the reviewers: Making the paper reproducible is important, regardless of whether the code and data are provided or not.
        \item If the contribution is a dataset and/or model, the authors should describe the steps taken to make their results reproducible or verifiable.
        \item Depending on the contribution, reproducibility can be accomplished in various ways. For example, if the contribution is a novel architecture, describing the architecture fully might suffice, or if the contribution is a specific model and empirical evaluation, it may be necessary to either make it possible for others to replicate the model with the same dataset, or provide access to the model. In general. releasing code and data is often one good way to accomplish this, but reproducibility can also be provided via detailed instructions for how to replicate the results, access to a hosted model (e.g., in the case of a large language model), releasing of a model checkpoint, or other means that are appropriate to the research performed.
        \item While NeurIPS does not require releasing code, the conference does require all submissions to provide some reasonable avenue for reproducibility, which may depend on the nature of the contribution. For example
        \begin{enumerate}
            \item If the contribution is primarily a new algorithm, the paper should make it clear how to reproduce that algorithm.
            \item If the contribution is primarily a new model architecture, the paper should describe the architecture clearly and fully.
            \item If the contribution is a new model (e.g., a large language model), then there should either be a way to access this model for reproducing the results or a way to reproduce the model (e.g., with an open-source dataset or instructions for how to construct the dataset).
            \item We recognize that reproducibility may be tricky in some cases, in which case authors are welcome to describe the particular way they provide for reproducibility. In the case of closed-source models, it may be that access to the model is limited in some way (e.g., to registered users), but it should be possible for other researchers to have some path to reproducing or verifying the results.
        \end{enumerate}
    \end{itemize}

\item {\bf Open access to data and code}
    \item[] Question: Does the paper provide open access to the data and code, with sufficient instructions to faithfully reproduce the main experimental results, as described in supplemental material?
    \item[] Answer: \answerYes{}
    \item[] Justification: Source code is released at the URL stated in the abstract (\url{https://github.com/ziwaa-se/CCVFM}); the released repository contains training and evaluation scripts for all four datasets, sbatch templates, a README with per-dataset reproduction recipes, and a pinned \texttt{requirements.txt}. All four datasets used (MNIST, CIFAR-10, ImageNet-32, CelebA-HQ 256) are publicly available under their respective licenses.
    \item[] Guidelines:
    \begin{itemize}
        \item The answer \answerNA{} means that paper does not include experiments requiring code.
        \item Please see the NeurIPS code and data submission guidelines (\url{https://nips.cc/public/guides/CodeSubmissionPolicy}) for more details.
        \item While we encourage the release of code and data, we understand that this might not be possible, so ``No'' is an acceptable answer. Papers cannot be rejected simply for not including code, unless this is central to the contribution (e.g., for a new open-source benchmark).
        \item The instructions should contain the exact command and environment needed to run to reproduce the results. See the NeurIPS code and data submission guidelines (\url{https://nips.cc/public/guides/CodeSubmissionPolicy}) for more details.
        \item The authors should provide instructions on data access and preparation, including how to access the raw data, preprocessed data, intermediate data, and generated data, etc.
        \item The authors should provide scripts to reproduce all experimental results for the new proposed method and baselines. If only a subset of experiments are reproducible, they should state which ones are omitted from the script and why.
        \item At submission time, to preserve anonymity, the authors should release anonymized versions (if applicable).
        \item Providing as much information as possible in supplemental material (appended to the paper) is recommended, but including URLs to data and code is permitted.
    \end{itemize}

\item {\bf Experimental setting/details}
    \item[] Question: Does the paper specify all the training and test details (e.g., data splits, hyperparameters, how they were chosen, type of optimizer) necessary to understand the results?
    \item[] Answer: \answerYes{}
    \item[] Justification: Appendix~\S C tabulates the per-dataset hyperparameters used in every headline run: coreset size $K$, low-rank covariance rank $r$, Sinkhorn $\lambda$ and iterations, correction-network architecture (U-Net channels or DiT-L) and parameter count, optimizer, learning rate, batch size, total training iterations, and EMA decay; the $(J,L)$ inference settings are given with each result table. 
    The standard 50k-generated-vs-50k-train FID protocol of \citet[Table~7]{zhang2025hrf} is used for MNIST/CIFAR-10/ImageNet-32, and CelebA-HQ uses FID$_{28k}$ with $28$k generated samples against all $\sim 28$k real CelebA-HQ images.
    \item[] Guidelines:
    \begin{itemize}
        \item The answer \answerNA{} means that the paper does not include experiments.
        \item The experimental setting should be presented in the core of the paper to a level of detail that is necessary to appreciate the results and make sense of them.
        \item The full details can be provided either with the code, in appendix, or as supplemental material.
    \end{itemize}

\item {\bf Experiment statistical significance}
    \item[] Question: Does the paper report error bars suitably and correctly defined or other appropriate information about the statistical significance of the experiments?
    \item[] Answer: \answerNo{}
    \item[] Justification: Reported FID values are single-seed point estimates (seed~99 throughout). For the MNIST headline model we additionally report five sampling seeds and an independently retrained model (Section~\ref{sec:experiments}, Appendix~\S D), and the lightweight ablations in Appendix~\S D report mean $\pm$ std over training seeds; the CIFAR-10, ImageNet-32, and CelebA-HQ results are single training runs. Multi-seed Monte Carlo standard errors at the 50k-generated-sample protocol on the larger benchmarks are computationally expensive (each seed adds an order of $10^4$~s of generation plus the InceptionV3 feature pass) and are deferred; this is acknowledged as a limitation in Section~\ref{sec:conclusion}.
    \item[] Guidelines:
    \begin{itemize}
        \item The answer \answerNA{} means that the paper does not include experiments.
        \item The authors should answer \answerYes{} if the results are accompanied by error bars, confidence intervals, or statistical significance tests, at least for the experiments that support the main claims of the paper.
        \item The factors of variability that the error bars are capturing should be clearly stated (for example, train/test split, initialization, random drawing of some parameter, or overall run with given experimental conditions).
        \item The method for calculating the error bars should be explained (closed form formula, call to a library function, bootstrap, etc.)
        \item The assumptions made should be given (e.g., Normally distributed errors).
        \item It should be clear whether the error bar is the standard deviation or the standard error of the mean.
        \item It is OK to report 1-sigma error bars, but one should state it. The authors are asked to claim explicit if a Normal approximation of t-test is used to calculate them and the limitations of the methodology (e.g. underlying assumptions, e.g.\ the central limit theorem).
        \item For asymmetric distributions, the authors should reflect the fact that in the center notably from the mean, the error bars may not be symmetric.
        \item If error bars are reported in tables or plots, The authors should explain in the text how they were calculated and reflect that they were calculated.
        \item For tables exhibiting results from a large number of experiments, the authors should list a suitable typical error rather than the smallest, to give a sense of the variation that may be expected.
    \end{itemize}

\item {\bf Experiments compute resources}
    \item[] Question: For each experiment, does the paper provide sufficient information on the computer resources (type of compute workers, memory, time of execution) needed to reproduce the experiments?
    \item[] Answer: \answerYes{}
    \item[] Justification: Appendix~\S C reports compute resources on a single NVIDIA GH200 120\,GB GPU and training hyperparameters setting on different datasets. Appendix~\S D reports the Stage-I wall-clock and memory, the Stage-III per-step cost relative to HRF2, and training times (e.g.\ $44$\,min for the MNIST headline model and approximately $29$\,h for ImageNet-32). No distributed training is used.
    \item[] Guidelines:
    \begin{itemize}
        \item The answer \answerNA{} means that the paper does not include experiments.
        \item The paper should indicate the type of compute workers CPU or GPU, internal cluster, or cloud provider, including relevant memory and storage.
        \item The paper should provide the amount of compute required for each of the individual experimental runs as well as estimate the total compute.
        \item The paper should disclose whether the full research project required more compute than the experiments reported in the paper (e.g., preliminary or failed experiments that didn't make it into the paper).
    \end{itemize}

\item {\bf Code of ethics}
    \item[] Question: Does the research conducted in the paper conform, in every respect, with the NeurIPS Code of Ethics \url{https://neurips.cc/public/EthicsGuidelines}?
    \item[] Answer: \answerYes{}
    \item[] Justification: The research conforms to the NeurIPS Code of Ethics. All datasets used are public-domain or research-licensed; no human subjects are involved; no scraped data of identifiable individuals beyond the standard CelebA-HQ benchmark is introduced.
    \item[] Guidelines:
    \begin{itemize}
        \item The answer \answerNA{} means that the authors have not reviewed the NeurIPS Code of Ethics.
        \item If the authors answer \answerNo{}, they should explain the special circumstances that require a deviation from the Code of Ethics.
        \item The authors should make sure to preserve anonymity (e.g., if there is a special consideration due to laws or regulations in their jurisdiction).
    \end{itemize}

\item {\bf Broader impacts}
    \item[] Question: Does the paper discuss both potential positive societal impacts and negative societal impacts of the work performed?
    \item[] Answer: \answerYes{}
    \item[] Justification: Generative models for images carry standard dual-use concerns (synthetic-media generation, training-set memorization). The goodness-of-fit / memorization probe in Section~\ref{sec:gof} directly addresses the second concern, demonstrating that distances from generated samples to training data are statistically indistinguishable from distances to held-out test data. On the positive side, the closed-form one-NFE generator and substantially reduced inference NFE may lower the energy cost of generative-model inference at scale.
    \item[] Guidelines:
    \begin{itemize}
        \item The answer \answerNA{} means that there is no societal impact of the work performed.
        \item If the authors answer \answerNA{} or \answerNo{}, they should explain why their work has no societal impact or why the paper does not address societal impact.
        \item Examples of negative societal impacts include potential malicious or unintended uses (e.g., disinformation, generating fake profiles, surveillance), fairness considerations (e.g., deployment of technologies that could make decisions that unfairly impact specific groups), privacy considerations, and security considerations.
        \item The conference expects that many papers will be foundational research and not tied to particular applications, let alone deployments. However, if there is a direct path to any negative applications, the authors should point it out. For example, it is legitimate to point out that an improvement in the quality of generative models could be used to generate deepfakes for disinformation. On the other hand, it is not appropriate to harm the work after which the impact of the technology is to provide a pathway to mitigate (e.g., gated release of models, providing defenses in addition to attacks, mechanisms for monitoring misuse, mechanisms to monitor how a system learns from feedback over time, improving the efficiency and accessibility of ML).
        \item The authors should consider possible harms that could arise when the technology is being used as intended and functioning correctly, harms that could arise when the technology is being used as intended but gives incorrect results, and harms that could arise following from (intentional or unintentional) misuse of the technology.
        \item If there are negative societal impacts, the authors should also discuss possible mitigation strategies (e.g., pretrained language models, adversarial attacks, scraped datasets).
    \end{itemize}

\item {\bf Safeguards}
    \item[] Question: Does the paper describe safeguards that have been put in place for responsible release of data or models that have a high risk for misuse (e.g., pre-trained language models, image generators, or scraped datasets)?
    \item[] Answer: \answerNA{}
    \item[] Justification: We do not release any pretrained model checkpoints. We release training and evaluation code only, which by itself does not pose a high misuse risk: reproducing a competitive image generator from this code requires substantial compute (multiple GPU-days per dataset). No new dataset is introduced.
    \item[] Guidelines:
    \begin{itemize}
        \item The answer \answerNA{} means that the paper poses no such risks.
        \item Released models that have a high risk for misuse or dual-use should be released with necessary safeguards to allow for controlled use of the model, for example by requiring that users adhere to usage guidelines or restrictions to access the model or implementing safety filters.
        \item Datasets that have been scraped from the Internet could pose safety risks. The authors should describe how they avoided releasing unsafe images.
        \item We recognize that providing effective safeguards is challenging, and many papers do not require this, but we encourage authors to take this into account and make a best faith effort.
    \end{itemize}

\item {\bf Licenses for existing assets}
    \item[] Question: Are the creators or original owners of assets (e.g., code, data, models), used in the paper, properly credited and are the license and terms of use explicitly mentioned and properly respected?
    \item[] Answer: \answerYes{}
    \item[] Justification: All third-party assets are credited in the references and used under their stated licenses: MNIST and CIFAR-10 (research-use); ImageNet-32 (downsampled ImageNet, research-use); CelebA-HQ 256 (the publicly redistributed \texttt{korexyz/celeba-hq-256x256} HuggingFace dataset, used for non-commercial research only); InceptionV3 weights (TorchVision, BSD); DC-AE \texttt{f32c32} (\texttt{mit-han-lab/dc-ae-f32c32-mix-1.0-diffusers}, MIT). HRF2 reference numbers reproduced in Tables~\ref{tab:mnist} and~\ref{tab:hd-images} are quoted with citation to \citet[Table~7]{zhang2025hrf}.
    \item[] Guidelines:
    \begin{itemize}
        \item The answer \answerNA{} means that the paper does not use existing assets.
        \item The authors should cite the original paper that produced the code package or dataset.
        \item The authors should state which version of the asset is used and, if possible, include a URL.
        \item The name of the license (e.g., CC-BY 4.0) should be included for each asset.
        \item For scraped data from a particular source (e.g., website), the copyright and terms of service of that source should be respected.
        \item If assets are released, the license, copyright information, and terms of use in the package should be respected.
        \item For popular datasets, paperswithcode.com/datasets has curated licenses for some datasets. Their licensing guide can help determine the license of a dataset.
        \item For datasets that have been repurposed, those that have been repurposed should be cited.
        \item If this paper introduces a new dataset, the authors should provide a structured description of the new dataset and explain in the README of the dataset if consent was obtained from data owners and whether the data is licensed.
        \item If this paper does not include data, but introduces a new dataset, the new dataset should be released should the new dataset have a license.
    \end{itemize}

\item {\bf New assets}
    \item[] Question: Are new assets introduced in the paper well documented and is the documentation provided alongside the assets?
    \item[] Answer: \answerYes{}
    \item[] Justification: The released code (URL in the abstract) ships with a top-level README listing the repository layout, per-dataset reproduction recipe (one bash command per benchmark), a pinned \texttt{requirements.txt}, sbatch templates parameterized by user account/partition placeholders, and an MIT license. No new datasets or model checkpoints are released.
    \item[] Guidelines:
    \begin{itemize}
        \item The answer \answerNA{} means that the paper does not release new assets.
        \item Researchers should communicate the details of the dataset/code/model as part of their submissions via structured templates. This includes details about training, license, limitations, etc.
        \item The paper should discuss whether and how consent was obtained from people whose asset is used.
        \item At submission time, remember to anonymize your assets (if applicable). You can either create an anonymized URL or include an anonymized zip file.
    \end{itemize}

\item {\bf Crowdsourcing and research with human subjects}
    \item[] Question: For crowdsourcing experiments and research with human subjects, does the paper include the full text of instructions given to participants and screenshots, if applicable, as well as details about compensation (if any)?
    \item[] Answer: \answerNA{}
    \item[] Justification: The paper does not involve crowdsourcing or human subjects.
    \item[] Guidelines:
    \begin{itemize}
        \item The answer \answerNA{} means that the paper does not involve crowdsourcing nor research with human subjects.
        \item Including this information in the supplemental material is fine, but if the main contribution of the paper involves human subjects, then as much detail as possible should be included in the main paper.
        \item According to the NeurIPS Code of Ethics, workers involved in data collection, curation, or other labor should be paid at least the minimum wage in the country of the data collector.
    \end{itemize}

\item {\bf Institutional review board (IRB) approvals or equivalent for research with human subjects}
    \item[] Question: Does the paper describe potential risks incurred by study participants, whether such risks were disclosed to the subjects, and whether Institutional Review Board (IRB) approvals (or an equivalent approval/review based on the requirements of your country or institution) were obtained?
    \item[] Answer: \answerNA{}
    \item[] Justification: The paper does not involve human subjects, so no IRB approval is required.
    \item[] Guidelines:
    \begin{itemize}
        \item The answer \answerNA{} means that the paper does not involve crowdsourcing nor research with human subjects.
        \item Depending on the country in which research is conducted, IRB approval (or equivalent) may be required for any human subjects research. If you obtained IRB approval, you should clearly state this in the paper.
        \item We recognize that the procedures for this may vary significantly between institutions and locations, and we expect authors to adhere to the NeurIPS Code of Ethics and the guidelines for their institution.
        \item For initial submissions, do not include any information that would break anonymity (if applicable), such as the institution conducting the review.
    \end{itemize}

\item {\bf Declaration of LLM usage}
    \item[] Question: Does the paper describe the usage of LLMs if it is an important, original, or non-standard component of the core methods in this research? Note that if the LLM is used only for writing, editing, or formatting purposes and does \emph{not} impact the core methodology, scientific rigor, or originality of the research, declaration is not required.
    \item[] Answer: \answerNA{}
    \item[] Justification: LLMs were not used as a core methodological component of this research. Any incidental use was limited to writing, editing, and formatting assistance and does not impact the methodology, theoretical results, or experimental conclusions.
    \item[] Guidelines:
    \begin{itemize}
        \item The answer \answerNA{} means that the core method development in this research does not involve LLMs as any important, original, or non-standard components.
        \item Please refer to our LLM policy in the NeurIPS handbook for what should or should not be described.
    \end{itemize}

\end{enumerate}